\documentclass{article}
\usepackage[preprint]{icml2026}
\usepackage[textsize=tiny]{todonotes}

\usepackage{amsthm}

\newtheorem{lemma}{Lemma}
\usepackage{amssymb}
\usepackage{graphicx}
\usepackage[font=small,labelfont=bf]{caption}
\usepackage{subcaption}
\usepackage{float}
\usepackage{tcolorbox}
\usepackage{subcaption}

\usepackage{pifont,booktabs}
\usepackage{enumitem}

\usepackage{pifont}
\newcommand{\cmark}{\ding{51}}
\newcommand{\xmark}{\ding{55}}

\usepackage{graphicx}
\usepackage[font=small,labelfont=bf]{caption}

\usepackage{float}    
\usepackage{tcolorbox}
\usepackage{makecell}

\usepackage[utf8]{inputenc} 
\usepackage[T1]{fontenc}    
\usepackage{hyperref}       
\usepackage{url}            
\usepackage{booktabs}       
\usepackage{amsfonts}       
\usepackage{nicefrac}       
\usepackage{microtype}      
\usepackage{xcolor}         
\usepackage{multirow}
\usepackage{amsmath, amsthm}
\usepackage{todonotes}

\usepackage{listings}

\definecolor{codegray}{gray}{0.95}
\definecolor{codegreen}{rgb}{0,0.5,0}
\definecolor{codeblue}{rgb}{0,0,0.6}

\usepackage{listings}
\usepackage{xcolor}

\definecolor{LightGray}{gray}{0.95}
\definecolor{CodeBlue}{RGB}{0,0,180}
\definecolor{CodeGreen}{RGB}{0,120,0}
\definecolor{CodeRed}{RGB}{163,21,21}
\definecolor{CodeGray}{gray}{0.4}

\lstdefinestyle{arxivpython}{
  language=Python,
  backgroundcolor=\color{LightGray},
  basicstyle=\ttfamily\footnotesize,
  keywordstyle=\color{CodeBlue}\bfseries,
  commentstyle=\color{CodeGreen}\itshape,
  stringstyle=\color{CodeRed},
  numberstyle=\tiny\color{CodeGray},
  numbers=left,
  numbersep=8pt,
  frame=lines,
  rulecolor=\color{black},
  breaklines=true,
  breakatwhitespace=true,
  showstringspaces=false,
  tabsize=2,
  captionpos=b
}

\usepackage{amssymb}
\usepackage{graphicx}
\usepackage[font=small,labelfont=bf]{caption}
\usepackage{subcaption}
\usepackage{float}
\usepackage{tcolorbox}
\usepackage{pifont,booktabs}

\usepackage{adjustbox}   

\usepackage{siunitx}

\begin{document}
\setcounter{tocdepth}{-1}
\addtocontents{toc}{\protect\setcounter{tocdepth}{-1}}

\twocolumn[ 
  \icmltitle{Activation-Space Uncertainty Quantification for Pretrained Networks}

  \icmlsetsymbol{equal}{*}

  \begin{icmlauthorlist}
    \icmlauthor{Richard Bergna}{cam}
    \icmlauthor{Stefan Depeweg}{siemens}
    \icmlauthor{Sergio Calvo-Ordoñez}{oxford}
    \icmlauthor{Jonathan Plenk}{oxford}
    \icmlauthor{Alvaro Cartea}{oxford}
    \icmlauthor{Jose Miguel Hernández-Lobato}{cam}
  \end{icmlauthorlist}

  \icmlaffiliation{cam}{Department of Engineering, University of Cambridge, Cambridge, UK}
  \icmlaffiliation{siemens}{Siemens AG, Munich, Germany}
  \icmlaffiliation{oxford}{Mathematical Institute and Oxford-Man Institute, University of Oxford, Oxford, UK}

  \icmlcorrespondingauthor{Richard Bergna}{rsb63@cam.ac.uk}

  \icmlkeywords{Uncertainty Quantification, Gaussian Processes, Post-Hoc Methods, Deep Learning}

  \vskip 0.2in
]

\printAffiliationsAndNotice{}

\begin{abstract}

Reliable uncertainty estimates are crucial for deploying pretrained models; yet, many strong methods for quantifying uncertainty require retraining, Monte Carlo sampling, or expensive second-order computations and may alter a frozen backbone’s predictions. To address this, we introduce \textbf{Gaussian Process Activations (GAPA)}, a post-hoc method that shifts Bayesian modeling from weights to activations. GAPA replaces standard nonlinearities with Gaussian-process activations whose posterior mean \emph{exactly} matches the original activation, preserving the backbone's point predictions by construction while providing closed-form epistemic variances in activation space. To scale to modern architectures, we use a sparse variational inducing-point approximation over cached training activations, combined with \emph{local} $k$-nearest-neighbor subset conditioning, enabling deterministic single-pass uncertainty propagation without sampling, backpropagation, or second-order information. Across regression, classification, image segmentation, and language modeling, GAPA matches or outperforms strong post-hoc baselines in calibration and out-of-distribution detection while remaining efficient at test time.

\end{abstract}
\section{Introduction}

Reliable uncertainty quantification (UQ) is crucial in risk-sensitive deployments, yet many effective research methods remain impractical in modern settings \cite{abdar2021review}. Weight-space Bayesian approaches (e.g., variational BNNs) often require retraining, labeled data or multi-sample evaluation, ensembles multiply compute, and Laplace-style methods rely on curvature estimates that scale poorly as models and output spaces grow \cite{blundell2015weight, mackay1992practical,bergna2024uncertainty, gal2016dropout, lakshminarayanan2017simple, ritter2018scalable, ortega2023variational}. The gap is most pronounced for pretrained backbones, where weights are not expected to be modified and test-time budgets favor \emph{single-pass} inference. In this regime, a practical post-hoc method should be single-pass, prediction-preserving, epistemic, and scalable to foundation models (Table~\ref{tab:method_comparison}).

\begin{table}[ht]
\centering
\caption{The gap in uncertainty quantification methods.}

\setlength{\tabcolsep}{3pt}
\renewcommand{\arraystretch}{1.05}

\resizebox{\columnwidth}{!}{%
\begin{tabular}{@{}lccccc@{}}
\toprule

\textbf{Method}
& \makecell{\textbf{Post-hoc}} 
& \makecell{\textbf{Single}\\\textbf{Pass}} 
& \makecell{\textbf{Preserves}\\\textbf{Mean}} 
& \makecell{\textbf{Epistemic}\\\textbf{UQ}} 
& \makecell{\textbf{Foundation}\\\textbf{Ready}} \\

\midrule
BNNs & \xmark & \xmark & \xmark & \cmark & \xmark \\
Ensembles & \xmark & \xmark & --- & \cmark & \xmark \\
MC Dropout & \xmark & \xmark & \xmark & \cmark & \xmark \\
Laplace & \cmark & \xmark & \xmark & \cmark & \xmark \\
LL-Laplace & \cmark & \cmark & \xmark & \cmark & \xmark \\
Temp. Scaling & \cmark & \cmark & \xmark & \xmark & \cmark \\
\midrule
\textbf{GAPA (Ours)} & \cmark & \cmark & \cmark & \cmark & \cmark \\
\bottomrule
\end{tabular}%
}

\vspace{2pt}
\label{tab:method_comparison}
\end{table}

We address this by shifting uncertainty modeling from weights to \emph{activations}. We introduce \textbf{Gaussian Process Activations (GAPA)}, a drop-in module that replaces deterministic activations with Gaussian-process activations whose \emph{posterior mean matches the original nonlinearity}, thereby preserving the frozen backbone’s point predictions by construction while producing activation-space epistemic variances (Figure~\ref{fig:baselines_classification}). For scalability, GAPA conditions on cached training activations using a sparse approximation (compression + local $k$NN conditioning), and propagates the resulting uncertainty through the network via deterministic variance-propagation rules, enabling \emph{single-pass} predictive uncertainty. Our contributions are

\begin{enumerate}[leftmargin=*, topsep=2pt, itemsep=1pt, parsep=0pt, partopsep=0pt]
\item \textbf{Mean-preserving post-hoc UQ:} GAPA provides epistemic uncertainty for pretrained networks while preserving point predictions.
\item \textbf{Scalable conditioning:} we combine induction points and local $k$-NN conditioning for practical inference at modern scales.
\item \textbf{Deterministic propagation:} we derive single-pass variance propagation from activation space to output space.
\item \textbf{Empirical validation:} across regression, classification, segmentation, and language modelling, GAPA improves calibration and OOD detection with fast inference time.
\end{enumerate}

\begin{figure*}[t]
    \centering
    \includegraphics[width=0.95\textwidth]{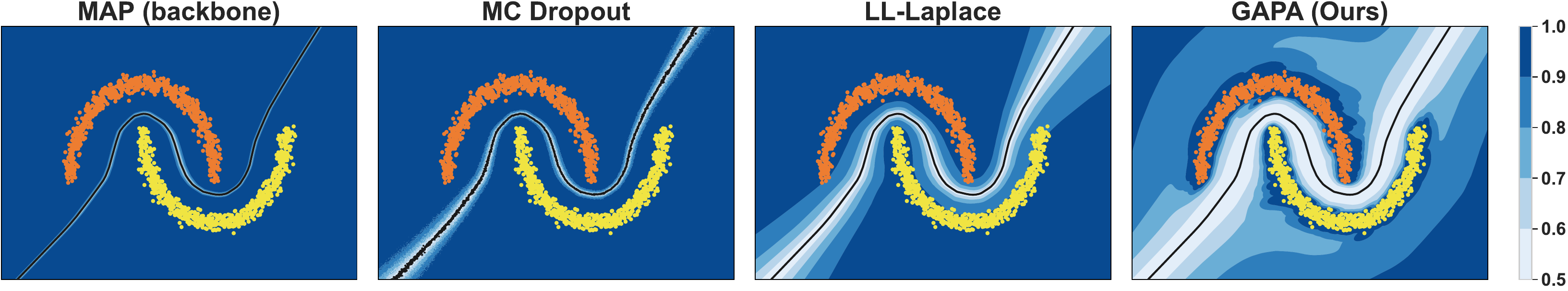}
    \caption{Comparison of uncertainty quantification methods on a toy binary classification 
task. \textbf{Left to right}: MAP (deterministic backbone), MC Dropout, Last-Layer Laplace, 
and GAPA (ours). Background shading indicates predictive confidence (darker = more confident); 
orange/yellow points show the two classes. \textbf{Key observation}: GAPA preserves the 
backbone's decision boundary (black line) exactly while adding  epistemic 
uncertainty that grows smoothly away from training data.}
    \label{fig:baselines_classification}
\end{figure*}

\section{Model Proposition}

\begin{figure}[t]
  \centering
  \includegraphics[width=\linewidth]{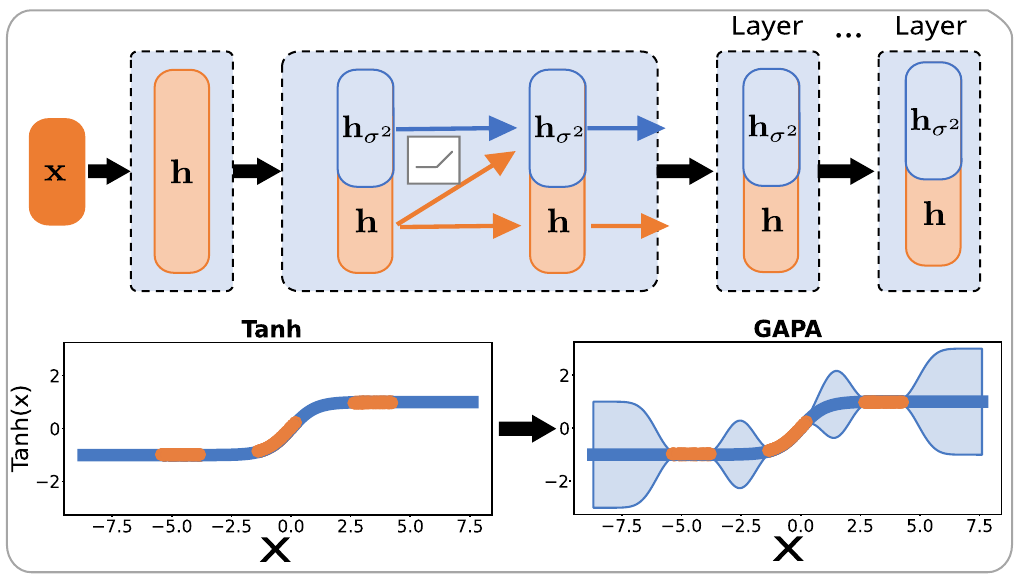}
\caption{GAPA overview. \textbf{Top:} GAPA leaves the network’s point predictions unchanged (mean-preserving activations) while propagating an additional epistemic variance signal to the output. \textbf{Bottom left:} deterministic $\tanh$ activation; orange points denote cached training activations. \textbf{Bottom right:} GAPA-$\tanh$, whose posterior mean matches $\tanh$ exactly; the shaded region shows $\pm 2$ standard deviations.}

  \label{fig:arch_gapa_layer}
\end{figure}

At a high level, GAPA augments a frozen neural network with activation-space uncertainty while strictly preserving its original deterministic predictions. Figure~\ref{fig:arch_gapa_layer} provides a structural overview. In the following sections, we formalize our uncertainty perspective (Sec.~\ref{sec:uq_perspective}), define the GP activation layer (Sec.~\ref{sec:gapa1}), introduce a scalable inference mechanism (Sec.~\ref{subsec:scalable}), and derive rules for single-pass variance propagation through deep architectures (Sec.~\ref{subsec:propagating_variance}).

\paragraph{Method pipeline.}
GAPA operates in two phases.
\textbf{(i) Offline collection:} we run a single forward pass of a reference input training set through the pre-trained backbone and cachepre-activations at selected layers.
Optionally, we compress the cache into a smaller inducing set via $k$-means, yielding inducing inputs that admit a variational inducing-point interpretation in the sense of~\citet{titsias2009variational}.
\textbf{(ii) Test-time inference:} we replace deterministic activations with Gaussian-process (GP) activations that return activation-space epistemic variances.
These variances are then propagated forward through the remaining frozen network using closed-form variance propagation rules, enabling deterministic single-pass predictive uncertainty without sampling, backpropagation, or retraining, while preserving the backbone's point predictions.

\subsection{Uncertainty Modeling Perspective}
\label{sec:uq_perspective}

We position GAPA by stating \emph{what is random} in the predictive distribution.
Let $\mathcal D=\{(x_n,y_n)\}_{n=1}^N$ denote the training data, and let $x$ be a test input with corresponding output $y$. We denote by $f(x)$ the latent predictor (e.g., a neural network or GP) evaluated at $x$. Predictive uncertainty is obtained by marginalizing the latent predictor:
\begin{equation}
p(y \mid x, \mathcal{D})
=
\int p\!\bigl(y \mid f(x)\bigr)\, p\!\bigl(f(x) \mid x,\mathcal{D}\bigr)\, d f(x).
\label{eq:Bayesian_modeling}
\end{equation}


\paragraph{Weight-space uncertainty.}
Weight-space methods (BNNs, Laplace) parameterize $f(x)=f(x;\mathbf w)$ and infer a posterior
$\mathbf w \sim p(\mathbf w \mid \mathcal D)$, inducing epistemic uncertainty via variability across plausible weights.

\paragraph{Activation-space uncertainty (GAPA).}
GAPA keeps the frozen weights deterministic and instead places epistemic uncertainty on the hidden-layer
\emph{activation} (e.g., ReLU outputs). For a chosen layer $\ell$, write
\[
f(x)=h_\ell\!\big(\boldsymbol\phi(\mathbf z_\ell(x))\big),
\]
where $\mathbf z_\ell(x)$ are layer-$\ell$ pre-activations, $\boldsymbol\phi$ is an element-wise nonlinearity, and
$h_\ell(\cdot)$ denotes the remaining frozen network.
We replace $\boldsymbol\phi$ with a GP activation $\mathbf g_\ell$ such that its posterior mean matches the original
activation, $\boldsymbol\mu_\ell(\mathbf z)=\boldsymbol\phi(\mathbf z)$ (Sec.~\ref{sec:gapa1}), thereby preserving the
backbone point prediction exactly.
Let $\mathbf a_\ell := \mathbf g_\ell(\mathbf z_\ell(x))$ be the resulting random activation. Then
\begin{equation}
p(y \mid x, \mathcal{D})
=
\int p\!\bigl(y \mid h_\ell(\mathbf a_\ell)\bigr)\,
p(\mathbf a_\ell \mid \mathbf z_\ell(x), \mathcal{D})\, d\mathbf a_\ell .
\label{eq:gapa_marginal}
\end{equation}
Uncertainty grows as test-time pre-activations move away from regions supported by the training data.

\subsection{Gaussian Process Activation Function}\label{sec:gapa1}
We define the core \textsc{GAPA} module: a drop-in replacement for a deterministic activation that (i) preserves the frozen backbone's point predictions exactly and (ii) returns a distance-aware epistemic variance in activation space.
\paragraph{Setup.}
Consider a frozen network and layer $\ell$ of width $d_\ell$.
Let
\[
\mathbf z_\ell = \mathbf W_\ell \mathbf h_{\ell-1} + \mathbf b_\ell \in \mathbb R^{d_\ell},
\qquad
\mathbf h_\ell = \boldsymbol\phi(\mathbf z_\ell),
\]
where $\boldsymbol\phi(\mathbf z)=(\phi(z_1),\dots,\phi(z_{d_\ell}))^\top$ is an element-wise nonlinearity (e.g.\ ReLU).

\paragraph{GP activation.}
GAPA replaces the deterministic activation $\boldsymbol\phi(\cdot)$ with a
vector-valued Gaussian process (GP)
\[
\mathbf g_\ell(\cdot):\mathbb R^{d_\ell}\to\mathbb R^{d_\ell},
\qquad
\mathbf g_\ell(\cdot)\sim\mathcal{GP}\!\big(\boldsymbol m_\ell(\cdot),\,\mathbf K_\ell(\cdot,\cdot)\big).
\]
For any input $\mathbf z$, this GP induces a \emph{Gaussian marginal distribution}
over the activation vector
\[
\mathbf a_\ell := \mathbf g_\ell(\mathbf z),
\]
with posterior mean $\boldsymbol\mu_\ell(\mathbf z)$ and covariance
$\mathbf K_\ell(\mathbf z,\mathbf z)$.

For scalability, we use a diagonal output kernel,
\[
\mathbf K_\ell(\mathbf z,\mathbf z')
=
\mathrm{diag}\!\big(k_{\ell,1}(\mathbf z,\mathbf z'),\dots,k_{\ell,d_\ell}(\mathbf z,\mathbf z')\big),
\]
which is equivalent to modeling each activation dimension as an independent scalar GP.
Importantly, GAPA never samples from these distributions: all uncertainty is propagated
analytically via closed-form moment propagation.


\paragraph{Data collation for the GP.} 
We run a single forward pass of the backbone's training data and cache the resulting pre-activations
\[
\tilde{\mathbf Z}_\ell=\{\tilde{\mathbf z}^{(m)}_\ell\}_{m=1}^M,\qquad \tilde{\mathbf z}^{(m)}_\ell\in\mathbb R^{d_\ell}.
\]
At each cached input we form \emph{noiseless pseudo-observations} by evaluating the original activation,
\[
\tilde{\mathbf y}^{(m)}_\ell \;=\; \boldsymbol\phi\!\big(\tilde{\mathbf z}^{(m)}_\ell\big)\in\mathbb R^{d_\ell},
\qquad m=1,\dots,M,
\]
and collect them as $\tilde{\mathbf Y}_\ell\in\mathbb R^{M\times d_\ell}$. We condition the GP on the dataset
\[
\mathcal D_\ell=\big\{(\tilde{\mathbf z}^{(m)}_\ell,\tilde{\mathbf y}^{(m)}_\ell)\big\}_{m=1}^M,
\]
using a small jitter/noise term $\sigma_n^2$ for numerical stability.

\paragraph{Mean preservation.}
We choose the GP prior mean to match the original activation,
\[
\boldsymbol m_\ell(\mathbf z)=\boldsymbol\phi(\mathbf z).
\]
Under standard GP regression, the posterior mean at $\mathbf z_\ell^\ast$ can be written as
\[
\boldsymbol\mu_\ell(\mathbf z_\ell^\ast)
=
\boldsymbol m_\ell(\mathbf z_\ell^\ast)
+
\mathbf A_\ell(\mathbf z_\ell^\ast)\,
\underbrace{\big(\tilde{\mathbf Y}_\ell-\boldsymbol m_\ell(\tilde{\mathbf Z}_\ell)\big)}_{=\,\mathbf 0},
\]
where $\mathbf A_\ell(\mathbf z_\ell^\ast)=\mathbf K_\ell(\mathbf z_\ell^\ast,\tilde{\mathbf Z}_\ell)
\big(\mathbf K_\ell(\tilde{\mathbf Z}_\ell,\tilde{\mathbf Z}_\ell)+\sigma_n^2 \mathbf I\big)^{-1}$.
Since $\tilde{\mathbf Y}_\ell=\boldsymbol\phi(\tilde{\mathbf Z}_\ell)=\boldsymbol m_\ell(\tilde{\mathbf Z}_\ell)$ by construction, the residual term is identically zero and hence
\[
\boldsymbol\mu_\ell(\mathbf z_\ell^\ast)=\boldsymbol m_\ell(\mathbf z_\ell^\ast)=\boldsymbol\phi(\mathbf z_\ell^\ast)
\qquad\text{for all }\mathbf z_\ell^\ast.
\]

Therefore, the \emph{GAPA activation} has posterior mean equal to the original activation function.
Substituting $\mathbf h_\ell=\mathbf g_\ell(\mathbf z_\ell)$ into a frozen network preserves the backbone's point predictions exactly.
The remaining posterior covariance of the GP activation, which quantifies epistemic uncertainty in activation space,
is specified next.

\paragraph{Posterior covariance.}
While the posterior mean is unchanged, the posterior covariance is non-zero and diagonal. For each neuron $i$,
\begin{equation*}
\begin{aligned}
[\boldsymbol\Sigma_\ell(\mathbf z_\ell^\ast)]_{ii}
&=
k_{\ell,i}(\mathbf z_\ell^\ast,\mathbf z_\ell^\ast)
-
\Big(
\mathbf k_{\ell,i}(\mathbf z_\ell^\ast)^\top \\
&\qquad
\big(\mathbf K_{\ell,i} + \sigma_n^2 \mathbf I\big)^{-1}
\mathbf k_{\ell,i}(\mathbf z_\ell^\ast)
\Big),
\end{aligned}
\end{equation*}
where $[\mathbf K_{\ell,i}]_{mn}=k_{\ell,i}(\tilde{\mathbf z}^{(m)}_\ell,\tilde{\mathbf z}^{(n)}_\ell)$ and
$[\mathbf k_{\ell,i}(\mathbf z_\ell^\ast)]_{m}=k_{\ell,i}(\mathbf z_\ell^\ast,\tilde{\mathbf z}^{(m)}_\ell)$.
Collecting all neuron-wise variances yields the layer covariance
\begin{equation}\label{eq:gapa_cov}
\boldsymbol\Sigma_\ell(\mathbf z_\ell^\ast)
=
\mathrm{diag}\!\big(
[\boldsymbol\Sigma_\ell(\mathbf z_\ell^\ast)]_{11},
\ldots,
[\boldsymbol\Sigma_\ell(\mathbf z_\ell^\ast)]_{d_\ell d_\ell}
\big).
\end{equation}

The resulting layer covariance satisfies $\boldsymbol\Sigma_\ell(\mathbf z_\ell^\ast) \in \mathbb R^{d_\ell \times d_\ell}$ and captures neuron-wise uncertainty that increases as $\mathbf z_\ell^\ast$ departs from the cached pre-activations.


\paragraph{Why diagonal covariance?}
We model neurons as conditionally independent (i.e., a diagonal output covariance) for tractability. A full multi-output covariance would require storing and propagating dense $d_\ell \times d_\ell$ matrices, which is prohibitive in memory and compute for modern wide networks.
The diagonal approximation is standard in scalable Bayesian models and is sufficient in our setting to capture activation-level epistemic uncertainty, as evidenced by our strong empirical results (Sec.~\ref{sec:results}).

\subsection{Local Inducing-Point Approximation}
\label{subsec:scalable}

At layer $\ell$, GAPA conditions $d_\ell$ independent scalar GPs on a cache of $N_\ell$ pre-activations
$\tilde{\mathbf Z}_\ell\in\mathbb R^{N_\ell\times d_\ell}$ obtained from a single offline forward pass of the
backbone’s training data (or a subset thereof) through the frozen network.
Even with a diagonal output kernel, exact GP conditioning in Eq.~\eqref{eq:gapa_cov} requires solving an
$N_\ell\times N_\ell$ linear system per neuron, yielding $\mathcal O(d_\ell N_\ell^3)$ time and
$\mathcal O(d_\ell N_\ell^2)$ memory, which is prohibitive for modern networks.
We therefore use a two-stage approximation:
(i) an \emph{offline} global inducing set that compresses the cache, and
(ii) \emph{test-time} \emph{local} conditioning on the $K$ nearest inducing points to each query pre-activation.

\paragraph{Stage 1: Inducing-point construction (offline).}
We construct inducing inputs $\mathbf Z_\ell\in\mathbb R^{M_\ell\times d_\ell}$ with $M_\ell\ll N_\ell$
by running $k$-means on the cached pre-activations $\tilde{\mathbf Z}_\ell$ and taking the $M_\ell$ centroids.
These inducing points provide a compressed representation of the training-data activation cache
used for scalable GP inference. A formal connection between this procedure and variational inducing-point GPs
is provided in Appendix~\ref{app:variational_inducing}, and Appendix~\ref{Appedix:n_inducing_points} studies sensitivity to the number of inducing points.

\paragraph{Stage 2: Local $K$-nearest-neighbour conditioning (test time).}
At test time, using all $M_\ell$ inducing points would require an $M_\ell\times M_\ell$ solve per query, which is still
expensive. We therefore adopt a local GP approximation \citep{gramacy2015local}: for each query pre-activation
$\mathbf z^\ast_\ell\in\mathbb R^{d_\ell}$, we form a small local subset of inducing inputs by $K$-nearest neighbours in
activation space.

Concretely, let $\mathcal N_K(\mathbf z^\ast_\ell)\subseteq\{1,\dots,M_\ell\}$ denote the indices of the $K$ nearest inducing points
to $\mathbf z^\ast_\ell$ under Euclidean distance, and define
\[
\mathbf Z_{\ell,K}(\mathbf z^\ast_\ell)
\;=\;
\mathbf Z_\ell^{\mathcal N_K(\mathbf z^\ast_\ell)}
\;\in\; \mathbb R^{K\times d_\ell},
\]
where $\mathbf Z_\ell^{\mathcal N_K(\cdot)}$ denotes the submatrix formed by selecting the corresponding rows of $\mathbf Z_\ell$. Neighbour retrieval is performed using FAISS \citep{douze2025faiss} with approximate nearest-neighbour search,
yielding sublinear query time  in practice. Appendix~\ref{app:knn_sweep} further studies sensitivity to $K$.


\paragraph{Approximate posterior covariance.}
Given the local inducing inputs $\mathbf Z_{\ell,K}(\mathbf z^\ast_\ell)\in\mathbb R^{K\times d_\ell}$,
we approximate the posterior variance of neuron $i\in\{1,\dots,d_\ell\}$ by applying the standard GP
conditional-variance formula restricted to this local subset:
\begin{equation}
\begin{aligned}
\sigma^2_{\ell,i}(\mathbf z^\ast_\ell)
\;&\approx\;
k_{\ell,i}(\mathbf z^\ast_\ell,\mathbf z^\ast_\ell) \\
&\quad-
\mathbf k_{\ell,i}(\mathbf z^\ast_\ell)^\top
\big(\mathbf K_{\ell,i} + \sigma_n^2 \mathbf I\big)^{-1}
\mathbf k_{\ell,i}(\mathbf z^\ast_\ell),
\end{aligned}
\label{eq:local_knn_var}
\end{equation}
where $\sigma_n^2$ is a small jitter term for numerical stability.
Let $\{\mathbf z^{(m)}_{\ell,K}\}_{m=1}^K$ denote the rows of $\mathbf Z_{\ell,K}(\mathbf z^\ast_\ell)$.
Then the $K\times K$ kernel matrix and $K$-vector are defined as
\[
[\mathbf K_{\ell,i}]_{mn}
=
k_{\ell,i}(\mathbf z^{(m)}_{\ell,K},\mathbf z^{(n)}_{\ell,K}),
\text{}
[\mathbf k_{\ell,i}(\mathbf z^\ast_\ell)]_{m}
=
k_{\ell,i}(\mathbf z^\ast_\ell,\mathbf z^{(m)}_{\ell,K}).
\]
Collecting neuron-wise variances yields the diagonal layer covariance
\[
\boldsymbol\Sigma_\ell(\mathbf z^\ast_\ell)
=
\mathrm{diag}\!\big(
\sigma^2_{\ell,1}(\mathbf z^\ast_\ell),\ldots,\sigma^2_{\ell,d_\ell}(\mathbf z^\ast_\ell)
\big)\in \mathbb R^{d_\ell\times d_\ell}.
\]

\paragraph{Conservative uncertainty (monotonicity).}
Intuitively, conditioning a GP on fewer points cannot reduce posterior uncertainty.
Fix kernel hyperparameters and observation noise $\sigma_n^2>0$, and let $A\subseteq C$ be two conditioning sets.
Then for any test input $\mathbf z^\ast$,
\[
\operatorname{Var}\!\big(f(\mathbf z^\ast)\mid \mathcal D_A\big)
\;\ge\;
\operatorname{Var}\!\big(f(\mathbf z^\ast)\mid \mathcal D_C\big).
\]
Consequently, conditioning on the local subset $\mathbf Z_{\ell,K}(\mathbf z^\ast_\ell)\subseteq \mathbf Z_\ell$
cannot underestimate epistemic uncertainty relative to using the full inducing set.
A formal statement and proof are provided in Appendix~\ref{app:conservative};
we additionally ablate sensitivity to $K$ in Appendix~\ref{app:knn_sweep}.

\paragraph{Computational complexity.}
Offline preprocessing comprises caching $\tilde{\mathbf Z}_\ell$ (one forward pass), constructing $\mathbf Z_\ell$ via
$k$-means, and building a FAISS index over $\mathbf Z_\ell$, where typically $K \ll M_\ell \ll N_\ell$.
At test time, each query requires (i) neighbour search in $\mathcal O(\log M_\ell)$ and
(ii) solving a $K\times K$ system in Eq.~\eqref{eq:local_knn_var}.
With fixed $K$ (we use $K=50$), the per-query linear algebra is constant-size and the dominant dependence on the
inducing set size is $\mathcal O(\log M_\ell)$; memory is linear in $M_\ell$. See Table~\ref{tab:complexity} for a summary.

\begin{table}[t]
\centering
\small
\setlength{\tabcolsep}{4pt}
\caption{Computational complexity per layer $\ell$.
Exact GP refers to conditioning on all $N_\ell$ cached activations.
$I_{\text{kmeans}}$ denotes the number of $k$-means iterations.}
\label{tab:complexity}

\resizebox{\columnwidth}{!}{%
\begin{tabular}{lcc}
\toprule
\textbf{Operation} & \textbf{Exact GP} & \textbf{GAPA (Ours)} \\
\midrule
Preprocessing
& $\mathcal O(N_\ell^3)$
& $\mathcal O(N_\ell I_{\text{kmeans}}) + \mathcal O(M_\ell \log M_\ell)$ \\

Inference (per query)
& $\mathcal O(N_\ell^2)$
& $\mathcal O(\log(M_\ell)) + \mathcal O(K^3)$ \\

Memory
& $\mathcal O(N_\ell^2)$
& $\mathcal O(M_\ell)$ \\
\bottomrule
\end{tabular}
}
\end{table}

\subsection{Variance Propagation Through the Network}
\label{subsec:propagating_variance}

\paragraph{Gaussian flow intuition.}
Once we replace deterministic activations with GAPA modules, the forward pass no longer carries only a point value.
Instead, at any layer we track a \emph{Gaussian summary} of the hidden state: a mean vector and a diagonal covariance
matrix $\boldsymbol\Sigma_{\mathbf h}$.
Concretely, each GAPA layer maps a (possibly uncertain) pre-activation input to a Gaussian output,
so uncertainty can be \emph{propagated forward} through the remaining frozen layers.
Specialized rules for architectures such as self-attention and RMSNorm are provided in
Appendix~\ref{app:propagation}.

\paragraph{Notation.}
We denote by $\boldsymbol\mu_{\mathbf h}$ the mean of a vector-valued random variable $\mathbf h$ and by
$\boldsymbol\Sigma_{\mathbf h}$ its covariance matrix. We write the \emph{variance vector}
$\mathbf v_{\mathbf h}:=\mathrm{diag}(\boldsymbol\Sigma_{\mathbf h})$, whose entries are
$[\mathbf v_{\mathbf h}]_i=\mathrm{Var}(h_i)$.
For vectors $\mathbf a,\mathbf b$ of the same size, $\mathbf a\odot \mathbf b$ denotes the \emph{Hadamard}
(element-wise) product, and $\mathbf a^{\odot 2}:=\mathbf a\odot \mathbf a$.

\paragraph{(i) Linear layers.}
Consider a linear transformation $\mathbf z=\mathbf W\mathbf h+\mathbf b$ where $\mathbf h$ has diagonal covariance.
Under the diagonal approximation, the output variance remains diagonal and the variance vector propagates as
\begin{equation}
\mathbf v_{\mathbf z}
\;=\;
(\mathbf W\odot \mathbf W)\,\mathbf v_{\mathbf h}.
\label{eq:var_linear}
\end{equation}
Intuitively, each output coordinate $z_i=\sum_j W_{ij}h_j$ is a weighted sum of independent components, so its variance
is the sum of squared weights times input variances.

\paragraph{(ii) Element-wise nonlinearities (delta method).}
We use first-order moment propagation (delta method): means follow the deterministic forward pass,
while variances are updated by local linearization. Let $y=g(z)$ where $z\sim\mathcal N(\mu,\sigma^2)$ and $g$ is a scalar nonlinearity (e.g., ReLU, $\tanh$).
We linearize around $\mu$:
\[
g(z)\approx g(\mu)+g'(\mu)(z-\mu).
\]
This gives the standard delta-method moments:
\begin{equation}
\mathbb E[y]\approx g(\mu),
\qquad
\mathrm{Var}(y)\approx (g'(\mu))^2\sigma^2.
\label{eq:delta_scalar}
\end{equation}
Applied element-wise to vectors $\mathbf y=g(\mathbf z)$ with diagonal variances, yields
\begin{equation}
\boldsymbol\mu_{\mathbf y}=g(\boldsymbol\mu_{\mathbf z}),
\qquad
\mathbf v_{\mathbf y}\approx \big(g'(\boldsymbol\mu_{\mathbf z})\big)^{\odot 2}\odot \mathbf v_{\mathbf z}.
\label{eq:delta_vector}
\end{equation}
Intuitively, a nonlinearity either amplifies or attenuates uncertainty depending on its local slope.

\paragraph{(iii) Stacking GAPA layers (noisy-input correction).}
When we place multiple GAPA layers in a network, the input to a downstream GAPA layer is no longer a point
$\mathbf z^\ell$ but a \emph{distribution} (summarized as a Gaussian with mean $\boldsymbol\mu_{\mathbf z}$ and variance
vector $\mathbf v_{\mathbf z}$). Standard GP conditioning assumes a deterministic test input; here the test-time input
is itself uncertain. Intuitively, even if the GP were evaluated at the same mean location, uncertainty in the input
location induces additional variability in the output whenever the GP mean changes with $\mathbf z$.
Following the noisy-input GP (NIGP) approximation \citep{mchutchon2011Gaussian}, we capture this effect by adding a
first-order correction term.

Concretely, for neuron $i$ at layer $\ell$, we evaluate the local inducing-point posterior variance at the mean input
and add the NIGP correction
\begin{equation*}
\sigma^2_{\ell,i}(\mathbf z^\ell)
\;\approx\;
\underbrace{\sigma^2_{\text{epi},\ell,i}\!\big(\boldsymbol\mu_{\mathbf z}\big)}_{\text{epistemic}}
\;+\;
\underbrace{\lambda_{\ell,i}\!\big(\boldsymbol\mu_{\mathbf z}\big)}_{\text{input uncertainty}}
\;+\;
\underbrace{\sigma^2_{y,i}}_{\text{aleatoric}},
\end{equation*}

where $\sigma^2_{\text{epi},\ell,i}(\boldsymbol\mu_{\mathbf z})$ is the local GP variance from
Eq.~\eqref{eq:local_knn_var} computed using the $K$-NN subset selected at $\boldsymbol\mu_{\mathbf z}$, and
\begin{equation*}
\lambda_{\ell,i}\!\big(\boldsymbol\mu_{\mathbf z}\big)
=
\big(\nabla_{\mathbf z^\ell}\mu_{\ell,i}(\boldsymbol\mu_{\mathbf z})\big)^\top
\mathrm{diag}(\mathbf v_{\mathbf z})\,
\big(\nabla_{\mathbf z^\ell}\mu_{\ell,i}(\boldsymbol\mu_{\mathbf z})\big)
\end{equation*}
is the NIGP correction induced by input uncertainty.
In our mean-preserving element-wise setting $\mu_{\ell,i}(\mathbf z)=\phi_\ell(z_i)$, this simplifies to
$\lambda_{\ell,i}(\boldsymbol\mu_{\mathbf z})=(\phi'_\ell(\mu_{\mathbf z,i}))^2\,v_{\mathbf z,i}$.
We set $\sigma^2_{y,i}=0$ for classification; for regression we add a learned heteroscedastic noise head
(Sec.~\ref{sec:hyperparameters}).


\subsection{Hyperparameter Strategy}
\label{sec:hyperparameters}

A key design choice in GAPA is to \emph{fix} GP hyperparameters rather than optimize them.

\paragraph{Empirical (post-hoc) hyperparameters.} GAPA is designed as a \emph{post-hoc} module for pretrained neural networks: after standard training, we attach GAPA using only one forward pass on the backbone’s input training data (or a subset thereof) to cache activations, without any additional backpropagation, fine-tuning or labels. Accordingly, we set GP hyperparameters \emph{once} from simple empirical statistics of the cached pre-activations. Because hyperparameters are estimated from activation statistics rather than a task-specific objective, the method is task-agnostic: the same procedure applies across settings (e.g., regression, classification, token prediction, and segmentation). Caching is a one-off offline step, and test-time inference remains unchanged thereafter. Details of the empirical hyperparameter construction are given in Appendix~\ref{Appendix:empirical_kernel}.

We use an RBF kernel $k_{\ell,i}(\mathbf z,\mathbf z')=c_i^2\exp\!\big(-\|\mathbf z-\mathbf z'\|^2/(2\ell_i^2)\big)$,
with $(c_i^2,\ell_i)$ set from cached activation statistics and a small jitter.
Details are given in Appendix~\ref{Appendix:empirical_kernel}, with downstream likelihoods for classification and
regression described in Appendix~\ref{app:laplace_bridge} and Appendix~\ref{Appendix:gapa_train}.

\section{Results}
\label{sec:results}

We evaluate GAPA as a \emph{post-hoc} uncertainty module for \emph{frozen} pretrained backbones across
regression (Sec.~\ref{sec:results_regression}),
classification (Sec.~\ref{sec:classification_results}),
image segmentation (Appendix Sec.~\ref{sec:segmentation}),
and language modeling (Sec.~\ref{app:llama}). Ablations on layer placement, inducing-set size \(M\), local subset size \(K\), and uncertainty-score choices
are provided in Appendix~\ref{sec:ablation}.


\subsection{Regression}
\label{sec:results_regression}

\paragraph{Setup and baselines.}
We evaluate on three regression benchmarks using the original train/test splits:
YEAR Prediction MSD, Airline \citep{hensman2013gaussian}, and Taxi \citep{salimbeni2017doubly}.
We compare against: \textbf{MAP} (backbone model), \textbf{Dropout} (MC Dropout with multiple stochastic forward passes; \citep{gal2016dropout}),
\textbf{Ensemble} (independently trained models; \citep{lakshminarayanan2017simple}), and post-hoc last-layer Bayesian baselines
(\textbf{LLA} variants, \textbf{ELLA}, \textbf{VaLLA}; \citep{ortega2023variational}).

\paragraph{Metrics.}
Performance is evaluated using Negative Log-Likelihood (NLL), Continuous Ranked Probability Score (CRPS),
and the Centered Quantile Metric (CQM), with definitions in Appendix~\ref{app:regression_metrics}.

\begin{table}[t]
    \centering
    \caption{Results on regression datasets.
    Best values are in \textcolor{purple}{purple}, and
    second-best in \textcolor{teal}{teal}.
    An asterisk (*) indicates a last-layer LLA variant.
    Results are averages over 5 random seeds;
    standard deviations ($<10^{-3}$ in all cases) are omitted for brevity.
    The full table with stds can be found in Table~\ref{tab:results_std} in the Appendix.}
    \label{tab:results} 
    \vspace{0.3em}

    \scriptsize
    \setlength{\tabcolsep}{2.4pt}
    \renewcommand{\arraystretch}{1.02}

    \begin{tabular}{l|ccc|ccc|ccc}
      \toprule
      \multirow{2}{*}{\textbf{Model}}
        & \multicolumn{3}{c|}{\textbf{Airline}}
        & \multicolumn{3}{c|}{\textbf{Year}}
        & \multicolumn{3}{c}{\textbf{Taxi}} \\
      \cmidrule(lr){2-4} \cmidrule(lr){5-7} \cmidrule(l){8-10}
        & \textbf{NLL} & \textbf{CRPS} & \textbf{CQM}
        & \textbf{NLL} & \textbf{CRPS} & \textbf{CQM}
        & \textbf{NLL} & \textbf{CRPS} & \textbf{CQM} \\
      \midrule
      MAP
        & 5.121 & 18.695 & 0.148
        & 3.673 & 5.023 & 0.134
        & 3.775 & \textcolor{teal}{3.755} & 0.211 \\
      LLA Diag
        & 5.125 & 18.648 & 0.143
        & 3.647 & 4.917 & 0.088
        & 3.722 & 3.990 & 0.257 \\
      LLA KFAC
        & 5.127 & 18.631 & 0.142
        & 3.648 & 4.915 & 0.086
        & 3.706 & 3.986 & 0.256 \\
      LLA*
        & 5.127 & \textcolor{teal}{18.631} & 0.141
        & 3.648 & 4.915 & 0.086
        & 3.726 & 3.985 & 0.256 \\
      LLA*KFAC
        & 5.127 & 18.631 & 0.141
        & 3.648 & \textcolor{teal}{4.914} & 0.086
        & 3.726 & 3.985 & 0.256 \\
      ELLA
        & 5.388 & 21.671 & 0.413
        & 4.020 & 6.049 & 0.424
        & 3.885 & \textcolor{purple}{3.680} & 0.219 \\
      VaLLA100
        & \textcolor{teal}{4.963} & 18.814 & \textcolor{teal}{0.099}
        & 3.515 & 5.004 & 0.047
        & 3.235 & 3.999 & 0.149 \\
      VaLLA200
        & 4.965 & 18.788 & \textcolor{purple}{0.098}
        & \textcolor{teal}{3.485} & 4.970 & \textcolor{teal}{0.041}
        & \textcolor{teal}{3.232} & 3.979 & \textcolor{teal}{0.142} \\
      Dropout
        & 5.102 & 19.066 & 0.938
        & 3.689 & 5.128 & 0.939
        & 3.849 & 4.592 & 0.951 \\
      Ensemble
        & 5.053 & 18.205 & 0.933
        & 3.639 & 4.833 & 0.938
        & 3.631 & 3.384 & 0.961 \\
      \textbf{GAPA}
        & \textcolor{purple}{4.946} & \textcolor{purple}{18.068} & 0.103
        & \textcolor{purple}{3.470} & \textcolor{purple}{4.663} & \textcolor{purple}{0.014}
        & \textcolor{purple}{3.112} & 4.035 & \textcolor{purple}{0.104} \\
      \bottomrule
    \end{tabular}

    \vspace{0.3em}
\end{table}

Table~\ref{tab:results} shows a consistent pattern across all regression benchmarks.
GAPA achieves the best NLL on all three datasets, indicating the strongest overall fit of the predictive distribution.
Moreover, GAPA attains the best CQM overall, demonstrating consistently well-calibrated predictive quantiles.
CRPS follows the same trend on Airline and Year, with a small degradation on Taxi, where ELLA perform best.

\subsection{Classification}
\label{sec:classification_results}

We now evaluate classification performance in terms of accuracy and calibration, and assess OOD detection via AUC using
predictive entropy and BALD. Metric definitions are provided in Appendix~\ref{app:classification_metrics}.

\paragraph{Baselines.}
We compare against (i) the deterministic backbone (\textbf{MAP});
(ii) sampling-based baselines requiring multiple forward passes or multiple trainings
(\textbf{MC Dropout} \citep{gal2016dropout}, \textbf{Deep Ensembles} \citep{lakshminarayanan2017simple});
(iii) post-hoc last-layer Bayesian baselines that Bayesianize only the head
(\textbf{LLA}/\textbf{ELLA}/\textbf{VaLLA}; \citep{ortega2023variational});
and (iv) distance-/feature-based alternatives including \textbf{SNGP} \citep{liu2020simple}, \textbf{MF-VI} \citep{blundell2015weight},
and a \textbf{subset GP} baseline based on the GP interpretation of local linearization.
We also report \textbf{linear probing} \citep{alain2016understanding} as a lightweight head-only baseline, and \textbf{DDU} \citep{mukhoti2023deep}
as a post-hoc feature-density uncertainty estimator.

\begin{table*}[t]
  \centering
\caption{
Classification performance and runtime on MNIST and Fashion-MNIST.
Accuracy (ACC), negative log-likelihood (NLL), expected calibration error (ECE),
OOD detection, and BALD are reported.
Best results are shown in \textcolor{purple}{purple} and second-best in \textcolor{teal}{teal}.
All results are averaged over 5 random seeds; standard deviations ($<10^{-3}$) are omitted.
Training and test times are wall-clock seconds (K = $10^3$).
The full table with standard deviations is provided in Table~\ref{tab:results_std} in the Appendix.
}
  \label{tab:results_classif_runtime_internal}

  \vspace{-0.4em}
  \footnotesize
  \setlength{\tabcolsep}{3pt}
  \renewcommand{\arraystretch}{1.00}

  \begin{adjustbox}{width=0.90\textwidth}

   \begin{tabular}{l|ccccc|cc|ccccc|cc}
    
  \toprule
    \multirow{3}{*}{\textbf{Model}}
      & \multicolumn{7}{c|}{\textbf{MNIST}}
      & \multicolumn{7}{c}{\textbf{FMNIST}} \\
    \cmidrule(lr){2-8} \cmidrule(lr){9-15}
    
      & \multicolumn{5}{c|}{\textbf{Metrics}}
      & \multicolumn{2}{c|}{\textbf{Time (s)}}
      & \multicolumn{5}{c|}{\textbf{Metrics}}
      & \multicolumn{2}{c}{\textbf{Time (s)}} \\
    \cmidrule(lr){2-6} \cmidrule(lr){7-8}
    \cmidrule(lr){9-13} \cmidrule(lr){14-15}
    
      & \textbf{ACC} & \textbf{NLL} & \textbf{ECE} & \textbf{OOD} & \textbf{BALD}
      & \textbf{Train} & \textbf{Test}
      & \textbf{ACC} & \textbf{NLL} & \textbf{ECE} & \textbf{OOD} & \textbf{BALD}
      & \textbf{Train} & \textbf{Test} \\
    \midrule


    MAP (backbone)
      & \textcolor{teal}{0.978} & \textcolor{purple}{0.068} & \textcolor{purple}{0.005} & 0.919 & 0.919
      & --- & 1.24
      & 0.859 & 0.392 & \textcolor{purple}{0.007} & 0.846 & 0.821
      & --- & 1.20 \\

    LLA Diag
      & 0.976 & 0.177 & 0.105 & 0.932 & 0.941
      & 2.34K & \textcolor{teal}{2.39}
      & 0.856 & 0.421 & 0.057 & 0.872 & 0.873
      & 2.34K & \textcolor{teal}{2.22} \\

    LLA KFAC
      & \textcolor{teal}{0.978} & 0.102 & 0.042 & \textcolor{purple}{0.971} & \textcolor{teal}{0.971}
      & 130.0 & 2.85
      & 0.858 & 0.395 & 0.020 & 0.909 & \textcolor{teal}{0.970}
      & 129.9 & 2.84 \\

    LLA*
      & \textcolor{teal}{0.978} & \textcolor{teal}{0.070} & \textcolor{teal}{0.009} & 0.924 & 0.924
      & 42.0 & 4.6
      & 0.859 & 0.395 & 0.019 & 0.850 & 0.716
      & 42.0 & 4.7 \\

    LLA* KFAC
      & \textcolor{purple}{0.979} & \textcolor{teal}{0.070} & \textcolor{teal}{0.009} & 0.923 & 0.928
      & \textcolor{purple}{31.2} & 17.6
      & 0.859 & 0.394 & 0.017 & 0.849 & 0.717
      & \textcolor{purple}{31.2} & 17.4 \\

    ELLA
      & \textcolor{teal}{0.978} & \textcolor{purple}{0.068} & \textcolor{purple}{0.005} & 0.919 & 0.912
      & 821.8 & 148.7
      & 0.859 & 0.392 & \textcolor{purple}{0.007} & 0.846 & 0.765
      & 827.1 & 149.9 \\

    VaLLA 100
      & \textcolor{teal}{0.978} & \textcolor{purple}{0.068} & \textcolor{purple}{0.005} & 0.919 & 0.934
      & 2.19K & 16.4
      & \textcolor{teal}{0.865} & 0.382 & 0.019 & 0.925 & 0.963
      & 495.1 & 16.4 \\

    VaLLA 200
      & \textcolor{teal}{0.978} & \textcolor{purple}{0.068} & \textcolor{purple}{0.005} & 0.919 & 0.934
      & 3.43K & 18.3
      & \textcolor{purple}{0.867} & \textcolor{teal}{0.378} & 0.020 & \textcolor{teal}{0.937} & \textcolor{teal}{0.970}
      & 767.7 & 19.3 \\

    Linear Probing
      & 0.977 & 0.117 & 0.015 & 0.884 & 0.883
      & 2.78K & 3.6
      & 0.858 & 0.395 & 0.048 & 0.785 & 0.776
      & 2.64K & \textcolor{teal}{3.8} \\

    GPP
      & \textcolor{teal}{0.978} & 1.648 & 0.784 & 0.934 & 0.904
      & 5.79K & 23.5K
      & 0.857 & 1.716 & 0.692 & 0.867 & 0.962
      & 5.57K & 2.27K \\

    Dropout
      & \textcolor{teal}{0.978} & 0.072 & \textcolor{teal}{0.009} & 0.923 & 0.944
      & --- & 4.3
      & 0.858 & 0.393 & \textcolor{teal}{0.009} & 0.850 & 0.911
      & --- & 4.3 \\

    Ensemble
      & \textcolor{purple}{0.979} & \textcolor{teal}{0.069} & 0.038 & 0.936 & 0.962
      & --- & 11.9
      & 0.859 & \textcolor{purple}{0.373} & 0.041 & 0.863 & 0.938
      & --- & 11.9 \\

    DDU (est.)
      & \textcolor{teal}{0.978} & \textcolor{purple}{0.068} & \textcolor{purple}{0.005} & 0.921 & 0.919
      & 202.4 & 8.2
      & 0.859 & 0.392 & \textcolor{purple}{0.007} & 0.876 & \textcolor{teal}{0.983}
      & 202.3 & 8.2 \\

    \textbf{GAPA-Diag}
      & \textcolor{teal}{0.978} & 0.073 & 0.016 & \textcolor{teal}{0.963} & \textcolor{purple}{0.976}
      & \textcolor{teal}{91.9} &  \textcolor{purple}{2.05}
      & 0.859 & 0.390 & \textcolor{teal}{0.009} & \textcolor{purple}{0.941} & \textcolor{purple}{0.993}
      & \textcolor{teal}{92.1} & \textcolor{purple}{2.05} \\

    \textbf{GAPA-Full}
      & \textcolor{teal}{0.978} & 0.072 & 0.013 & \textcolor{teal}{0.969} & \textcolor{purple}{0.983}
      & \textcolor{teal}{96.1} &  \textcolor{purple}{8.92}
      
      & 0.859 & 0.388 & \textcolor{teal}{0.009} & \textcolor{purple}{0.990} & \textcolor{purple}{0.997}
      & 96.1 & 8.91 \\

    \bottomrule
  \end{tabular}
  \end{adjustbox}

  \vspace{-0.6em}
\end{table*}

\begin{figure}[t]
  \centering
  \vspace{-0.4em}
  \includegraphics[width=\columnwidth]{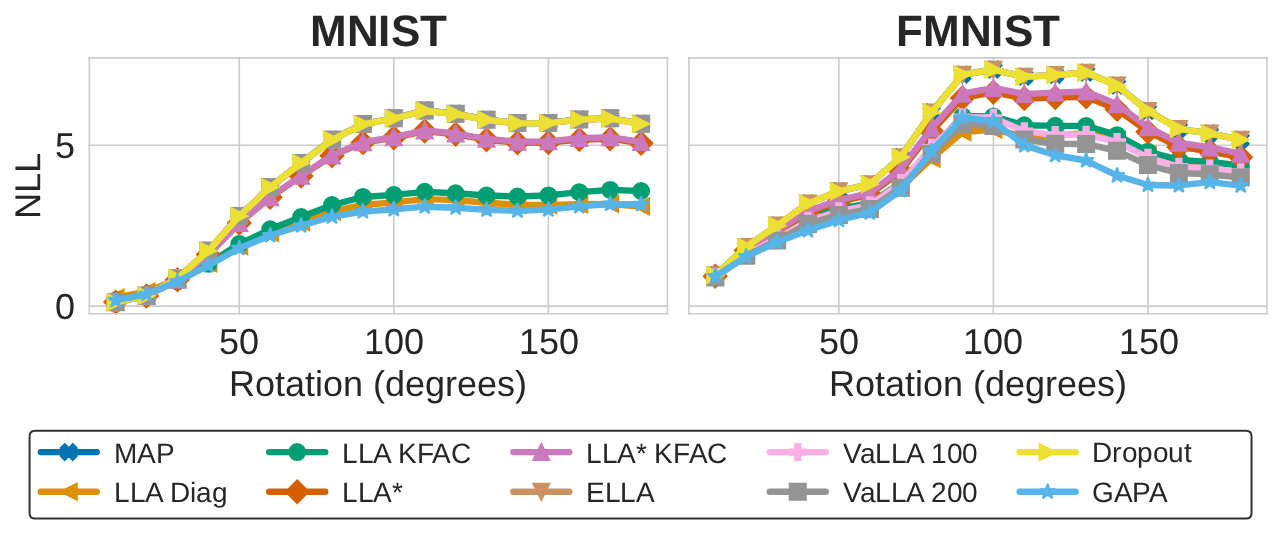}
  \vspace{-0.4em}
  \caption{Predictive NLL under rotation corruption for MNIST (left panel) and FMNIST (right panel); lower is better. Results are averaged over 5 random seeds.}
  \label{fig:combined_table_rotated_classification}
  \vspace{-0.7em}
\end{figure}

\begin{table}[t]
  \centering
    \caption{
    CIFAR-10 results with ResNet-44/56 backbones.
    Metrics include NLL ($\downarrow$), OOD detection ($\uparrow$), and train/test runtime.
    Best and second-best values are highlighted in \textcolor{purple}{purple} and \textcolor{teal}{teal}.
    Results are averaged over 5 seeds; stds ($<10^{-3}$) are omitted.
    Times are in seconds (K = $10^3$).
    See Table \ref{tab:resnet_full} for additional ResNet experiments.
    }
  \label{tab:resnet_44_56_main}
  \vspace{0.2em}

  \footnotesize
  \setlength{\tabcolsep}{2.2pt}
  \renewcommand{\arraystretch}{1.08}

  \begin{adjustbox}{width=\columnwidth}
  \begin{tabular}{l|cc|cc|cc|cc}
    \toprule
    \multirow{3}{*}{\textbf{Model}}
      & \multicolumn{4}{c|}{\textbf{ResNet-44}}
      & \multicolumn{4}{c}{\textbf{ResNet-56}} \\
    \cmidrule(lr){2-5}\cmidrule(lr){6-9}
    
      & \multicolumn{2}{c|}{\textbf{Metrics}}
      & \multicolumn{2}{c|}{\textbf{Time (s)}}
      & \multicolumn{2}{c|}{\textbf{Metrics}}
      & \multicolumn{2}{c}{\textbf{Time (s)}} \\
    \cmidrule(lr){2-3}\cmidrule(lr){4-5}
    \cmidrule(lr){6-7}\cmidrule(lr){8-9}
    
      & \textbf{NLL} & \textbf{OOD}
      & \textbf{Train} & \textbf{Test}
      & \textbf{NLL} & \textbf{OOD}
      & \textbf{Train} & \textbf{Test} \\
    \midrule

    MAP         & 0.275 & 0.885 & -- & 0.761 & 0.252 & 0.924 & -- & 0.949 \\
    MF-VI       & 0.206 & 0.890 & 1.63 & 1.03 & 0.188 & 0.929 & 1.97 & 1.18 \\
    SNGP        & 0.242 & 0.901 & 35.0 & 2.89 & 0.229 & 0.940 & 43.5 & 3.01 \\
    GP (subset) & 0.424 & 0.897 & 8.25K & 357 & 0.403 & 0.936 &8.42K & 382 \\
    LLA Diag    & 0.218 & 0.860 & 40.4 &  \textcolor{purple}{0.947} & 0.195 & 0.923 & 40.67 &  \textcolor{purple}{1.12} \\ 
    LLA KFAC    & 0.213 & 0.855 & 63.1 & 3.62 & 0.193 & 0.917 & 71.3 & 4.13  \\
    LLA*        & 0.237 & 0.895 & 4.98K & \textcolor{teal}{0.962}& 0.213 & 0.934 & 5.55K & \textcolor{teal}{1.16}\\
    LLA* KFAC   & 0.232 & 0.894 & 58.0 & 1.97 & 0.202 & 0.933 & 62.2 & 2.18 \\
    ELLA        & 0.204 & 0.885 & 1.12K & 78.3 & \textcolor{teal}{0.187} & 0.924 &  1.13K & 91.0   \\
    Sampled LLA & \textcolor{purple}{0.200} & 0.899 & 11.0K & 2.51K
                & \textcolor{purple}{0.185} & 0.944 & 14.6K & 2.84K \\
    VaLLA 200     & 0.201 & \textcolor{teal}{0.928} & 16.7K & 272.9 
                & 0.188 & \textcolor{purple}{0.960} & 26.3K & 363.8 \\
    \textbf{GAPA (ours)} 
                & 0.230 & \textcolor{purple}{0.931} & \textcolor{purple}{8.03} &   2.85
                & 0.230 & \textcolor{teal}{0.953} & \textcolor{purple}{10.29} &  3.30 \\
    \bottomrule 
  \end{tabular}
  \end{adjustbox}

  \vspace{0.1em}
\end{table}

\begin{figure*}[t]
  \centering
  \includegraphics[width=0.95\textwidth]{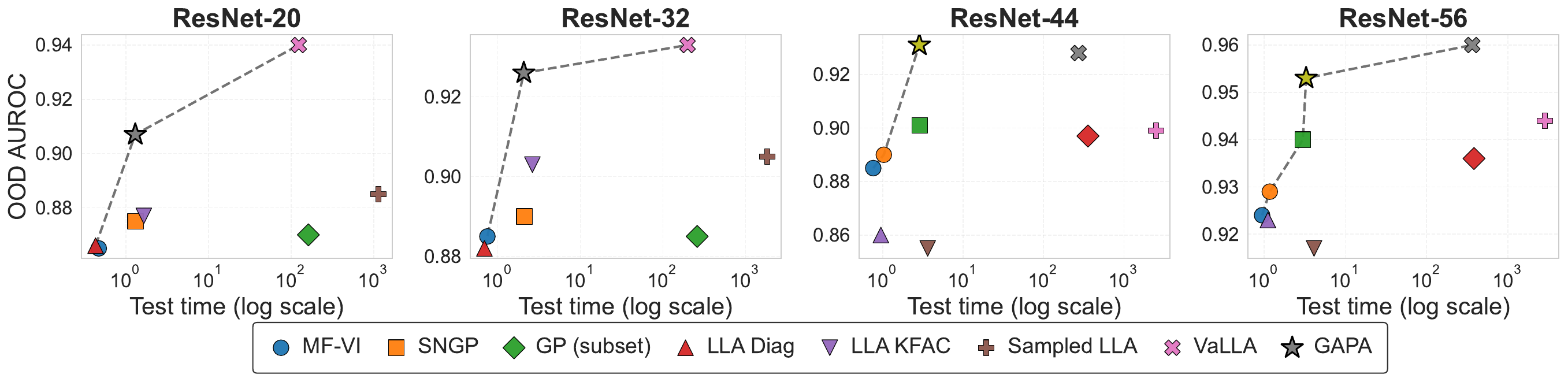}
  \caption{
  \textbf{OOD detection vs inference cost on CIFAR-10.}
  OOD AUROC is plotted against test-time inference cost (log scale) for ResNet backbones.
  Dashed lines indicate Pareto frontiers (higher OOD, lower cost).
  GAPA consistently lies on the frontier, achieving strong OOD performance
  at substantially lower inference cost than baselines.
  }
  \label{fig:ood_time_44_56}
\end{figure*}

\paragraph{MNIST / FMNIST (MLP).}
Following \citet{ortega2023variational}, we train a 2-layer MLP with 200 hidden units and $\tanh$ activations on MNIST \citep{lecun2002gradient}
and Fashion-MNIST \citep{xiao2017fashion}. We evaluate OOD detection by swapping datasets: Fashion-MNIST is treated as OOD for MNIST (MNIST$\rightarrow$FMNIST) and MNIST as OOD for Fashion-MNIST (FMNIST$\rightarrow$MNIST), reporting AUROC based on predictive entropy and BALD-OOD (Table~\ref{tab:results_classif_runtime_internal}).
Since GAPA is mean-preserving, it leaves the backbone's point predictions unchanged, and therefore matches the MAP classifier in accuracy by construction.
GAPA achieves the best BALD AUROC in both directions and the best predictive-entropy AUROC on FMNIST$\rightarrow$MNIST, while remaining competitive on MNIST$\rightarrow$FMNIST.
Moreover, GAPA enables deterministic single-pass inference with near-MAP runtime (2.05s test time), yielding substantial speedups over Monte Carlo and curvature-based baselines.

\paragraph{Robustness under distribution shift.}
We evaluate robustness by rotating the \emph{test} images by increasing angles and plotting predictive NLL
for MNIST and FMNIST as shown in Figure~\ref{fig:combined_table_rotated_classification}.
As rotation increases and inputs move further from the training distribution, GAPA maintains competitive NLL under shifts
and achieves lower NLL than most baselines for large rotations, while appropriately increasing uncertainty. This indicates robust and well-calibrated behaviour, with the model correctly identifying heavily rotated inputs as out-of-distribution.

\paragraph{CIFAR-10 (pretrained ResNets).}
We evaluate CIFAR-10 with pretrained ResNet-20/32/44/56 backbones \citep{he2016deep}
and use SVHN as the OOD dataset.
Here the central question is the \emph{robustness--efficiency trade-off}: many strong OOD baselines are prohibitively expensive at inference, while fast methods often sacrifice OOD detection.
On representative ResNet-44/56 backbones (Table~\ref{tab:resnet_44_56_main}), GAPA attains strong OOD AUROC (0.931/0.953)
with low test-time cost (2.85s/3.30s), compared to the most accurate-but-slow baselines such as VaLLA and GP-subset
(test costs in the $10^2$--$10^3$s range). Figure~\ref{fig:ood_time_44_56} makes this explicit, GAPA lies on, or very close to, the Pareto frontier, offering one of the best OOD--inference cost trade-offs for all ResNet-20/32/44/56.

\subsection{Language models}\label{sec:llama}
\newcommand{\Ent}[1]{-\sum_{v=1}^V #1_v \log #1_v}
\newcommand{\softmax}{\operatorname{softmax}}    
We attach GAPA post hoc to LLaMA-3.2-3B (hidden size $3072$). For each chosen transformer block (we report layer indices), we log $\sim\!12$M pre-activations on WikiText-103 (training split) at sequence length $L=96$ and build a nearest-neighbor cache for uncertainty propagation. We use \texttt{KMeans} as preprocessing step, in Appendix \ref{app:llama} we also provide additional ablation experiments on the inducing point selection. Uncertainty propagation rules for key LLM model components, such as self attention or \texttt{RMSNorm} we detail in Appendix \ref{app:propagation}.

As in classification, after variance propagation the output layer yields mean logits 
$\mu_{1:k}$ and diagonal logit variances $v_{1:k}$ from a single forward pass. 
Predictive uncertainty is estimated via a vectorized reparameterization step over the top-$k$ tokens per position 
($k=512$). Specifically, we draw $S=512$ samples $
\ell^{(s)} = \mu_{1:k} + \sqrt{v_{1:k}} \odot \epsilon^{(s)}$ with  $\epsilon^{(s)} \sim \mathcal{N}(0,I)$ and compute 
 $p^{(s)}=\mathrm{softmax}(\ell^{(s)})$, to form 
 $\bar p = \frac{1}{S}\sum_{s=1}^S p^{(s)}$.
This procedure introduces no additional network evaluations; sampling occurs only in logit space.

Uncertainty is then decomposed using entropy as
\begin{align*}
\mathrm{TU} = H(\bar p), 
\;
\mathrm{AU} = \frac{1}{S}\sum_{s=1}^S H(p^{(s)}), 
\;
\mathrm{EU} = \mathrm{TU} - \mathrm{AU},
\end{align*}
where $H(p)=-\sum_v p_v\log p_v$. 
This corresponds to the mutual-information decomposition of  uncertainty.


We define two datasets: ID (WikiText-103, validation data) and OOD (OpenWebText); each sequence is labeled \(y\in\{0,1\}\). We filter BOS/EOS and extra whitespaces around punctuation marks (WikiText-103) to avoid trivial cues; OpenWebText is prepared analogously.  Note that OpenWebText is most likely not OOD for the pretrained LLaMA model itself; however, it is OOD relative to the employed method. The task is, given a sequence to distinguish these two classes based on the predictive distribution. 
For scoring, we compute EU,AU at every position and average over the sequence (the GAPA-based method is bold); AUROC is then computed against the sequence label. We include last-layer Laplace (trained on WikiText-103, training data) as additional baseline. Additionally, we include a temperature-scaling oracle that searches $\tau$ to maximize test AUROC using $\Ent{\softmax(\boldsymbol{\ell}/\tau)}$. This is \emph{not} a fair baseline (it tunes on the test metric) but provides an upper bound for a global rescaling of logits.

\begin{figure*}[t]
  \centering

  \begin{minipage}[b]{0.49\textwidth}
    \centering
    \includegraphics[width=0.8\textwidth]{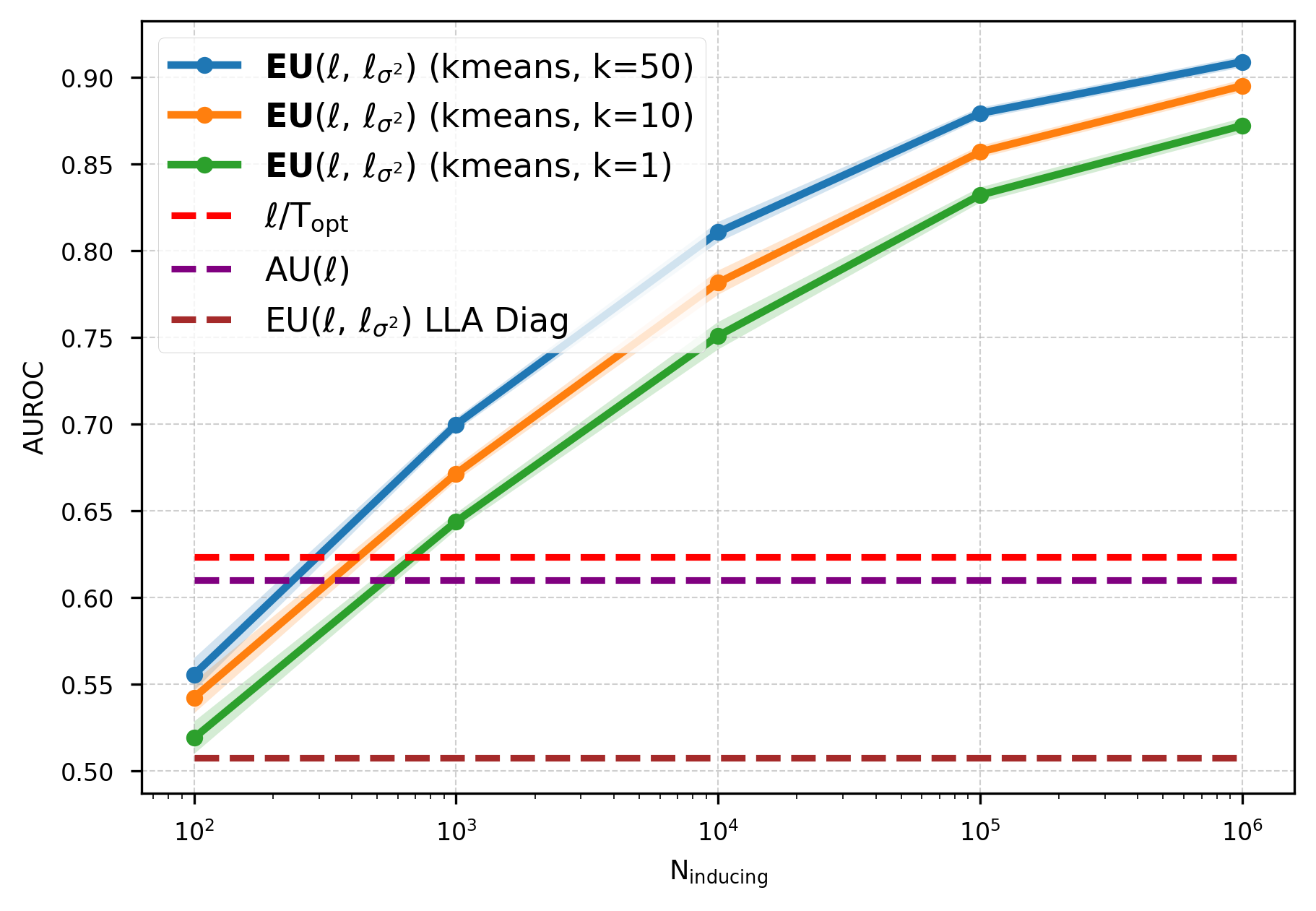}
  \end{minipage}
    \hfill
  \begin{minipage}[b]{0.49\textwidth}
    \centering
    \includegraphics[width=0.8\textwidth]{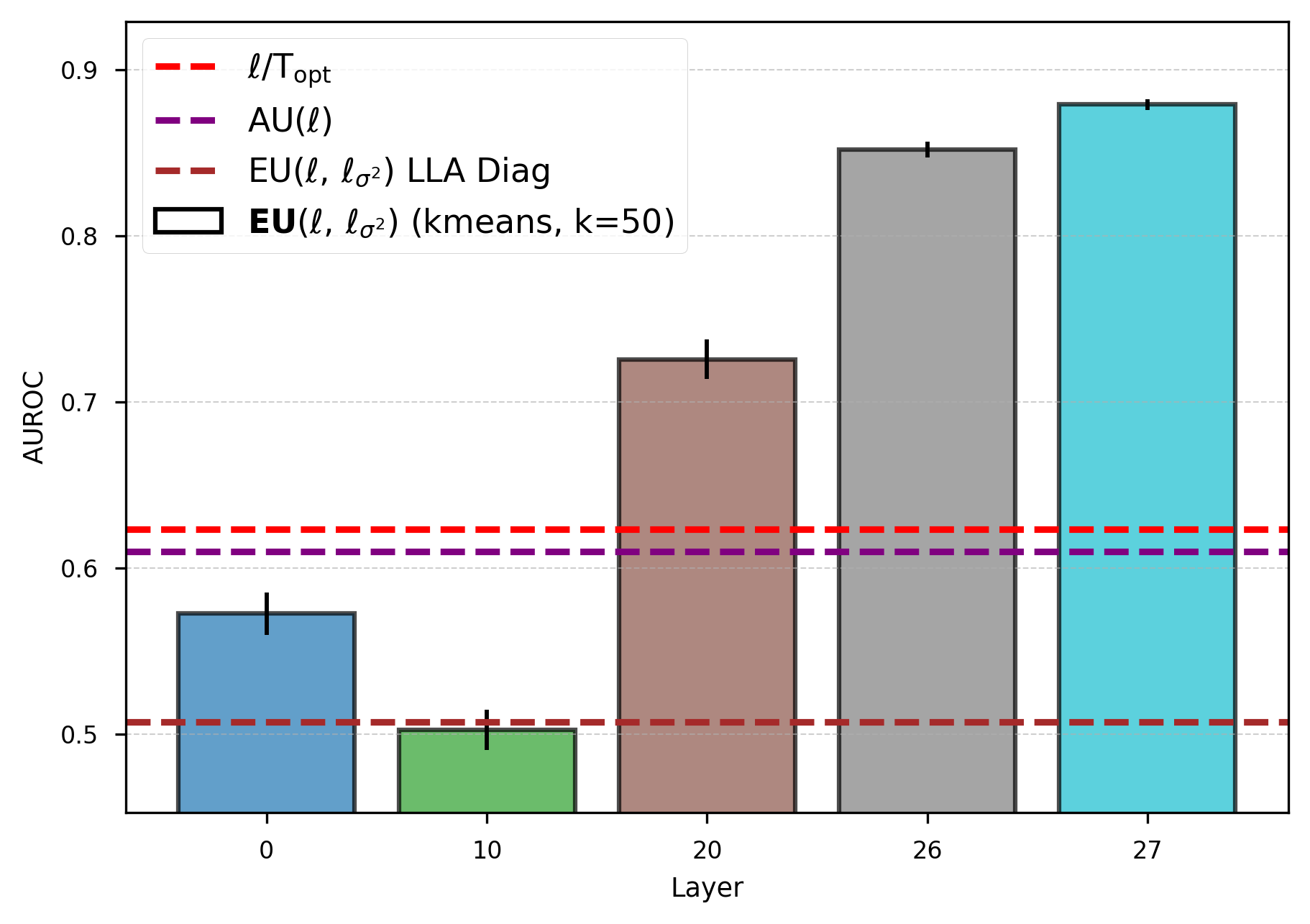}
  \end{minipage}
  \caption{Left: Effect of number of inducing points \(N_{\text{inducing}}\) and $k$ (for nearest neighbor inducing points) on OOD detection task with GAPA at layer [27]. Right: effect of layer placement of GAPA at \(N_{\text{inducing}}=10^5\). In both experiments results are averaged over $5$ runs with $512$ sequences each.  In both panels we also show the \(\ell/T_{\mathrm{opt}}\) bound (green) as an upper threshold of what can be achieved by global logits scaling.}
  \label{fig:inducing_vs_auc_llama}
\end{figure*}
 Fig. \ref{fig:inducing_vs_auc_llama} (left) shows that GAPA-based EU surpasses the oracle logit-temperature bound once \(N_{\text{inducing}}\!\gtrsim\!10^3\), indicating that activation-space epistemics capture distributional shift not recoverable by any global rescaling of logits.  Using higher number of local inducing points ($k=50$) improves performance. Last-layer linear Laplace does not perform better than chance, we found the Fisher-matrix estimation to become close to 0. The results in Figure \ref{fig:inducing_vs_auc_llama} (right) indicate that later layers improve performance with a noticeable drop in performance in the middle of the network.

\section{Related Work}
\label{sec:related}

Uncertainty quantification methods for deep networks differ primarily in \emph{where} uncertainty is placed and the resulting
test-time cost. Since we target \emph{frozen pretrained backbones} with \emph{single-pass} inference,
we focus on \emph{post-hoc Bayesian} baselines, and defer a broader discussion (sampling-based,
feature-based, calibration, and activation-space training methods) to Appendix~\ref{sec:related_extended}.
Laplace approximations place a Gaussian posterior over weights via a local quadratic approximation of the
log posterior~\citep{mackay1992practical, ritter2018scalable}; last-layer variants Bayesianize only the head
while keeping the feature extractor frozen, making them a standard post-hoc baseline.
In contrast, GAPA places uncertainty in \emph{activation space} and uses the original nonlinearity as the GP
prior mean, preserving the frozen network’s point predictions by construction while enabling deterministic
single-pass uncertainty estimation.

\section{Conclusion}

We introduced \emph{GAPA}, a post-hoc uncertainty
quantification method that places Bayesian modeling in \emph{activation space}
rather than weight space. By using the original nonlinearity as the GP prior mean,
GAPA preserves frozen model predictions exactly while providing principled
epistemic uncertainty. Designed for modern deployment constraints, GAPA employs
a scalable inducing-point approximation with local KNN conditioning, yielding
$O(\log M)$ inference and deterministic single-pass uncertainty estimation without
sampling, retraining, or backpropagation. Across regression, classification,
segmentation, and language modeling, GAPA matches or outperforms Laplace-family
methods in calibration and OOD detection while achieving favorable Pareto trade-offs
between uncertainty quality and inference cost, and scales to settings where
last-layer methods become prohibitive (e.g., large-vocabulary language models). GAPA’s primary limitation is the memory cost of storing inducing activations. Future work includes compressed or hierarchical indexing schemes and extending
beyond diagonal covariances to capture structured inter-neuron dependencies while
preserving scalability.

\section*{Impact Statement}

This work proposes a scalable, post-hoc method for uncertainty quantification in pretrained neural networks. The contribution is methodological in nature and is intended to improve the reliability and interpretability of model predictions. We do not foresee significant negative societal or ethical impacts beyond those commonly associated with the deployment of machine learning systems.

{
  \small
  \bibliographystyle{plainnat}
  \bibliography{references}
}

\clearpage
\onecolumn
\appendix

\setcounter{tocdepth}{2}
\addtocontents{toc}{\protect\setcounter{tocdepth}{2}}
\vspace{0.25em}
\section*{Appendix Table of Contents}
Appendix A–C: theory; D–G: additional experiments; H–L: implementation details; M–N: supplementary results.
\tableofcontents
\clearpage

\section{Conservativeness of Subset GP Conditioning}
\label{app:conservative}
These sections provide theoretical justification for the approximations used in GAPA. We formally justify the claim made in Section~4.3 that conditioning the GP posterior on a subset of inducing inputs yields a conservative estimate of epistemic uncertainty.

\begin{lemma}[Conservative uncertainty under subset conditioning]
Consider a Gaussian process prior with fixed kernel hyperparameters and Gaussian observation noise. Let $\tilde Z$ denote a set of inducing inputs and let $\tilde Z_k \subset \tilde Z$ be any subset. Then, for any test input $z^*$, the posterior variance satisfies
\[
\operatorname{Var}\!\left(f(z^*) \mid \tilde Z\right)
\;\le\;
\operatorname{Var}\!\left(f(z^*) \mid \tilde Z_k\right).
\]
That is, conditioning on a subset of inducing inputs cannot reduce posterior variance relative to conditioning on the full inducing set.
\end{lemma}

\begin{proof}
Let $A = \tilde Z_k$ and $B = \tilde Z \setminus \tilde Z_k$. Under the GP prior with Gaussian observation noise, the random variables
\[
\bigl(f(z^*),\, y_A,\, y_B\bigr)
\]
are jointly Gaussian. The conditional covariance identity for jointly Gaussian variables gives
\[
\operatorname{Var}(f(z^*) \mid y_A)
-
\operatorname{Var}(f(z^*) \mid y_A, y_B)
= 
K_{*B \mid A}
\left(K_{BB \mid A} + \sigma_n^2 I\right)^{-1}
K_{B* \mid A},
\]
where $K_{BB \mid A}$ and $K_{*B \mid A}$ denote conditional covariance matrices obtained via the Schur complement.

Since $K_{BB \mid A} + \sigma_n^2 I$ is positive definite, the right-hand side is positive semidefinite. Therefore,
\[
\operatorname{Var}(f(z^*) \mid y_A, y_B)
\;\le\;
\operatorname{Var}(f(z^*) \mid y_A),
\]
which establishes the result.
\end{proof}

This result implies that the local $k$-nearest-neighbour conditioning strategy used in GAPA yields a conservative approximation of epistemic uncertainty: posterior variance may be inflated due to discarded conditi

\section{Variational Inducing-Point Interpretation and Zero-Noise Limit}
\label{app:variational_inducing}

In this appendix, we show that the inducing-point construction used in GAPA
admits a variational interpretation in the sense of
Titsias~\cite{titsias2009variational}, and that in the limit of vanishing
observation noise the resulting posterior covariance reduces to the standard
inducing-point conditional covariance.

\paragraph{Setup.}
Consider a scalar Gaussian process
\[
f(\cdot) \sim \mathcal{GP}\!\big(m(\cdot), k(\cdot,\cdot)\big),
\]
and a set of inducing inputs
\[
\mathbf Z = \{ z_m \}_{m=1}^M,
\quad
\mathbf u = f(\mathbf Z) \in \mathbb{R}^M.
\]
Let \(\mathbf z^\ast\) denote a test input.

In GAPA, we condition the GP on noiseless pseudo-observations
\[
\tilde{\mathbf Y} = m(\mathbf Z),
\]
corresponding to evaluating the prior mean function at the inducing inputs.
For numerical stability, we introduce an auxiliary Gaussian noise term
\(\sigma_n^2\), which will be taken to zero.

\paragraph{Exact GP conditional.}
Conditioning a GP on noisy observations
\(\tilde{\mathbf Y}\) at inputs \(\mathbf Z\) yields the posterior covariance
\begin{equation}
\label{eq:gp_post_cov}
\operatorname{Var}\!\bigl(f(\mathbf z^\ast) \mid \tilde{\mathbf Y}\bigr)
=
k(\mathbf z^\ast,\mathbf z^\ast)
-
k(\mathbf z^\ast,\mathbf Z)
\bigl(
K_{\mathbf Z\mathbf Z} + \sigma_n^2 I
\bigr)^{-1}
k(\mathbf Z,\mathbf z^\ast),
\end{equation}
where \(K_{\mathbf Z\mathbf Z}\) is the kernel matrix on the inducing inputs.

\paragraph{Variational inducing-point posterior.}
Following Titsias~\cite{titsias2009variational}, the variational posterior over
the inducing variables \(\mathbf u\) is Gaussian,
\[
q(\mathbf u) = \mathcal{N}(\boldsymbol\mu, \mathbf A),
\]
with optimal parameters
\[
\boldsymbol\mu = \mathbf K_{\mathbf Z\mathbf Z}
\bigl(
\mathbf K_{\mathbf Z\mathbf Z} + \sigma_n^{-2}
\mathbf K_{\mathbf Z\mathbf Z}
\bigr)^{-1}
\tilde{\mathbf Y},
\qquad
\mathbf A =
\mathbf K_{\mathbf Z\mathbf Z}
\bigl(
\mathbf K_{\mathbf Z\mathbf Z} + \sigma_n^{-2}
\mathbf K_{\mathbf Z\mathbf Z}
\bigr)^{-1}
\mathbf K_{\mathbf Z\mathbf Z}.
\]

The corresponding variational predictive covariance at \(\mathbf z^\ast\) is
\begin{equation}
\label{eq:var_cov_general}
\operatorname{Var}_q\!\bigl(f(\mathbf z^\ast)\bigr)
=
k(\mathbf z^\ast,\mathbf z^\ast)
-
k(\mathbf z^\ast,\mathbf Z)
K_{\mathbf Z\mathbf Z}^{-1}
k(\mathbf Z,\mathbf z^\ast)
+
k(\mathbf z^\ast,\mathbf Z)
K_{\mathbf Z\mathbf Z}^{-1}
\mathbf A
K_{\mathbf Z\mathbf Z}^{-1}
k(\mathbf Z,\mathbf z^\ast).
\end{equation}

The final term is the \emph{variational correction} that distinguishes the
variational posterior from the exact GP conditional.

\paragraph{Zero-noise limit.}
We now consider the limit \(\sigma_n^2 \to 0\).
In this regime,
\[
\mathbf A
=
\mathbf K_{\mathbf Z\mathbf Z}
\bigl(
\mathbf K_{\mathbf Z\mathbf Z} + \sigma_n^2 I
\bigr)^{-1}
\mathbf K_{\mathbf Z\mathbf Z}
\;\xrightarrow[\sigma_n^2 \to 0]{}\;
\mathbf K_{\mathbf Z\mathbf Z}.
\]

Substituting into Eq.~\eqref{eq:var_cov_general}, the variational correction term
becomes
\[
k(\mathbf z^\ast,\mathbf Z)
K_{\mathbf Z\mathbf Z}^{-1}
\mathbf K_{\mathbf Z\mathbf Z}
K_{\mathbf Z\mathbf Z}^{-1}
k(\mathbf Z,\mathbf z^\ast)
=
k(\mathbf z^\ast,\mathbf Z)
K_{\mathbf Z\mathbf Z}^{-1}
k(\mathbf Z,\mathbf z^\ast),
\]
which exactly cancels the negative term in
Eq.~\eqref{eq:var_cov_general}.

Therefore, the variational predictive covariance reduces to
\begin{equation}
\operatorname{Var}\!\bigl(f(\mathbf z^\ast)\bigr)
=
k(\mathbf z^\ast,\mathbf z^\ast)
-
k(\mathbf z^\ast,\mathbf Z)
K_{\mathbf Z\mathbf Z}^{-1}
k(\mathbf Z,\mathbf z^\ast),
\end{equation}
which is precisely the standard inducing-point conditional covariance.

\paragraph{Implication for GAPA.}
Since GAPA conditions on noiseless pseudo-observations
\(\tilde{\mathbf Y} = \boldsymbol\phi(\mathbf Z)\) with a prior mean
\(m(\cdot) = \boldsymbol\phi(\cdot)\), the variational correction vanishes in the
zero-noise limit. Consequently, the posterior covariance used by GAPA coincides
with the standard inducing-point GP conditional covariance, while preserving
the deterministic backbone predictions exactly.

\section{Derivation for Stacking GAPA Layers}
\label{app:stacked_gapa_derivation}

When GAPA layers are stacked, the output of a preceding GAPA layer becomes the
input to a subsequent GAPA layer. Since each GAPA module produces a Gaussian
output, the input to the current layer is itself a random variable.
Let
\[
\mathbf z \sim \mathcal N(\boldsymbol\mu_z, \boldsymbol\Sigma_z)
\]
denote the pre-activation input to the current GAPA layer, where
\(\boldsymbol\Sigma_z\) is diagonal under our approximation.
We write
\[
\mathbf z = \boldsymbol\mu_z + \boldsymbol\varepsilon,
\qquad
\boldsymbol\varepsilon \sim \mathcal N(\mathbf 0, \boldsymbol\Sigma_z).
\]

We consider neuron-wise propagation. For neuron \(i\), the scalar input
\(z_i\) satisfies
\[
z_i = \mu_{z,i} + \varepsilon_i,
\qquad
\varepsilon_i \sim \mathcal N(0, \sigma_{z,i}^2).
\]

\paragraph{Base epistemic variance (deterministic input).}
Ignoring input uncertainty, the posterior variance of neuron \(i\) under the
KNN GAPA approximation is given by the standard inducing-point conditional
variance restricted to the local neighborhood
\(\mathbf Z_{\ell,K}(\mu_{z,i})\):
\begin{equation}
\label{eq:knn_base_var}
\sigma^2_{\mathrm{epi},i}(\mu_{z,i})
=
k_i(\mu_{z,i}, \mu_{z,i})
-
\mathbf k_i^\top
\bigl(
\mathbf K_i + \sigma_n^2 I
\bigr)^{-1}
\mathbf k_i,
\end{equation}
where \(\mathbf K_i \in \mathbb R^{K \times K}\) is the kernel matrix over the
\(K\) nearest inducing inputs, and
\([\mathbf k_i]_m = k_i(\mu_{z,i}, z_m)\).

\paragraph{Input uncertainty correction (NIGP).}
Because the test input is uncertain, we apply the noisy-input GP (NIGP)
approximation. Linearizing the posterior mean
\(\mu_i(z)\) around \(\mu_{z,i}\) yields an additional variance term
\begin{equation}
\label{eq:nigp_term}
\lambda_i(\mu_{z,i})
=
\sigma_{z,i}^2
\left(
\frac{\partial \mu_i(z)}{\partial z}
\Big|_{z = \mu_{z,i}}
\right)^2 .
\end{equation}

By construction of GAPA, the posterior mean equals the original deterministic
activation,
\[
\mu_i(z) = \phi_i(z),
\]
so the gradient term reduces to the derivative of the activation function,
and
\[
\lambda_i(\mu_{z,i})
=
\sigma_{z,i}^2 \bigl(\phi_i'(\mu_{z,i})\bigr)^2.
\]

\paragraph{Total predictive variance.}
Combining the epistemic variance from the local inducing-point posterior,
the NIGP correction due to input uncertainty, and optional observation noise
\(\sigma_{y,i}^2\), the total predictive variance for neuron \(i\) is
\begin{equation}
\label{eq:stacked_total_var}
\operatorname{Var}[y_i]
=
\underbrace{
\sigma^2_{\mathrm{epi},i}(\mu_{z,i})
}_{\text{activation-space epistemic}}
+
\underbrace{
\sigma_{z,i}^2 \bigl(\phi_i'(\mu_{z,i})\bigr)^2
}_{\text{propagated input uncertainty}}
+
\sigma_{y,i}^2 .
\end{equation}

This expression shows that stacking GAPA layers preserves mean predictions
exactly while propagating uncertainty forward in closed form, with each layer
contributing additional epistemic variance and attenuating or amplifying
uncertainty according to the local slope of the activation function.

\section{ResNets Pretrained Neural Networks}
\label{sec:ResNets} 

We report supplementary experiments and visualizations omitted from the main text for space.

\begin{table*}[t]
  \centering
    \caption{
    GAPA and baseline results on CIFAR-10 with ResNet backbones.
    This table reports the full results (including standard deviations) corresponding to
    Tables~\ref{tab:resnet_20_32} and~\ref{tab:resnet_44_56}.
    Best results are shown in \textcolor{purple}{purple} and second-best in \textcolor{teal}{teal}.
    }
    \label{tab:resnet_full}

  \vspace{0.2em}

  \footnotesize
  \setlength{\tabcolsep}{1.8pt}
  \renewcommand{\arraystretch}{1.05}

  \begin{adjustbox}{width=\textwidth}
  \begin{tabular}{l|ccccc|ccccc|ccccc|ccccc}
    \toprule
    & \multicolumn{5}{c|}{\textbf{ResNet-20}}
    & \multicolumn{5}{c|}{\textbf{ResNet-32}}
    & \multicolumn{5}{c|}{\textbf{ResNet-44}}
    & \multicolumn{5}{c}{\textbf{ResNet-56}} \\
    \cmidrule(lr){2-6}\cmidrule(lr){7-11}\cmidrule(lr){12-16}\cmidrule(lr){17-21}
    & ACC & NLL & OOD & Train & Test
    & ACC & NLL & OOD & Train & Test
    & ACC & NLL & OOD & Train & Test
    & ACC & NLL & OOD & Train & Test \\
    \midrule

    MAP
      & 92.6 & 0.282 & 0.876 & -- & --
      & 93.5 & 0.292 & 0.909 & -- & --
      & 94.0 & 0.275 & 0.885 & -- & 0.761
      & 94.4 & 0.252 & 0.924 & -- & 0.949 \\

    MF-VI
      & \textcolor{teal}{92.7} & 0.231 & 0.865 & 0.74 & \textcolor{purple}{0.47}
      & \textcolor{teal}{93.5} & 0.222 & 0.885 & 1.19 & \textcolor{purple}{0.75}
      & \textcolor{teal}{93.9} & 0.206 & 0.890 & 1.63 & \textcolor{teal}{1.03}
      & \textcolor{teal}{94.4} & 0.188 & 0.929 & 1.97 & \textcolor{teal}{1.18} \\

    SNGP
      & 92.4 & 0.266 & 0.875 & 15.9 & 1.31
      & 93.2 & 0.256 & 0.890 & 25.5 & 2.10
      & 93.8 & 0.242 & 0.901 & 35.0 & 2.89
      & 93.8 & 0.229 & \textcolor{teal}{0.940} & 43.5 & 3.01 \\

    GP (subset)
      & 92.6 & 0.555 & 0.870 & 3.75K & 162
      & 93.4 & 0.462 & 0.885 & 6.00K & 260
      & 93.6 & 0.424 & 0.897 & 8.25K & 357
      & 94.4 & 0.403 & 0.936 & 8.42K & 382 \\

    LLA Diag
      & 92.6 & 0.260 & 0.866 & 18.4 & \textcolor{teal}{0.43}
      & 93.5 & 0.242 & 0.882 & 29.4 & \textcolor{teal}{0.69}
      & 94.0 & 0.218 & 0.860 & 40.4 & \textcolor{purple}{0.947}
      & 94.3 & 0.195 & 0.923 & 40.67 & \textcolor{purple}{1.12} \\

    LLA KFAC
      & 92.6 & 0.241 & 0.877 & 28.7 & 1.65
      & 93.5 & 0.229 & 0.903 & 45.9 & 2.63
      & 94.0 & 0.213 & 0.855 & 63.1 & 3.62
      & 94.4 & 0.193 & 0.917 & 71.3 & 4.13 \\

    Sampled LLA
      & 92.5 & \textcolor{purple}{0.231} & 0.885 & 5.00K & 1.14K
      & 93.5 & \textcolor{teal}{0.217} & 0.905 & 8.00K & 1.83K
      & 94.0 & \textcolor{purple}{0.200} & 0.899 & 11.0K & 2.51K
      & 94.4 & \textcolor{purple}{0.185} & \textcolor{teal}{0.944} & 14.6K & 2.84K \\

    VaLLA
      & 92.4 & \textcolor{teal}{0.231} & \textcolor{purple}{0.940} & 7.59K & 124
      & 93.2 & \textcolor{purple}{0.212} & \textcolor{teal}{0.933} & 12.2K & 199
      & 93.8 & \textcolor{teal}{0.201} & \textcolor{teal}{0.928} & 16.7K & 272.9
      & 94.2 & \textcolor{teal}{0.188} & \textcolor{purple}{0.960} & 26.3K & 363.8 \\

    \textbf{GAPA (ours)}
      & 92.6 & 0.258 & \textcolor{teal}{0.907} & \textcolor{purple}{3.65} & 1.30
      & 93.5 & 0.259 & \textcolor{purple}{0.926} & \textcolor{purple}{5.84} & 2.07
      & 94.0 & 0.230 & \textcolor{purple}{0.931} & \textcolor{purple}{8.03} & 2.85
      & 94.4 & 0.230 & \textcolor{teal}{0.953} & \textcolor{purple}{10.29} & 3.30 \\

    \bottomrule
  \end{tabular}
  \end{adjustbox}
\end{table*}


\paragraph{ResNet backbones and computational trade-offs.}
Table~\ref{tab:resnet_full} reports full CIFAR-10 results for pretrained ResNet backbones,
including in-distribution performance (ACC/NLL), OOD detection (AUROC), and train/test runtime.
Across all depths, GAPA achieves competitive in-distribution accuracy and NLL while providing strong OOD detection.
Notably, methods with the strongest OOD scores (e.g., VaLLA and sampled/GP-subset baselines) often incur
substantially higher computational costs, both at test time and train time.

Figure~\ref{fig:ood_time_44_56} visualises the robustness--efficiency trade-off by plotting OOD AUROC against
test-time inference cost (log scale).
Dashed lines indicate Pareto frontiers (higher OOD, lower cost), highlighting methods that offer the best compromise
between reliability and efficiency.
Across backbones, GAPA lies on the Pareto frontier, achieving strong OOD AUROC while avoiding the
\(10^2\)--\(10^3\)s inference costs characteristic of more computationally intensive Bayesian baselines.

\section{Image Segmentation}
\label{sec:segmentation} 
\begin{figure}ht
  \centering
  \includegraphics[width=0.9\linewidth]{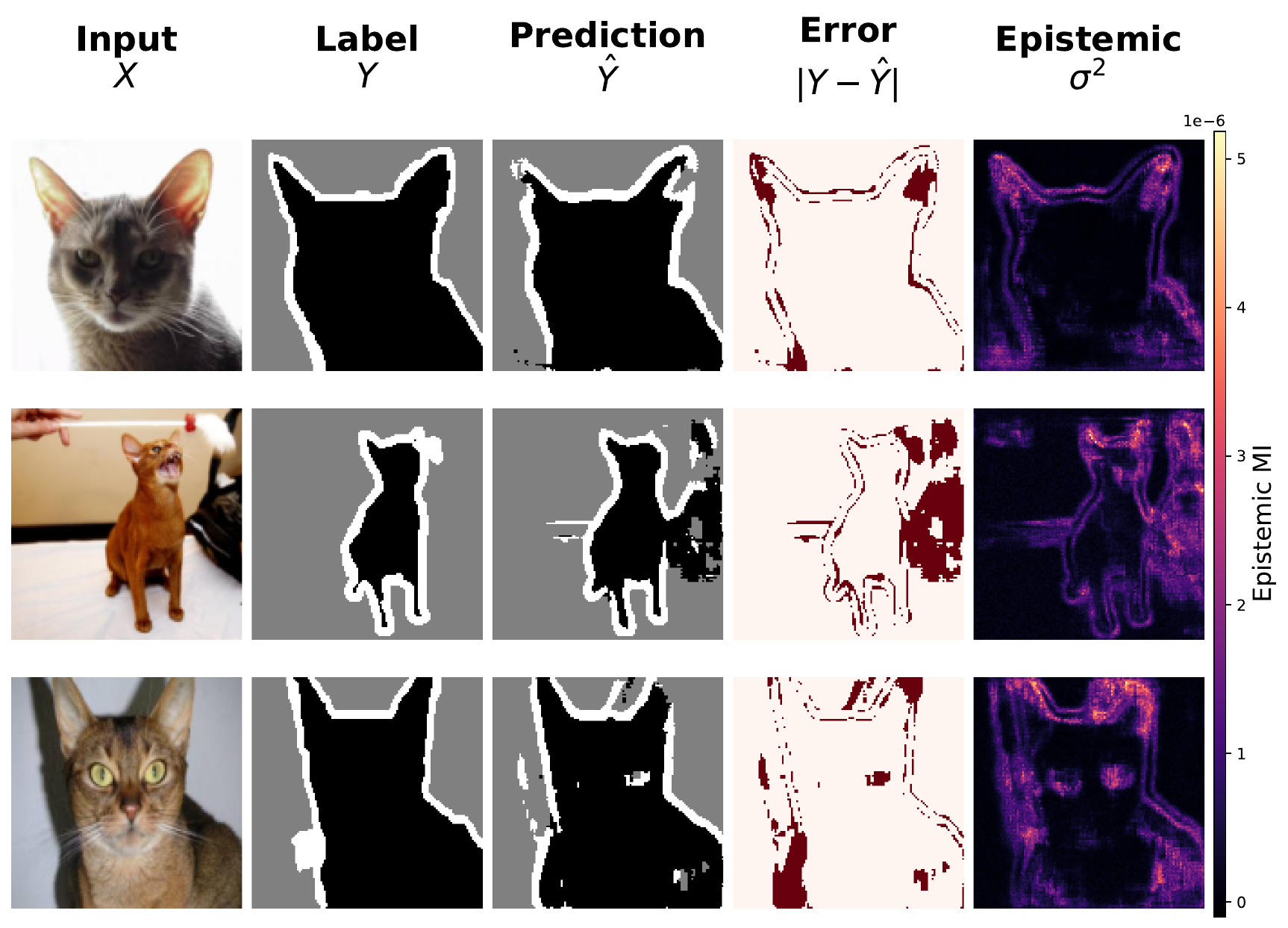}
  \caption{%
    Qualitative segmentation results with pixel-wise error and epistemic uncertainty.
    \textbf{Columns}: (1) Input image \(X\), (2) Ground-truth mask \(Y\), (3) Predicted mask \(\hat Y\),
    (4) Error map \(\lvert Y-\hat Y\rvert\), (5) Epistemic uncertainty (mutual information).
    \textbf{Rows}: three representative validation examples.}
  \label{fig:segmentation-uncertainty}
\end{figure}

As a proof of concept for high-dimensional outputs, we apply GAPA to a U-Net model \citep{ronneberger2015u} pre-trained on the Oxford-IIIT Pet dataset \citep{parkhi2012cats} for a 3-class segmentation task (background, pet, outline) with input images resized to \(128 \times 128\). The U-Net architecture features an encoder path with two downsampling stages (32 and 64 channels, using double convolutions and max pooling), leading to a bottleneck with 128 channels. From this bottleneck, an embedding head comprising adaptive average pooling and a linear layer projects the features to a \(d=64\) dimensional embedding vector. Standard skip connections are used in the decoder path.

For these experiments, GAPA was applied to this \(d=64\) dimensional embedding vector at the U-Net bottleneck. This vector represents the most compressed representation in the network, and its 1D nature (after pooling and flattening). The GAPA-processed embedding (mean preserved, variance added) is then reshaped and fed into the decoder to produce the final segmentation map.

The dimensionality of the full segmentation output space (e.g., \( 128 \times 128 \times 3\) or \( \sim\!224\times224 \) per image if referring to original dataset paper's output size before your resize) renders methods like full Laplace approximation computationally infeasible due to memory and time constraints (e.g., matrix inversions on \(\mathcal{O}(10^5)\) outputs or more). In contrast, applying GAPA at the compressed embedding stage scales efficiently. Figure~\ref{fig:segmentation-uncertainty} demonstrates that this approach not only produces accurate segmentation masks but also generates spatially localized epistemic uncertainty maps that precisely highlight regions where prediction errors occur.

\begin{figure*}[t]
    \centering
    \includegraphics[width=0.24\textwidth]{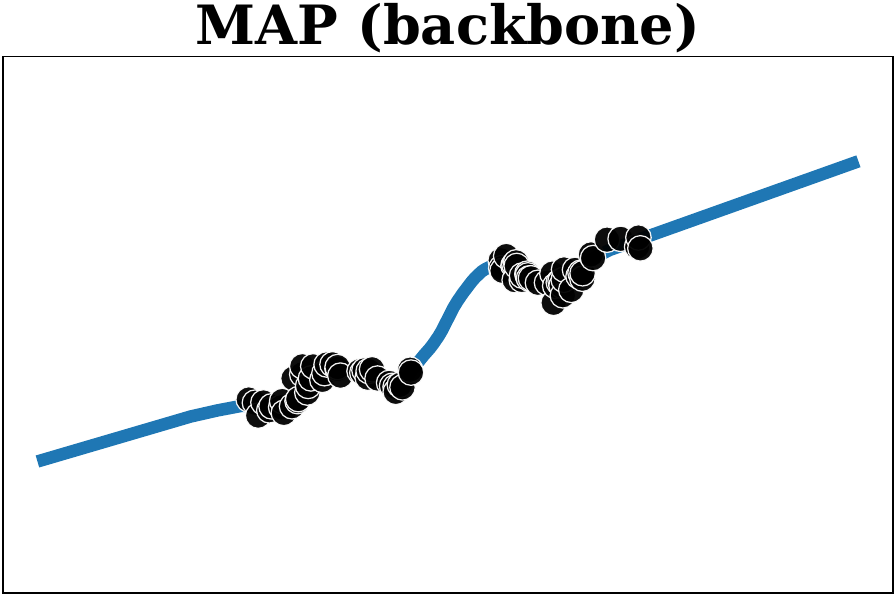}\hfill
    \includegraphics[width=0.24\textwidth]{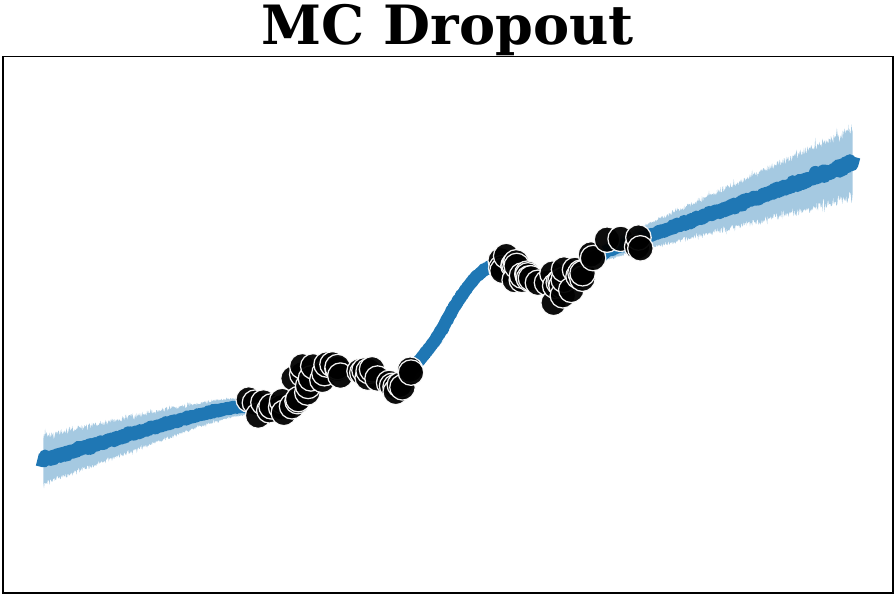}\hfill
    \includegraphics[width=0.24\textwidth]{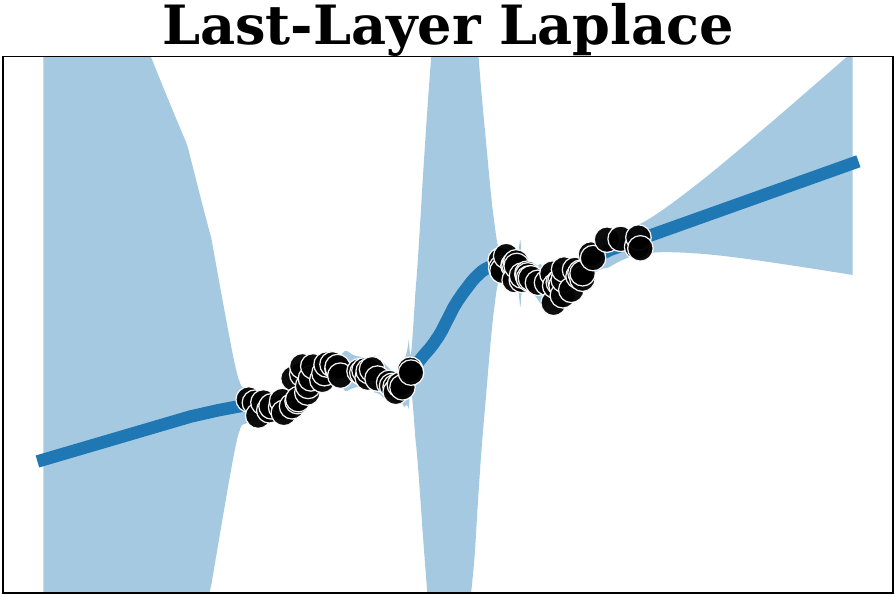}\hfill
    \includegraphics[width=0.24\textwidth]{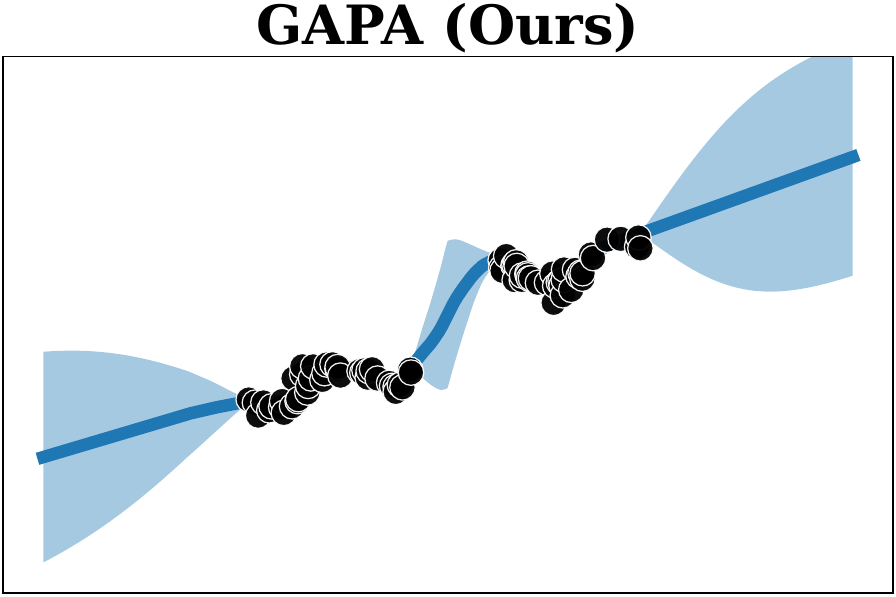}

    \caption{Regression predictions and uncertainty: (a) MAP backbone, (b) MC Dropout, (c) last-layer Laplace, (d) GAPA (ours).}
    \label{fig:example_regression}
\end{figure*}

\clearpage
\section{LLaMA-3.2 Additional Results}\label{app:llama}
In Figure \ref{fig:llama_random_ood} we show an additional ablation study where we compare inducing point selection using \texttt{Kmeans} or random choice. We find that random underperformed compared to \texttt{KMeans}. 
\begin{figure*}[h!]
\centering
    \includegraphics[width=0.45\textwidth]{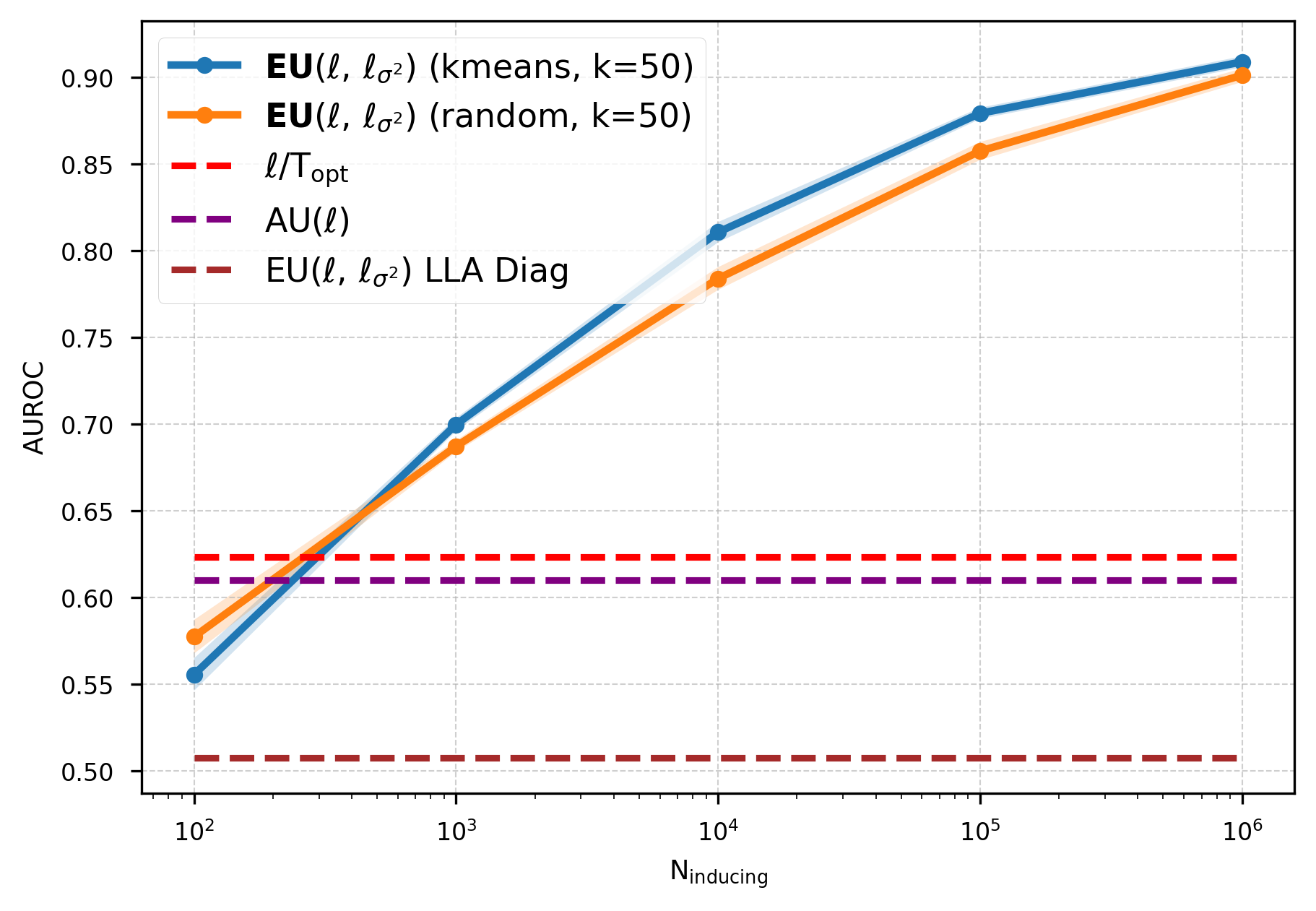}
  \caption{Effect of the number of inducing points \(N_{\text{inducing}}\) and preprocessing strategy on OOD detection task. We plot the AUC using EU (blue) with GAPA at layer [27]. Results are averaged over $5$ runs with $512$ sequences each.  In both panels we also show the \(\ell/T_{\mathrm{opt}}\) bound (green) as an upper threshold of what can be achieved by global logits scaling.}
    \label{fig:llama_random_ood}

  \end{figure*}

\section{GAPA Hyperparameters}
This section details hyperparameters, inference procedures, and architectural propagation rules.

\subsection{GAPA Empirical Hyperparameters}\label{Appendix:empirical_kernel}
\label{sec:gapa_hyp}

For the \textbf{GAPA}, we deliberately avoid any gradient–based
hyper-parameter optimisation.  
Instead, the RBF-kernel length scale~$\ell_k$, signal amplitude~$\sigma_k$,  
and the set of pseudo-inputs~$\mathcal Z$ are fixed once from simple
empirical statistics of the training data.

\paragraph{Length scale \boldmath$\ell_k$.}
We set every neuron’s length scale to the empirical median of all pairwise
Euclidean distances between training inputs:
\[
d_{ij} = \lVert x_i - x_j\rVert_2,\quad
\ell_k = \mathrm{Median}\bigl(\{d_{ij}\}\bigr).
\]
In our implementation we approximate this by sampling $10^6$ random pairs.

\paragraph{Signal variance \boldmath$\sigma_k^2$.}
For each hidden neuron we compute the sample standard deviation of its
pre-activations over the training set:
\[
\sigma_k = \operatorname{Std}\bigl\{h_k(x_i)\bigr\}_{i=1}^N.
\]
Clamped to a minimum of $10^{-6}$ to ensure numerical stability.

\paragraph{Pseudo-inputs \boldmath$\mathcal Z$.}
With a budget of $M$ inducing points we perform a greedy farthest-first
traversal over the training inputs:

\begin{enumerate}[leftmargin=*,nosep]
  \item Select an arbitrary $z_1$ from the training set.
  \item For $m=2,\dots,M$, choose $z_m$ as the training input whose
        minimum Euclidean distance to $\{z_1,\dots,z_{m-1}\}$ is maximal.
\end{enumerate}

\paragraph{KMeans pseudo-inputs.}
As an alternative to farthest-first traversal, we also provide a KMeans-based
strategy for selecting inducing points.  
In this variant, the pseudo-input set $\mathcal Z$ consists of the $M$ cluster
centroids obtained by running KMeans on the training activations.

We initialise the clustering using the standard \textbf{KMeans++} seeding
procedure: the first centre is chosen uniformly at random, and each subsequent
centre is selected with probability proportional to its squared distance from
the closest existing centre.  
This produces well-separated initial centroids and improves stability and
convergence compared to random initialisation.

KMeans provides a simple, task-agnostic alternative to
farthest-first traversal, and can be used interchangeably within GAPA for
constructing $\mathcal Z$.

\subsection{Regression Training Details}
\label{Appendix:gapa_train}

For regression, we parameterize the aleatoric variance using a small MLP head $s_\psi$ that takes hidden representations as input:
\[
\sigma^2_{\text{ale}}(\mathbf{x}) = \mathrm{softplus}(s_\psi(\mathbf{x})) + \varepsilon
\]
where $\varepsilon = 10^{-6}$ is a variance floor preventing numerical instability. The total predictive variance combines epistemic (from GAPA) and aleatoric components:
\[
\sigma^2_{\text{tot}}(\mathbf{x}) = \sigma^2_{\text{epi}}(\mathbf{x}) + \sigma^2_{\text{ale}}(\mathbf{x})
\]
We train only the parameters $\psi$ by minimizing:
\[
\mathcal{L}_{\text{reg}} = \frac{1}{N}\sum_{n=1}^N \left[\frac{(y_n - \mu_n)^2}{2\sigma^2_{\text{tot}}(\mathbf{x}_n)} + \frac{1}{2}\log(2\pi\sigma^2_{\text{tot}}(\mathbf{x}_n))\right]
\]
where $\mu_n$ is the fixed mean prediction from the frozen backbone. This preserves exact mean predictions while learning data-dependent noise.

\section{Nearest–Neighbour Retrieval with \texttt{Faiss}}
\label{sec:faiss}

GAPA requires fast retrieval of the $K$ nearest inducing activations in
activation space. Given a set of inducing inputs
\(
\mathcal Z = \{z_m\}_{m=1}^{M} \subset \mathbb R^{d}
\),
for each query activation $z^\ast$ we compute
\[
\mathcal N_K(z^\ast)
=
\arg\min_{\substack{S \subset \{1,\dots,M\}\\ |S|=K}}
\sum_{m \in S} \|z^\ast - z_m\|_2 .
\]
A brute-force search costs $\mathcal O(Md)$ per query.
Instead, we use \texttt{Faiss} to support efficient approximate
nearest-neighbour retrieval with sublinear complexity.

\subsection{Index construction}

\begin{enumerate}[leftmargin=*,nosep]
\item \textbf{Choice of index.}
For small to moderate inducing sets we use \texttt{IndexFlatL2} (exact search).
For large $M$, we use \texttt{IndexIVFPQ}, which partitions the space via
coarse $k$-means clustering and stores PQ-compressed residuals
\cite{douze2025faiss}.
\item \textbf{Training (optional).}
Quantization-based indices (IVF, PQ) require a one-time offline training
step on a representative subset of inducing activations.
\item \textbf{Adding vectors.}
All inducing activations $z_m$ are added once to the index; their identifiers
link back to the cached kernel statistics required for GP inference.
\end{enumerate}

The resulting index requires $\mathcal O(M)$ memory and supports
$K$-NN queries in $\mathcal O(\log M)$ expected time for IVF-based indices.

\subsection{Query procedure}

For each test-time activation $z^\ast$:

\begin{enumerate}[leftmargin=*,nosep]
\item Query the Faiss index to retrieve the $K$ nearest inducing inputs:
\[
(\mathbf d, \mathcal N_K)
\leftarrow
\texttt{index.search}(z^\ast,\,K).
\]
\item Form the local inducing set
\(
\mathbf Z_K(z^\ast) = \{z_m : m \in \mathcal N_K\}.
\)
\item Compute the neuron-wise posterior variance using the standard
inducing-point conditional restricted to this local set:
\[
\operatorname{Var}[f(z^\ast)]
=
k(z^\ast,z^\ast)
-
\mathbf k^\top
\bigl(
\mathbf K_K + \sigma_n^2 I
\bigr)^{-1}
\mathbf k,
\]
where $\mathbf K_K$ is the $K \times K$ kernel matrix on $\mathbf Z_K(z^\ast)$
and $[\mathbf k]_m = k(z^\ast, z_m)$.
\end{enumerate}

By construction, the posterior mean remains equal to the original
deterministic activation, so nearest-neighbour retrieval affects only the
epistemic variance.

\subsection{Complexity}

\begin{itemize}[leftmargin=*,nosep]
\item \emph{Index construction:} one-off $\mathcal O(Md)$ time and
$\mathcal O(M)$ memory.
\item \emph{Query:} $\mathcal O(\log M)$ approximate neighbour search plus
$\mathcal O(K^3)$ local linear algebra, with $K \ll M$ fixed.
\end{itemize}

This design keeps test-time GP inference independent of the total number of
cached activations while preserving a principled, distance-aware posterior
variance.

\section{Laplace-Bridge Approximation for Classification}
\label{app:laplace_bridge}

Given mean logits $\boldsymbol{\mu}\in\mathbb{R}^C$ and per-class variances $\mathbf{v}\in\mathbb{R}^C$ from GAPA propagation, we compute predictive probabilities using:
\begin{equation}
p(y=c \mid \mathbf{x}) \approx \frac{\exp\!\Big(\mu_c \big/ \sqrt{1 + (\pi/8)v_c}\Big)}{\sum_{c'=1}^{C} \exp\!\Big(\mu_{c'} \big/ \sqrt{1 + (\pi/8)v_{c'}}\Big)}
\end{equation}
The division and square root are applied element-wise to each logit before the softmax. This approximation integrates Gaussian logit uncertainty into categorical predictions without sampling, derived from the probit approximation $\Phi(x) \approx \sigma(x\sqrt{\pi/8})$ where $\Phi$ is the Gaussian CDF and $\sigma$ is the sigmoid function.

\section{Metrics}
\subsection{Regression Metrics}
\label{app:regression_metrics}

For evaluating performance on regression tasks (Section~\ref{sec:results_regression}), we use several key metrics.
First, the \textbf{Negative Log-Likelihood (NLL)} measures the quality of the predictive probability distribution. Assuming a Gaussian predictive distribution \(p(y|x) = \mathcal{N}(y; \mu(x), \sigma^2(x))\), where \(\mu(x)\) is the predicted mean and \(\sigma^2(x)\) is the predicted variance, the NLL for a true target value \(y_{\text{true}}\) is \( \frac{1}{2} \log(2\pi\sigma^2(x)) + \frac{(y_{\text{true}} - \mu(x))^2}{2\sigma^2(x)} \). Lower NLL values are better, indicating that the predictive distribution is both accurate and appropriately confident.
Second, the \textbf{Continuous Ranked Probability Score (CRPS)} \citep{gneiting2007strictly} generalizes the Mean Absolute Error (MAE) to probabilistic forecasts. For a predictive cumulative distribution function (CDF) \(F\) and a true outcome \(y_{\text{true}}\), it is defined as \(\text{CRPS}(F, y_{\text{true}}) = \int_{-\infty}^{\infty} (F(y) - \mathbf{1}\{y \ge y_{\text{true}}\})^2 dy\), where \(\mathbf{1}\{\cdot\}\) is the indicator function. For a Gaussian predictive distribution \(\mathcal{N}(\mu, \sigma^2)\), a closed-form expression exists. Lower CRPS values are better, indicating a sharper and more calibrated predictive distribution.
Finally, the \textbf{Centered Quantile Metric (CQM)}, as proposed by \citet{ortega2023variational}, evaluates the calibration of specific quantiles of the predictive distribution. It typically focuses on how well the predicted quantiles (e.g., the 5th and 95th percentiles) align with the empirical frequency of observations falling below these quantiles. A common formulation might assess the average miscalibration across symmetric quantiles, where lower CQM values generally indicate better quantile calibration.

\subsection{Classification Metrics}
\label{app:classification_metrics} 

For evaluating performance on classification tasks (Section~\ref{sec:classification_results}), we use several key metrics.
\textbf{Accuracy (ACC)} is the overall proportion of correctly classified samples; we note that GAPA, by design, preserves the mean predictions of the backbone network, so its ACC should match that of the original pre-trained model unless other methods being compared modify these predictions.
The \textbf{Negative Log‐Likelihood (NLL)}, in classification, is equivalent to the cross-entropy loss and measures the quality of the predictive probability distribution. For a given sample with true class label \(y_{\text{true}}\) (out of \(C\) classes) and where the model predicts a probability distribution \(p(y|x)\) over the classes, the NLL for that sample is specifically \(-\log p(y_{\text{true}}|x)\), which is the negative logarithm of the probability assigned by the model to the correct class; lower values indicate better performance.
\textbf{Expected Calibration Error (ECE)} measures the discrepancy between-+ a model's predicted confidences and its empirical accuracies. Predictions are typically binned by their confidence scores. For each bin \(B_m\), the accuracy \(\mathrm{acc}(B_m)\) and average confidence \(\mathrm{conf}(B_m)\) are computed. ECE is then a weighted average of the absolute difference: \(\sum_{m=1}^{M} \frac{|B_m|}{N} |\mathrm{acc}(B_m) - \mathrm{conf}(B_m)|\), where \(N\) is the total number of samples; lower values indicate better calibration.
For \textbf{Out-of-Distribution (OOD) Detection}, we report the Area Under the ROC curve (AUC). This evaluates the model's ability to distinguish between in-distribution (ID) and out-of-distribution (OOD) samples based on an uncertainty score. We primarily use the predictive entropy of the softmax distribution as the uncertainty score (denoted \textbf{OOD-Entropy} or \textbf{OOD-AUC}); higher AUC values (closer to 1) indicate better OOD detection performance.
We also evaluate \textbf{OOD Detection AUC with BALD (OOD-BALD)}, which is similar to the above, but the uncertainty score used for OOD detection is the Bayesian Active Learning by Disagreement (BALD) score \citep{houlsby2011Bayesian}. BALD measures the mutual information between the model's predictions and its parameters, often providing a better measure of epistemic uncertainty; a higher AUC indicates better OOD detection using BALD.

\section{Variance Propagation in Transformer Architectures}\label{app:propagation}
To implement variance propagation in transformers, in addition to the classical linear layers or activation, we need three additional propagation rules: \texttt{RMSNorm}, \texttt{CausalSelfAttention} and \texttt{Softmax}.

\subsection{Attention}\label{app:attention}
Here, we present two variants to propagate the variance through a self-attention layer.

Given an input vector $\mathbf x\in\mathbb{R}^{d}$ with per–feature variances
$\mathrm{Var}(x_j)=v_j$, we first form the standard query/key/value projections
\[
q = W^Q x,\quad
k = W^K x,\quad
v = W^V x,
\]
with
\begin{equation*}
\mathrm{Var}(q_i)
  = \sum_{j=1}^d (W^Q_{ij})^2\,v_j,
\quad
\mathrm{Var}(k_i)
  = \sum_{j=1}^d (W^K_{ij})^2\,v_j,
\quad
\mathrm{Var}(v_i)
  = \sum_{j=1}^d (W^V_{ij})^2\,v_j.
\end{equation*}

\textbf{Variant A.}\;
We treat the attention weights $a_{ts}$ as deterministic, and propagate akin to a linear layer propagation: 
\begin{equation*}
\operatorname{Var}(y_{t,i})
  \;=\;
  \sum_{s} a_{ts}^2\,\operatorname{Var}(v_{s,i}).
\end{equation*}

\medskip
\textbf{Variant B.}\;
Let $d_k$ be the head dimension and define the scaled logits
\(
e_{ts}=d_k^{-1/2}\,q_t^\top k_s.
\)
Under the delta method
\begin{equation*}
\operatorname{Var}(a_{ts})
  \;=\;
  \frac{1}{d_k}\sum_{h=1}^{d_k}\!\Bigl(
       q_{t,h}^2\,\operatorname{Var}(k_{s,h})
     + k_{s,h}^2\,\operatorname{Var}(q_{t,h})
     + \operatorname{Var}(q_{t,h})\,\operatorname{Var}(k_{s,h})
  \Bigr).
\end{equation*}
After masking and applying the soft-max propagation rule of Appendix \ref{app:softmax} we obtain
$\operatorname{Var}(a_{ts})$.  
The variance of the head output is then
\begin{equation*}
\operatorname{Var}(y_{t,i})
  \;=\;
  \sum_{s}\Bigl[
      \operatorname{Var}(a_{ts})\,v_{s,i}^{2}
    + a_{ts}^{2}\,\operatorname{Var}(v_{s,i})
    + \operatorname{Var}(a_{ts})\,\operatorname{Var}(v_{s,i})
  \Bigr].
\end{equation*}

While the second method is arguably modeling the overall variance propagation in a more sophisticated way, from a design choice perspective, the decision is not obvious: the first propagation scheme is much faster. Although we weren't directly able to use flash attention, in theory a \texttt{FlashAttention} kernel could be modded to calculate the squared attention operation on-the-fly at no additional cost. Secondly we found that the variances can grow quickly the more layer the transformer model has because of the compounding, multiplicative  effect of the variance over both the attention scores and the query, key and values. This compounding effect could be addressed with  full covariance propagation, however, for transformers the embedding space  is large (e.g. for LLaMA-3.2, 3B it is $4,096$), making a full covariance treatment computationally intractable (and low-rank approximations either still exhibit compounding effect (if too low-rank) or again intractable). Based on these reasons in the paper use Variant A:  we will assume the model knows where to look (deterministic attention weights), but is uncertain about what it sees there (value variance).

\subsection{RMSNorm}
Let $\mathbf{x} \in \mathbb{R}^d$ with per-feature variances $\mathrm{Var}(x_j) = v_j$.
RMSNorm computes the root mean square
\begin{equation*}
\text{RMS}(\mathbf{x}) = \sqrt{\frac{1}{d}\sum_{j=1}^d x_j^2 + \varepsilon},
\end{equation*}
and applies the transformation
\begin{equation*}
y_i = \gamma_i \cdot \frac{x_i}{\text{RMS}(\mathbf{x})},
\end{equation*}
where $\gamma_i$ are learned scale parameters and $\varepsilon > 0$ is a small constant for numerical stability.

As a first-order approximation  we define the expected RMS squared as
\begin{equation*}
s^2 = \mathbb{E}\Bigl[\frac{1}{d}\sum_{j=1}^d x_j^2\Bigr] + \varepsilon.
\end{equation*}
Using the identity $\mathbb{E}[x_j^2] = \mathrm{Var}(x_j) + \mathbb{E}[x_j]^2$, we can rewrite this as
\begin{align*}
s^2 &= \frac{1}{d}\sum_{j=1}^d \mathbb{E}[x_j^2] + \varepsilon \\
    &= \frac{1}{d}\sum_{j=1}^d \bigl(v_j + \mathbb{E}[x_j]^2\bigr) + \varepsilon.
\end{align*}

Under this approximation, we treat $s^2$ as deterministic and propagate variance as
\begin{equation*}
\mathrm{Var}(y_i) \approx \frac{\gamma_i^2}{s^2} \, v_i.
\end{equation*}
\medskip
\noindent
The following PyTorch implementation realizes the simplified scheme:

{\renewcommand{\baselinestretch}{1.2}\selectfont
\begin{lstlisting}[style=arxivpython]
class RMSNormVar(torch.nn.Module):
    def __init__(self, dim: int, eps: float = 1e-6):
        super().__init__()
        self.eps = eps
        self.weight = nn.Parameter(torch.ones(dim))

    def _norm(self, x):
        return x * torch.rsqrt(x.pow(2).mean(-1, keepdim=True) + self.eps)

    def forward(self, input_mean, input_var):
        """
        Args:
            input_mean (torch.Tensor): Input means, shape [batch_size, ..., feature_dim]
            input_var (torch.Tensor): Input variances, shape same as input_mean

        Returns:
            tuple: (output_mean, output_var)
        """

        output_mean = self._norm(input_mean) * self.weight

        # Compute expected rms squared
        expected_rms_squared = (
            input_mean.pow(2).mean(dim=-1, keepdim=True)
            + input_var.mean(dim=-1, keepdim=True)
            + self.eps
        )

        # Compute output variances
        output_var = input_var / expected_rms_squared * self.weight.pow(2)

        return output_mean, output_var
\end{lstlisting}
}
\subsection{Softmax}\label{app:softmax}
For softmax we  follow the Delta method approach. We note that this method is only used for the second variant of \texttt{SelfAttention}, whereas in this paper we use the first variant.

Let $\mathbf x\in\mathbb R^{K}$ with per–feature variances
$\mathrm{Var}(x_i)=v_i$.  The softmax output is
\begin{equation*}
s_k = \frac{e^{x_k}}{\sum_{j=1}^K e^{x_j}}.
\end{equation*}

The Jacobian of the softmax for fixed~$k$ is
\begin{equation*}
\frac{\partial s_k}{\partial x_i}
  = s_k\,(\delta_{ik} - s_i).
\end{equation*}

Applying the Delta method with 
$\Sigma_x=\operatorname{diag}(v_1,\dots,v_K)$ gives
\begin{equation*}
\operatorname{Var}(s_k)
  \;=\;
  \sum_{i=1}^{K}
    \Bigl(s_k(\delta_{ik}-s_i)\Bigr)^{2}\,v_i.
\end{equation*}

If we split out the $i=k$ term and the $i\neq k$ terms, this expands to

\begin{align*}
\operatorname{Var}(s_k)
  &= s_k^{2}\,(1 - s_k)^{2}\,v_k
  \;+\;
  \sum_{i \neq k} s_k^{2}\,s_i^{2}\,v_i \\
  &= s_k^{2}\Bigl[(1 - s_k)^{2}\,v_k + \sum_{i \neq k} s_i^{2}\,v_i\Bigr].
\end{align*}

\begin{lstlisting}[style=arxivpython]
def softmax_var(y_mean, x_var, axis=-1):
    y = y_mean.transpose(axis, -1)
    v = x_var.transpose(axis, -1)
    W = y.pow(2) * v
    S = W.sum(dim=-1, keepdim=True)
    sum_excluding_k = S - W
    diag_term = (1 - y).pow(2) * v
    var_last = y.pow(2) * (diag_term + sum_excluding_k)
    return var_last.transpose(-1, axis)
\end{lstlisting}

\newpage

\newpage

\clearpage


\section{Tables with Standard Deviations}

\subsection{Regression}
\begin{table*}[t]
    \centering
    \caption{Results on regression datasets with standard deviations (in $\times 10^{-3}$ units).
    Best values are in \textcolor{purple}{purple}, and
    second-best in \textcolor{teal}{teal}.
    An asterisk (*) indicates a last-layer LLA variant.
    Results are averages over 5 random seeds.
    This is the full version of Table~\ref{tab:results} with stds included.}
    \label{tab:results_std}
    \vspace{0.5em}
    \resizebox{\textwidth}{!}{%
    \begin{tabular}{l|ccc|ccc|ccc}
    \toprule
    \multirow{2}{*}{\textbf{Model}} & \multicolumn{3}{c|}{\textbf{Airline}} & \multicolumn{3}{c|}{\textbf{Year}} & \multicolumn{3}{c}{\textbf{Taxi}} \\
    \cmidrule(lr){2-4} \cmidrule(lr){5-7} \cmidrule(l){8-10}
     & \textbf{NLL} & \textbf{CRPS} & \textbf{CQM} & \textbf{NLL} & \textbf{CRPS} & \textbf{CQM} & \textbf{NLL} & \textbf{CRPS} & \textbf{CQM} \\
    \midrule
    MAP (backbone) & 5.121 ($\pm$0.5) & 18.695 ($\pm$0.6) & 0.148 ($\pm$0.4)
                   & 3.673 ($\pm$0.4) & 5.023 ($\pm$0.5) & 0.134 ($\pm$0.3)
                   & 3.775 ($\pm$0.5) & \textcolor{teal}{3.755} ($\pm$0.4) & 0.211 ($\pm$0.4) \\
    LLA Diag       & 5.125 ($\pm$0.4) & 18.648 ($\pm$0.5) & 0.143 ($\pm$0.3)
                   & 3.647 ($\pm$0.3) & 4.917 ($\pm$0.4) & 0.088 ($\pm$0.2)
                   & 3.722 ($\pm$0.4) & 3.990 ($\pm$0.5) & 0.257 ($\pm$0.3) \\
    LLA KFAC       & 5.127 ($\pm$0.3) & 18.631 ($\pm$0.4) & 0.142 ($\pm$0.3)
                   & 3.648 ($\pm$0.3) & 4.915 ($\pm$0.4) & 0.086 ($\pm$0.2)
                   & 3.706 ($\pm$0.3) & 3.986 ($\pm$0.4) & 0.256 ($\pm$0.3) \\
    LLA*           & 5.127 ($\pm$0.4) & \textcolor{teal}{18.631} ($\pm$0.5) & 0.141 ($\pm$0.3)
                   & 3.648 ($\pm$0.3) & {4.915} ($\pm$0.4) & 0.086 ($\pm$0.2)
                   & 3.726 ($\pm$0.4) & 3.985 ($\pm$0.5) & 0.256 ($\pm$0.3) \\
    LLA* KFAC      & 5.127 ($\pm$0.3) & {18.631} ($\pm$0.4) & 0.141 ($\pm$0.3)
                   & 3.648 ($\pm$0.3) & \textcolor{teal}{4.914} ($\pm$0.4) & 0.086 ($\pm$0.2)
                   & 3.726 ($\pm$0.4) & 3.985 ($\pm$0.4) & 0.256 ($\pm$0.3) \\
    ELLA           & 5.388 ($\pm$0.6) & 21.671 ($\pm$0.7) & 0.413 ($\pm$0.5)
                   & 4.020 ($\pm$0.5) & 6.049 ($\pm$0.6) & 0.424 ($\pm$0.4)
                   & 3.885 ($\pm$0.5) & \textcolor{purple}{3.680} ($\pm$0.4) & 0.219 ($\pm$0.4) \\
    VaLLA 100      & \textcolor{teal}{4.963} ($\pm$0.3) & 18.814 ($\pm$0.5) & \textcolor{teal}{0.099} ($\pm$0.2)
                   & {3.515} ($\pm$0.3) & 5.004 ($\pm$0.5) & {0.047} ($\pm$0.2)
                   & {3.235} ($\pm$0.3) & 3.999 ($\pm$0.4) & {0.149} ($\pm$0.2) \\
    VaLLA 200      & {4.965} ($\pm$0.3) & 18.788 ($\pm$0.4) & \textcolor{purple}{0.098} ($\pm$0.2)
                   & \textcolor{teal}{3.485} ($\pm$0.3) & 4.970 ($\pm$0.4) & \textcolor{teal}{0.041} ($\pm$0.2)
                   & \textcolor{teal}{3.232} ($\pm$0.3) & {3.979} ($\pm$0.4) & \textcolor{teal}{0.142} ($\pm$0.2) \\
    Dropout        & 5.102 ($\pm$0.5) & 19.066 ($\pm$0.6) & 0.938 ($\pm$0.5)
                   & 3.689 ($\pm$0.5) & 5.128 ($\pm$0.5) & 0.939 ($\pm$0.4)
                   & 3.849 ($\pm$0.6) & 4.592 ($\pm$0.6) & 0.951 ($\pm$0.5) \\
    Ensemble       & 5.053 ($\pm$0.4) & 18.205 ($\pm$0.5) & 0.933 ($\pm$0.4)
                   & 3.639 ($\pm$0.4) & 4.833 ($\pm$0.5) & 0.938 ($\pm$0.4)
                   & 3.631 ($\pm$0.5) & 3.384 ($\pm$0.5) & 0.961 ($\pm$0.4) \\
    \textbf{GAPA (ours)}
                   & \textcolor{purple}{4.946} ($\pm$0.3) & \textcolor{purple}{18.068} ($\pm$0.4) & {0.103} ($\pm$0.3)
                   & \textcolor{purple}{3.470} ($\pm$0.3) & \textcolor{purple}{4.663} ($\pm$0.4) & \textcolor{purple}{0.014} ($\pm$0.2)
                   & \textcolor{purple}{3.112} ($\pm$0.3) & 4.035 ($\pm$0.4) & \textcolor{purple}{0.104} ($\pm$0.2) \\
    \bottomrule
    \end{tabular}
    }
    \vspace{0.5em}
\end{table*}

\subsection{Feedforward Neural Network Classification}
\begin{table*}[t]
  \centering
  \caption{Results on classification datasets with standard deviations (in $\times 10^{-3}$ units).
  Best values are in \textcolor{purple}{purple}, second-best in \textcolor{teal}{teal}.
  Values are averages over 5 random seeds; consistent with $<10^{-3}$ in all cases.}
  \label{tab:results_classif_std}
  \vspace{0.5em}
  \resizebox{\textwidth}{!}{%
    \begin{tabular}{l|ccccc|ccccc}
      \toprule
      \multirow{2}{*}{\textbf{Model}}
        & \multicolumn{5}{c|}{\textbf{MNIST}}
        & \multicolumn{5}{c}{\textbf{FMNIST}} \\
      \cmidrule(lr){2-6} \cmidrule(lr){7-11}
        & \textbf{ACC} & \textbf{NLL} & \textbf{ECE} & \textbf{OOD} & \textbf{BALD}
        & \textbf{ACC} & \textbf{NLL} & \textbf{ECE} & \textbf{OOD} & \textbf{BALD} \\
      \midrule
      MAP
        & \textcolor{teal}{0.978} ($\pm$0.4) & \textcolor{purple}{0.068} ($\pm$0.2) & \textcolor{purple}{0.005} ($\pm$0.3) & 0.919 ($\pm$0.5) & 0.919 ($\pm$0.4)
        & 0.859 ($\pm$0.3) & 0.392 ($\pm$0.6) & \textcolor{purple}{0.007} ($\pm$0.3) & 0.846 ($\pm$0.5) & 0.821 ($\pm$0.5) \\
      LLA Diag
        & 0.976 ($\pm$0.5) & 0.177 ($\pm$0.5) & 0.105 ($\pm$0.6) & 0.932 ($\pm$0.6) & 0.941 ($\pm$0.5)
        & 0.856 ($\pm$0.4) & 0.421 ($\pm$0.5) & 0.057 ($\pm$0.4) & 0.872 ($\pm$0.5) & 0.873 ($\pm$0.6) \\
      LLA KFAC
        & \textcolor{teal}{0.978} ($\pm$0.4) & 0.102 ($\pm$0.4) & 0.042 ($\pm$0.4) & \textcolor{purple}{0.971} ($\pm$0.3) & \textcolor{teal}{0.971} ($\pm$0.4)
        & 0.858 ($\pm$0.4) & 0.395 ($\pm$0.5) & 0.020 ($\pm$0.3) & 0.909 ($\pm$0.4) & \textcolor{teal}{0.970} ($\pm$0.5) \\
      LLA*
        & \textcolor{teal}{0.978} ($\pm$0.4) & \textcolor{teal}{0.070} ($\pm$0.3) & \textcolor{teal}{0.009} ($\pm$0.3) & 0.924 ($\pm$0.5) & 0.924 ($\pm$0.5)
        & 0.859 ($\pm$0.4) & 0.395 ($\pm$0.5) & 0.019 ($\pm$0.3) & 0.850 ($\pm$0.5) & 0.716 ($\pm$0.5) \\
      LLA* KFAC
        & \textcolor{purple}{0.979} ($\pm$0.3) & \textcolor{teal}{0.070} ($\pm$0.3) & \textcolor{teal}{0.009} ($\pm$0.2) & 0.923 ($\pm$0.4) & 0.928 ($\pm$0.5)
        & 0.859 ($\pm$0.4) & 0.394 ($\pm$0.5) & 0.017 ($\pm$0.3) & 0.849 ($\pm$0.4) & 0.717 ($\pm$0.6) \\
      ELLA
        & \textcolor{teal}{0.978} ($\pm$0.4) & \textcolor{purple}{0.068} ($\pm$0.3) & \textcolor{purple}{0.005} ($\pm$0.2) & 0.919 ($\pm$0.4) & 0.912 ($\pm$0.5)
        & 0.859 ($\pm$0.4) & 0.392 ($\pm$0.5) & \textcolor{purple}{0.007} ($\pm$0.3) & 0.846 ($\pm$0.4) & 0.765 ($\pm$0.6) \\
      VaLLA 100
        & \textcolor{teal}{0.978} ($\pm$0.3) & \textcolor{purple}{0.068} ($\pm$0.3) & \textcolor{purple}{0.005} ($\pm$0.2) & 0.919 ($\pm$0.4) & 0.934 ($\pm$0.4)
        & \textcolor{teal}{0.865} ($\pm$0.3) & \textcolor{teal}{0.382} ($\pm$0.4) & 0.019 ($\pm$0.3) & 0.925 ($\pm$0.4) & 0.963 ($\pm$0.5) \\
      VaLLA 200
        & \textcolor{teal}{0.978} ($\pm$0.4) & \textcolor{purple}{0.068} ($\pm$0.3) & \textcolor{purple}{0.005} ($\pm$0.2) & 0.919 ($\pm$0.4) & 0.934 ($\pm$0.4)
        & \textcolor{purple}{0.867} ($\pm$0.3) & \textcolor{purple}{0.378} ($\pm$0.4) & 0.020 ($\pm$0.3) & \textcolor{teal}{0.937} ($\pm$0.4) & \textcolor{teal}{0.970} ($\pm$0.5) \\
      Linear Probing
        & 0.977 ($\pm$0.4) & 0.117 ($\pm$0.4) & 0.015 ($\pm$0.4) & 0.884 ($\pm$0.5) & 0.883 ($\pm$0.5)
        & 0.858 ($\pm$0.4) & 0.395 ($\pm$0.5) & 0.048 ($\pm$0.5) & 0.785 ($\pm$0.5) & 0.776 ($\pm$0.5) \\
      GPP
        & \textcolor{teal}{0.978} ($\pm$0.3) & 1.648 ($\pm$0.5) & 0.784 ($\pm$0.5) & 0.934 ($\pm$0.5) & 0.904 ($\pm$0.5)
        & 0.857 ($\pm$0.4) & 1.716 ($\pm$0.5) & 0.692 ($\pm$0.6) & 0.867 ($\pm$0.5) & 0.962 ($\pm$0.5) \\
      Dropout
        & \textcolor{teal}{0.978} ($\pm$0.4) & 0.072 ($\pm$0.3) & 0.009 ($\pm$0.3) & 0.923 ($\pm$0.4) & 0.944 ($\pm$0.4)
        & 0.858 ($\pm$0.4) & 0.393 ($\pm$0.5) & 0.009 ($\pm$0.3) & 0.850 ($\pm$0.4) & 0.911 ($\pm$0.4) \\
      Ensemble
        & \textcolor{purple}{0.979} ($\pm$0.3) & 0.069 ($\pm$0.3) & 0.038 ($\pm$0.5) & 0.936 ($\pm$0.5) & 0.962 ($\pm$0.4)
        & 0.839 ($\pm$0.5) & 0.473 ($\pm$0.6) & 0.041 ($\pm$0.4) & 0.876 ($\pm$0.5) & 0.983 ($\pm$0.5) \\
      \textbf{GAPA (ours)}
        & \textcolor{teal}{0.978} ($\pm$0.3) & 0.109 ($\pm$0.4) & 0.049 ($\pm$0.4) & \textcolor{teal}{0.960} ($\pm$0.4) & \textcolor{purple}{0.972} ($\pm$0.4)
        & 0.859 ($\pm$0.4) & 0.389 ($\pm$0.5) & \textcolor{teal}{0.013} ($\pm$0.3) & \textcolor{purple}{0.973} ($\pm$0.4) & \textcolor{purple}{0.993} ($\pm$0.3) \\
      \bottomrule
    \end{tabular}%
  }
  \vspace{0.5em}
\end{table*}

\subsection{ResNet}

\begin{table}[t]
  \centering
  \caption{GAPA and baselines on CIFAR-10 with ResNet-20 and ResNet-32.
  Best in \textcolor{purple}{purple}, second-best in \textcolor{teal}{teal}.
  Results shown as mean ($\pm$ std) over 5 runs.}
  \label{tab:resnet_20_32}

  \footnotesize
  \setlength{\tabcolsep}{2.2pt}
  \renewcommand{\arraystretch}{1.05}

  \begin{adjustbox}{width=\columnwidth}
  \begin{tabular}{l|ccccc|ccccc}
    \toprule
    & \multicolumn{5}{c|}{\textbf{ResNet-20}}
    & \multicolumn{5}{c}{\textbf{ResNet-32}} \\
    \cmidrule(lr){2-6}\cmidrule(lr){7-11}
    & ACC & NLL & OOD & Train & Test
    & ACC & NLL & OOD & Train & Test \\
    \midrule

    MAP
    & 92.6 ($\pm$0.06) & 0.282 ($\pm$0.10) & 0.876 ($\pm$0.10) & -- & --
    & 93.5 ($\pm$0.06) & 0.292 ($\pm$0.10) & 0.909 ($\pm$0.10) & -- & -- \\

    MF-VI
    & \textcolor{teal}{92.7 ($\pm$0.12)} & 0.231 ($\pm$0.15) & 0.865 ($\pm$0.12) & 0.74 ($\pm$0.04) & \textcolor{purple}{0.47 ($\pm$0.02)}
    & \textcolor{teal}{93.5 ($\pm$0.11)} & 0.222 ($\pm$0.15) & 0.885 ($\pm$0.12) & 1.19 ($\pm$0.06) & \textcolor{purple}{0.75 ($\pm$0.04)} \\

    SNGP
    & 92.4 ($\pm$0.08) & 0.266 ($\pm$0.12) & 0.875 ($\pm$0.10) & 15.9 ($\pm$0.8) & 1.31 ($\pm$0.07)
    & 93.2 ($\pm$0.08) & 0.256 ($\pm$0.12) & 0.890 ($\pm$0.10) & 25.5 ($\pm$1.3) & 2.10 ($\pm$0.11) \\

    LLA Diag
    & 92.6 ($\pm$0.08) & 0.260 ($\pm$0.12) & 0.866 ($\pm$0.10) & 18.4 ($\pm$0.9) & \textcolor{teal}{0.43 ($\pm$0.02)}
    & 93.5 ($\pm$0.07) & 0.242 ($\pm$0.12) & 0.882 ($\pm$0.10) & 29.4 ($\pm$1.5) & \textcolor{teal}{0.69 ($\pm$0.03)} \\

    Sampled LLA
    & 92.5 ($\pm$0.09) & \textcolor{purple}{0.231 ($\pm$0.14)} & 0.885 ($\pm$0.12) & 5.00K ($\pm$0.25K) & 1.14K ($\pm$0.06K)
    & 93.5 ($\pm$0.08) & \textcolor{teal}{0.217 ($\pm$0.14)} & 0.905 ($\pm$0.12) & 8.00K ($\pm$0.40K) & 1.83K ($\pm$0.09K) \\

    VaLLA
    & 92.4 ($\pm$0.10) & \textcolor{teal}{0.231 ($\pm$0.15)} & \textcolor{purple}{0.940 ($\pm$0.12)} & 7.59K ($\pm$0.38K) & 124 ($\pm$6)
    & 93.2 ($\pm$0.09) & \textcolor{purple}{0.212 ($\pm$0.15)} & \textcolor{teal}{0.933 ($\pm$0.12)} & 12.2K ($\pm$0.61K) & 199 ($\pm$10) \\

    \textbf{GAPA (ours)}
    & 92.6 ($\pm$0.07) & 0.258 ($\pm$0.12) & \textcolor{teal}{0.907 ($\pm$0.10)} & \textcolor{purple}{3.65 ($\pm$0.18)} & 1.30 ($\pm$0.07)
    & 93.5 ($\pm$0.07) & 0.259 ($\pm$0.12) & \textcolor{purple}{0.926 ($\pm$0.10)} & \textcolor{purple}{5.84 ($\pm$0.29)} & 2.07 ($\pm$0.10) \\

    \bottomrule
  \end{tabular}
  \end{adjustbox}
\end{table}

\begin{table}[t]
  \centering
  \caption{GAPA and baselines on CIFAR-10 with ResNet-44 and ResNet-56.
  Best in \textcolor{purple}{purple}, second-best in \textcolor{teal}{teal}.
  Results shown as mean ($\pm$ std) over 5 runs.}
  \label{tab:resnet_44_56}

  \footnotesize
  \setlength{\tabcolsep}{2.2pt}
  \renewcommand{\arraystretch}{1.05}

  \begin{adjustbox}{width=\columnwidth}
  \begin{tabular}{l|ccccc|ccccc}
    \toprule
    & \multicolumn{5}{c|}{\textbf{ResNet-44}}
    & \multicolumn{5}{c}{\textbf{ResNet-56}} \\
    \cmidrule(lr){2-6}\cmidrule(lr){7-11}
    & ACC & NLL & OOD & Train & Test
    & ACC & NLL & OOD & Train & Test \\
    \midrule

    MAP
    & 94.0 ($\pm$0.05) & 0.275 ($\pm$0.10) & 0.885 ($\pm$0.10) & -- & 0.761 ($\pm$0.04)
    & 94.4 ($\pm$0.05) & 0.252 ($\pm$0.10) & 0.924 ($\pm$0.10) & -- & 0.949 ($\pm$0.05) \\

    MF-VI
    & \textcolor{teal}{93.9 ($\pm$0.10)} & 0.206 ($\pm$0.14) & 0.890 ($\pm$0.12) & 1.63 ($\pm$0.08) & \textcolor{teal}{1.03 ($\pm$0.05)}
    & \textcolor{teal}{94.4 ($\pm$0.10)} & 0.188 ($\pm$0.14) & 0.929 ($\pm$0.12) & 1.97 ($\pm$0.10) & \textcolor{teal}{1.18 ($\pm$0.06)} \\

    SNGP
    & 93.8 ($\pm$0.07) & 0.242 ($\pm$0.12) & 0.901 ($\pm$0.10) & 35.0 ($\pm$1.8) & 2.89 ($\pm$0.14)
    & 93.8 ($\pm$0.07) & 0.229 ($\pm$0.12) & \textcolor{teal}{0.940 ($\pm$0.10)} & 43.5 ($\pm$2.2) & 3.01 ($\pm$0.15) \\

    LLA Diag
    & 94.0 ($\pm$0.07) & 0.218 ($\pm$0.12) & 0.860 ($\pm$0.10) & 40.4 ($\pm$2.0) & \textcolor{purple}{0.947 ($\pm$0.05)}
    & 94.3 ($\pm$0.06) & 0.195 ($\pm$0.12) & 0.923 ($\pm$0.10) & 40.7 ($\pm$2.0) & \textcolor{purple}{1.12 ($\pm$0.06)} \\

    Sampled LLA
    & 94.0 ($\pm$0.08) & \textcolor{purple}{0.200 ($\pm$0.13)} & 0.899 ($\pm$0.12) & 11.0K ($\pm$0.55K) & 2.51K ($\pm$0.13K)
    & 94.4 ($\pm$0.07) & \textcolor{purple}{0.185 ($\pm$0.13)} & \textcolor{teal}{0.944 ($\pm$0.12)} & 14.6K ($\pm$0.73K) & 2.84K ($\pm$0.14K) \\

    VaLLA
    & 93.8 ($\pm$0.09) & \textcolor{teal}{0.201 ($\pm$0.14)} & \textcolor{teal}{0.928 ($\pm$0.12)} & 16.7K ($\pm$0.84K) & 272.9 ($\pm$14)
    & 94.2 ($\pm$0.08) & \textcolor{teal}{0.188 ($\pm$0.14)} & \textcolor{purple}{0.960 ($\pm$0.12)} & 26.3K ($\pm$1.32K) & 363.8 ($\pm$18) \\

    \textbf{GAPA (ours)}
    & 94.0 ($\pm$0.06) & 0.230 ($\pm$0.12) & \textcolor{purple}{0.931 ($\pm$0.10)} & \textcolor{purple}{8.03 ($\pm$0.40)} & 2.85 ($\pm$0.14)
    & 94.4 ($\pm$0.06) & 0.230 ($\pm$0.12) & \textcolor{teal}{0.953 ($\pm$0.10)} & \textcolor{purple}{10.29 ($\pm$0.51)} & 3.30 ($\pm$0.17) \\

    \bottomrule
  \end{tabular}
  \end{adjustbox}
\end{table}


\section{Ablation Studies}
\label{sec:ablation}
We investigate three key design choices in GAPA: layer placement, number of inducing points, and sampling strategy.

\begin{figure}[htbp]
    \centering 
    \includegraphics[width=0.49\textwidth]{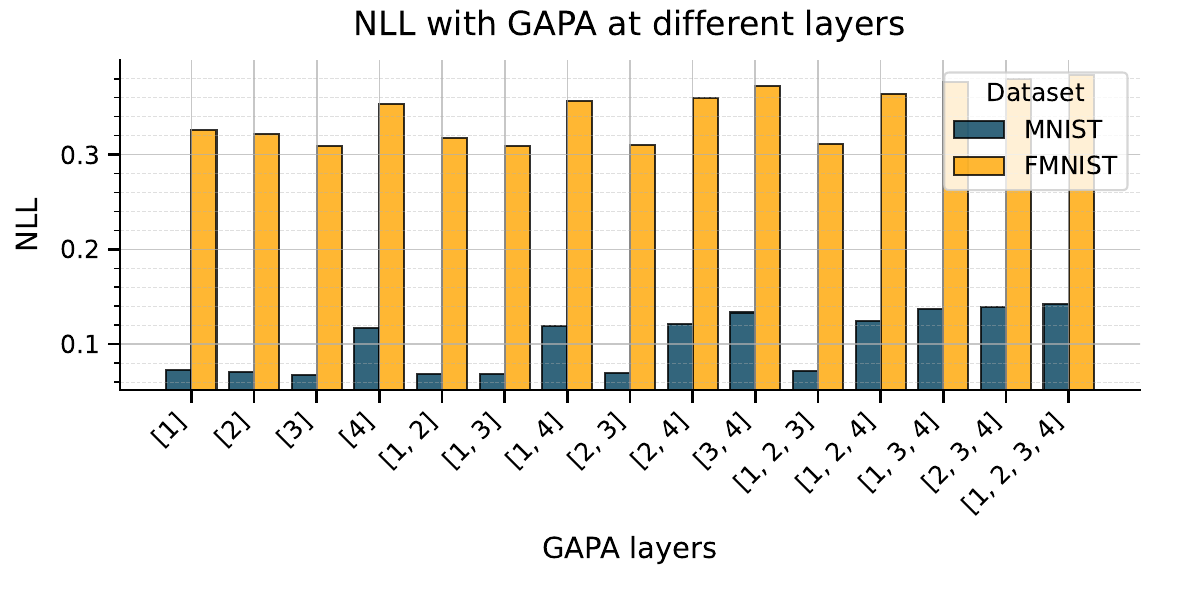}
    \includegraphics[width=0.49\textwidth]{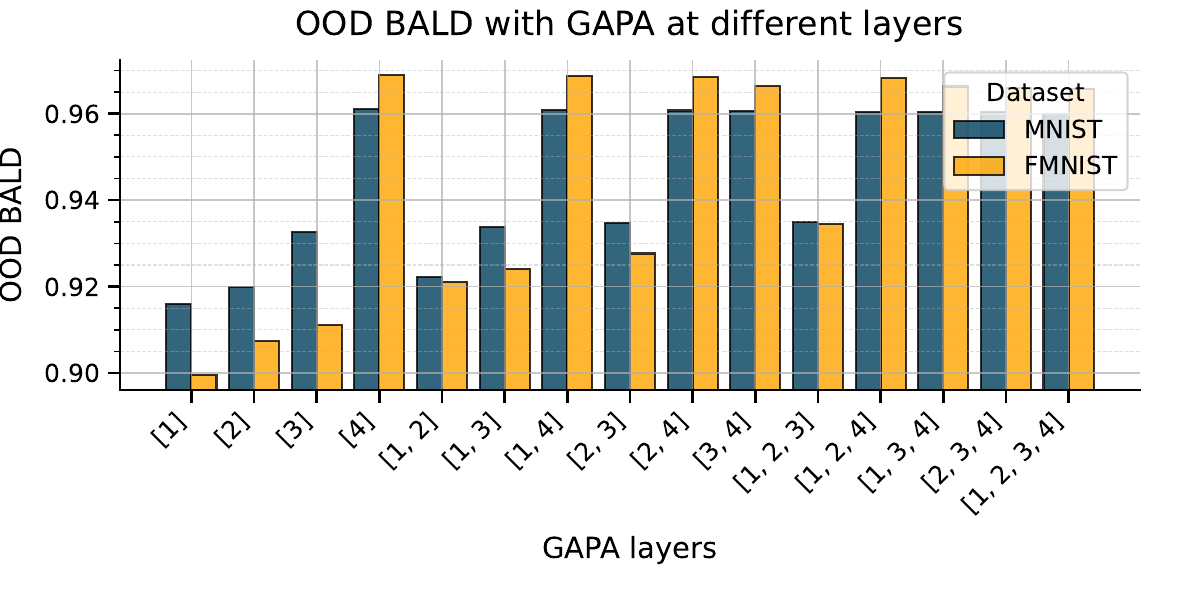}

    \caption{Comparison of metrics at different GAPA layer placements (\(M=55{,}000\))}
    \label{fig:example_M_points}
\end{figure}

\subsection{Where to put GAPA}
\begin{table*}[t]
\centering
\caption{Comparison of metrics at different GAPA layer placements (\(M=55{,}000\)).
Best values are \textbf{bold}. Lower is better (↓) for NLL; higher is better (↑) for OOD-AUC/BALD.}
\label{tab:gapa_layer_ablation}
\vspace{0.3em}
\setlength{\tabcolsep}{6pt}
\renewcommand{\arraystretch}{1.05}
\resizebox{\textwidth}{!}{%
\begin{tabular}{l|ccc|ccc}
\toprule
\multirow{2}{*}{\textbf{GAPA layers}}
  & \multicolumn{3}{c|}{\textbf{MNIST}} 
  & \multicolumn{3}{c}{\textbf{FMNIST}} \\
\cmidrule(lr){2-4}\cmidrule(l){5-7}
  & \textbf{NLL} $\downarrow$ & \textbf{OOD-AUC} $\uparrow$ & \textbf{OOD BALD} $\uparrow$
  & \textbf{NLL} $\downarrow$ & \textbf{OOD-AUC} $\uparrow$ & \textbf{OOD BALD} $\uparrow$ \\
\midrule
{[1]}         & 0.072 & 0.915 & 0.916 & 0.326 & 0.870 & 0.900 \\
{[2]}         & 0.070 & 0.921 & 0.920 & 0.321 & 0.884 & 0.907 \\
{[3]}         & \textbf{0.068} & 0.933 & 0.933 & \textbf{0.309} & 0.901 & 0.911 \\
{[4]}         & 0.117 & 0.951 & 0.957 & 0.353 & \textbf{0.973} & \textbf{0.969} \\
{[1, 2]}      & 0.069 & 0.923 & 0.922 & 0.318 & 0.901 & 0.921 \\
{[1, 3]}      & 0.069 & 0.934 & 0.934 & \textbf{0.309} & 0.912 & 0.924 \\
{[1, 4]}      & 0.120 & \textbf{0.953} & \textbf{0.961} & 0.357 & \textbf{0.973} & \textbf{0.969} \\
{[2, 3]}      & 0.070 & 0.935 & 0.935 & 0.310 & 0.917 & 0.928 \\
{[2, 4]}      & 0.122 & \textbf{0.953} & \textbf{0.961} & 0.360 & \textbf{0.973} & 0.968 \\
{[3, 4]}      & 0.134 & \textbf{0.953} & \textbf{0.961} & 0.372 & \textbf{0.973} & 0.966 \\
{[1, 2, 3]}   & 0.072 & 0.936 & 0.935 & 0.312 & 0.924 & 0.934 \\
{[1, 2, 4]}   & 0.125 & \textbf{0.953} & 0.960 & 0.364 & \textbf{0.973} & 0.968 \\
{[1, 3, 4]}   & 0.137 & \textbf{0.953} & 0.960 & 0.376 & \textbf{0.973} & 0.966 \\
{[2, 3, 4]}   & 0.139 & \textbf{0.953} & 0.960 & 0.380 & \textbf{0.973} & 0.966 \\
{[1, 2, 3, 4]}& 0.142 & \textbf{0.953} & 0.960 & 0.384 & \textbf{0.974} & 0.966 \\
\bottomrule
\end{tabular}
}
\vspace{0.3em}
\end{table*}

Table~\ref{tab:gapa_layer_ablation} (and Figure \ref{fig:example_M_points}) examines GAPA placement across our 4-layer network. For MNIST, placing GAPA at layer 3 achieves the best NLL (0.068), while layer 4 or any combination including layer 4 maximizes OOD detection (0.953 AUC, 0.961 BALD). For FMNIST, similar patterns emerge: layer 3 minimizes NLL (0.309), while layer 4 dominates OOD metrics (0.973 AUC, 0.969 BALD). Interestingly, adding more GAPA layers generally degrades NLL while maintaining strong OOD performance, suggesting a trade-off between calibration and uncertainty awareness. The final layer (closest to output) appears most critical for OOD detection, while intermediate layers better preserve calibration.

\subsection{Number of inducing inputs}
\label{Appedix:n_inducing_points}

\begin{table}[t]
  \centering
  \caption{Metrics across different $M$ values for MNIST and FMNIST, GAPA at the 4th layer.}
  \label{tab:m_values_mnist_fmnist}

  \footnotesize
  \setlength{\tabcolsep}{2.2pt}
  \renewcommand{\arraystretch}{1.05}

  \begin{adjustbox}{width=\columnwidth}
  \begin{tabular}{c|rrrrr|rrrrr}
    \toprule
     & \multicolumn{5}{c|}{MNIST} & \multicolumn{5}{c}{FMNIST} \\
    $M$
    & NLL $\downarrow$ & OOD $\uparrow$ & BALD $\uparrow$ & set up/s $\downarrow$ & inference/s $\downarrow$
    & NLL $\downarrow$ & OOD $\uparrow$ & BALD $\uparrow$ & set up/s $\downarrow$ & inference/s $\downarrow$ \\
    \midrule
    10    & 0.248 & 0.897 & 0.919 & \textbf{2.733}   & \textbf{7.517}  & 0.489 & 0.957 & 0.936 & \textbf{0.257}   & \textbf{7.584} \\
    100   & 0.248 & 0.897 & 0.919 & 185.477          & 7.478           & 0.489 & 0.957 & 0.936 & 181.340          & 7.625 \\
    1000  & 0.246 & 0.898 & 0.920 & 184.787          & 7.674           & 0.486 & 0.957 & 0.937 & 183.503          & 7.763 \\
    5000  & 0.219 & 0.913 & 0.934 & 195.889          & 8.663           & 0.470 & 0.960 & 0.943 & 194.468          & 8.702 \\
    10000 & 0.181 & 0.933 & 0.950 & 212.990          & 10.119          & 0.442 & 0.964 & 0.952 & 211.333          & 9.873 \\
    20000 & 0.139 & 0.947 & 0.958 & 247.684          & 12.498          & 0.390 & 0.970 & 0.964 & 241.000          & 12.164 \\
    40000 & 0.119 & \textbf{0.953} & \textbf{0.961} & 301.511 & 16.926 & 0.355 & 0.972 & 0.968 & 301.086 & 16.826 \\
    55000 & \textbf{0.117} & \textbf{0.953} & \textbf{0.961} & 455.735 & 20.445
          & \textbf{0.353} & \textbf{0.973} & \textbf{0.969} & 384.825 & 20.527 \\
    \bottomrule
  \end{tabular}
  \end{adjustbox}
\end{table}

Table~\ref{tab:gapa_layer_ablation} shows performance as $M$ increases from 10 to 55,000. Both datasets exhibit clear saturation: MNIST plateaus around $M=40,000$ (NLL: 0.119→0.117, OOD: 0.953), while FMNIST shows similar convergence. Computational costs scale sub-linearly due to FAISS indexing—setup time increases from 2.7s to 455s for MNIST, while inference remains tractable (7.5s→20s). This demonstrates GAPA's efficiency: near-optimal uncertainty quantification is achievable with moderate $M$ values, making the method practical for larger models.

\subsection{Inducing point selection: KMeans vs.\ farthest-point sampling}

We compare two strategies for selecting inducing points: the farthest-point
sampling (FPS) method used in the main paper, and the KMeans-based option
introduced in Appendix~\ref{sec:gapa_hyp}.  
Figures~\ref{fig:kmeans_fmnist_nll}--\ref{fig:kmeans_mnist_time} report
results for MNIST and FMNIST across a range of inducing-point budgets $M$.

Overall, both methods exhibit similar behaviour: performance improves
monotonically with $M$ and saturates once a sufficient coverage of the
activation space is achieved.  
KMeans, however, provides a more efficient trade-off between coverage and
inducing-point count, reaching its plateau at substantially smaller $M$
values than FPS.  
This makes KMeans a practical alternative when memory, storage, or index
construction time is a constraint.

\begin{figure*}[t]
    \centering
    \includegraphics[width=0.49\textwidth]{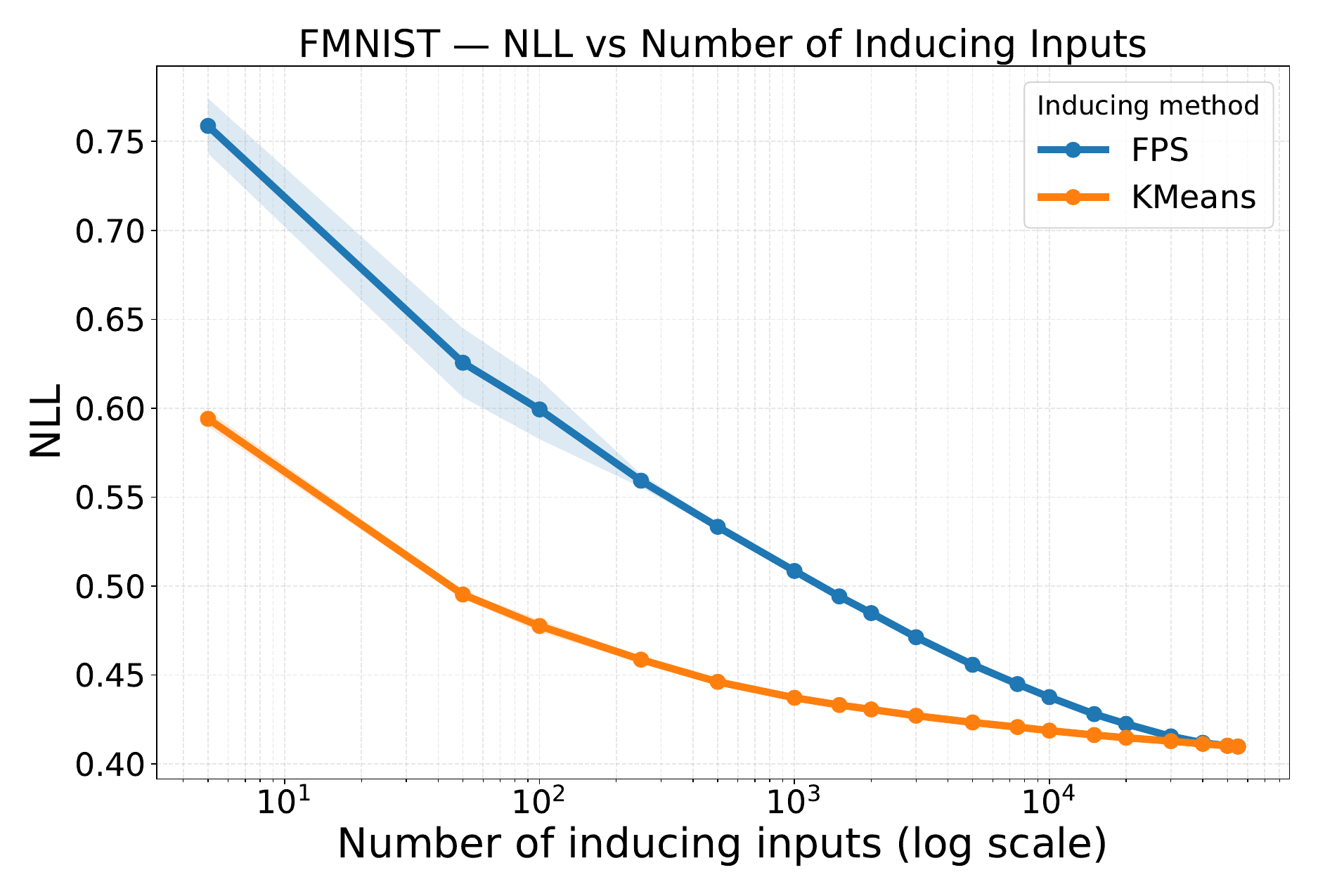}
    \includegraphics[width=0.49\textwidth]{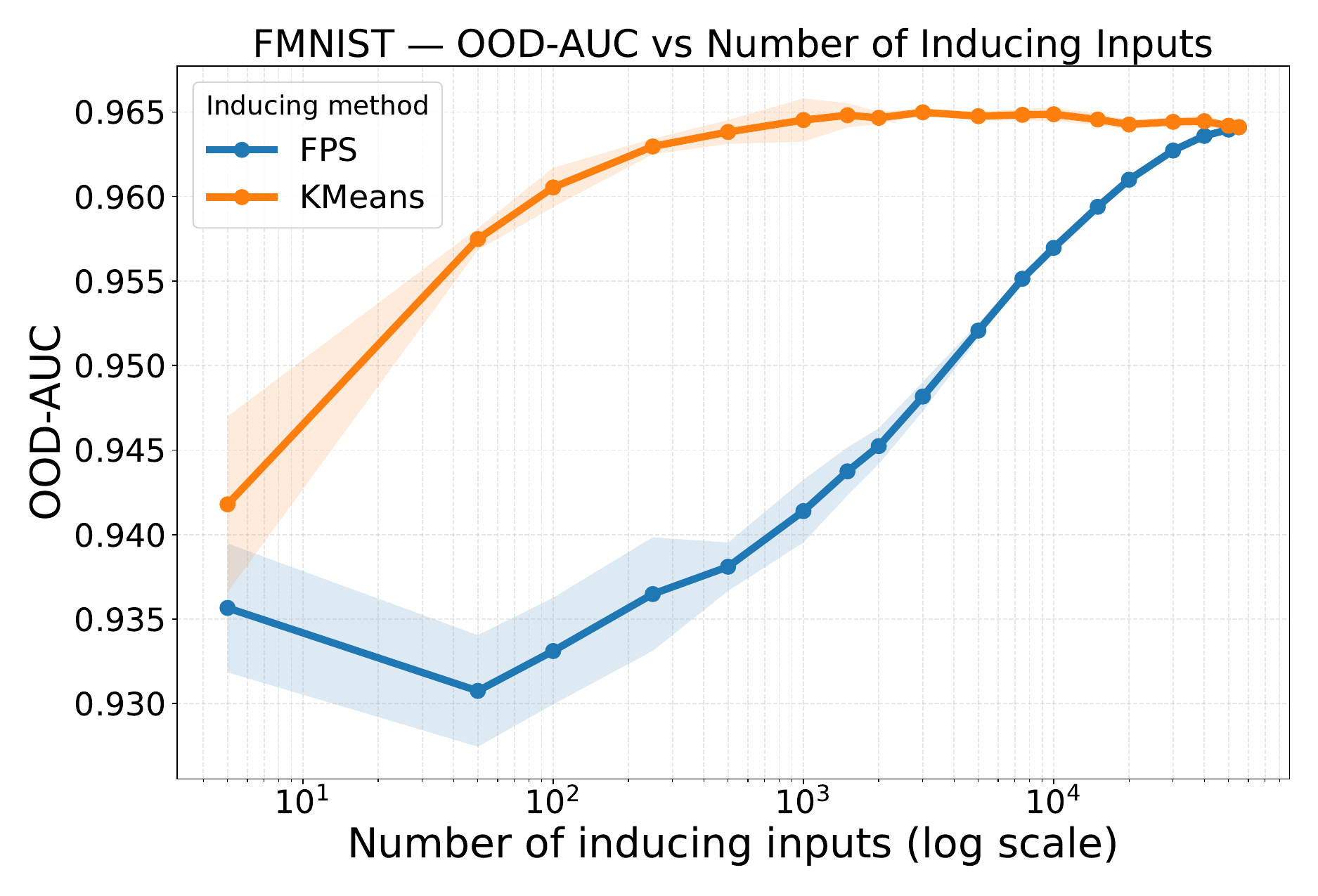}
    \caption{FMNIST: NLL (left) and OOD-AUC (right) for KMeans vs.\ FPS across $M$.}
    \label{fig:kmeans_fmnist_nll}
\end{figure*}

\begin{figure*}[t]
    \centering
    \includegraphics[width=0.49\textwidth]{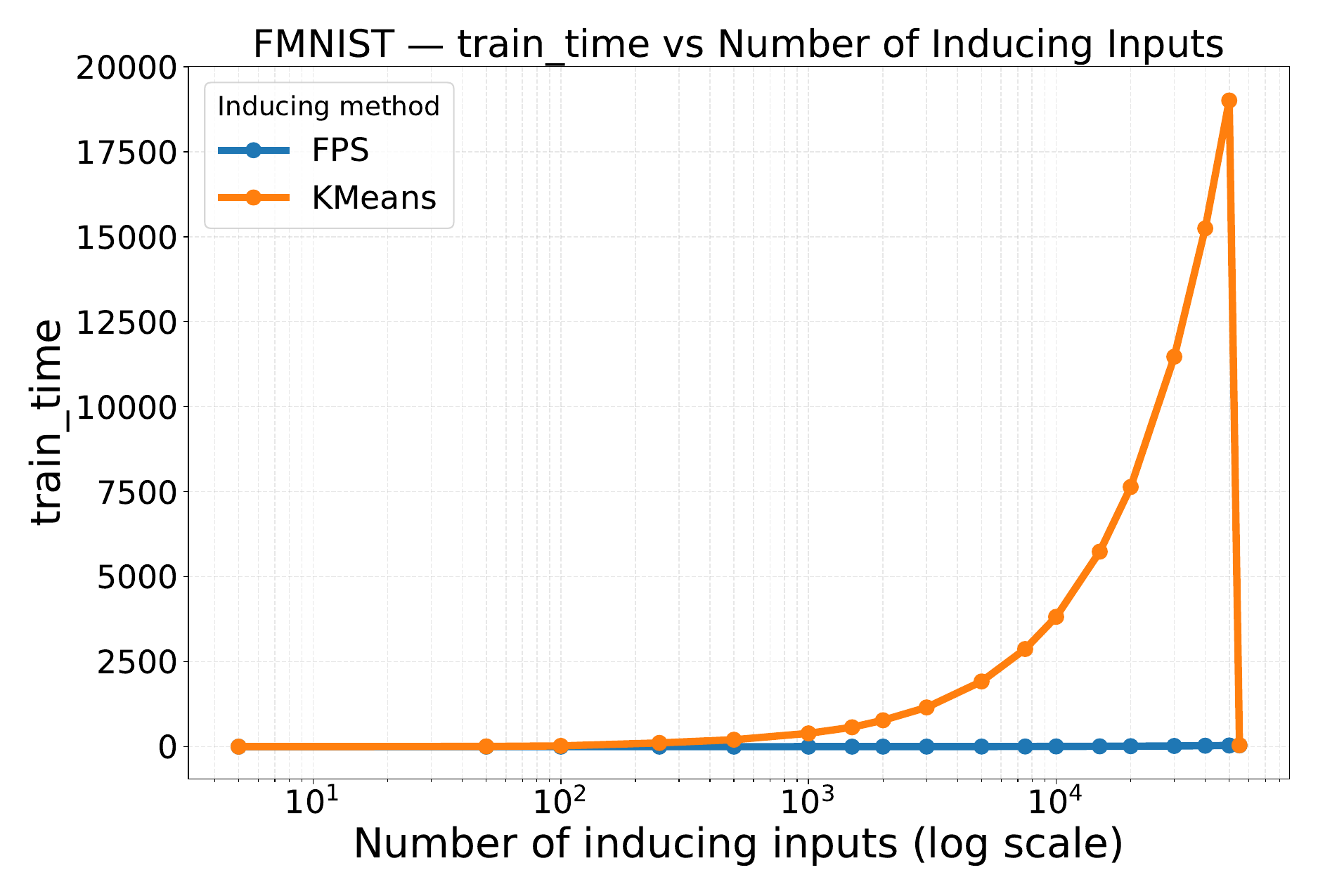}
    \includegraphics[width=0.49\textwidth]{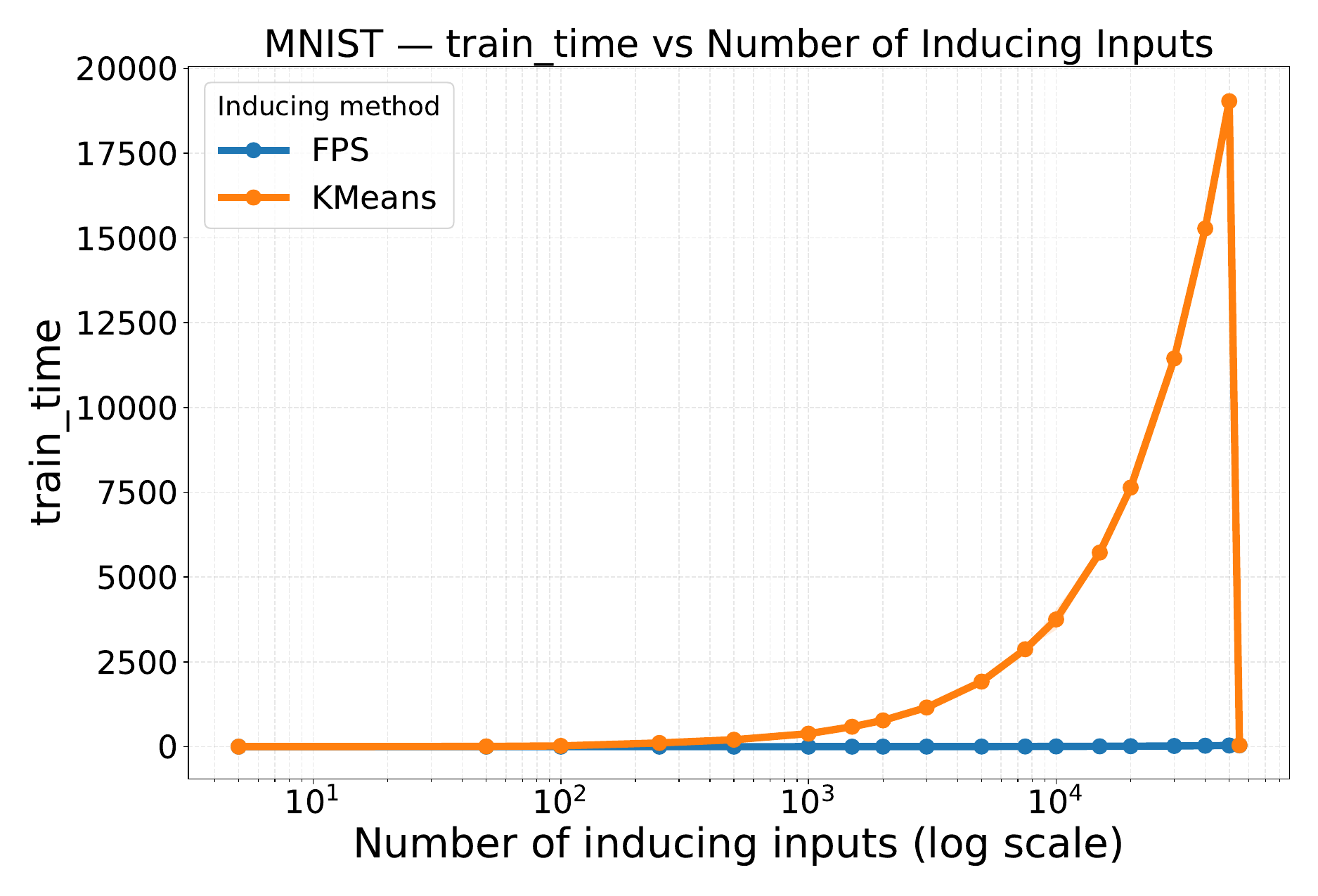}
    \caption{Setup time (FAISS indexing) for KMeans vs.\ FPS on FMNIST (left) and MNIST (right).}
    \label{fig:kmeans_mnist_time}
\end{figure*}

\begin{figure*}[t]
    \centering
    \includegraphics[width=0.49\textwidth]{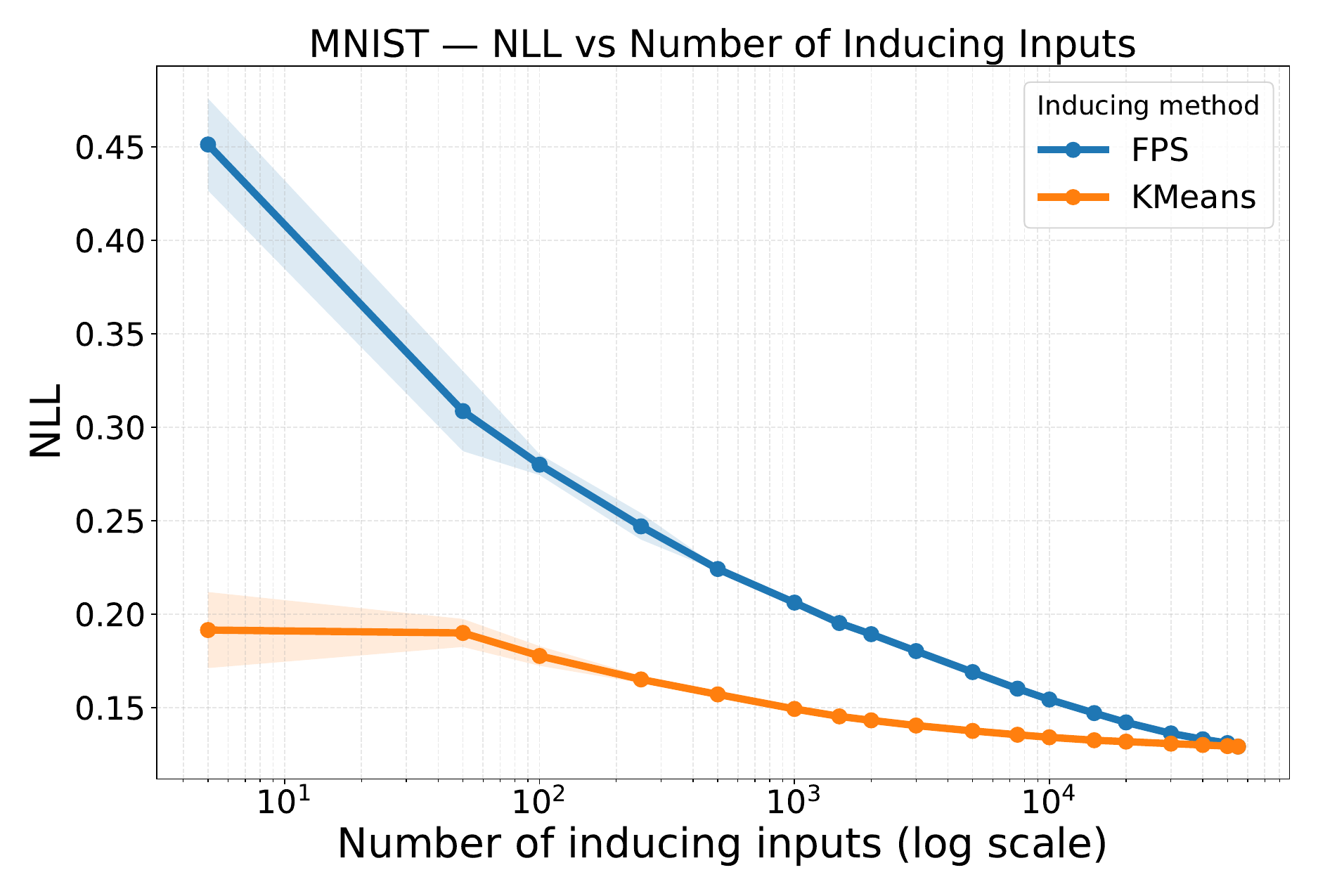}
    \includegraphics[width=0.49\textwidth]{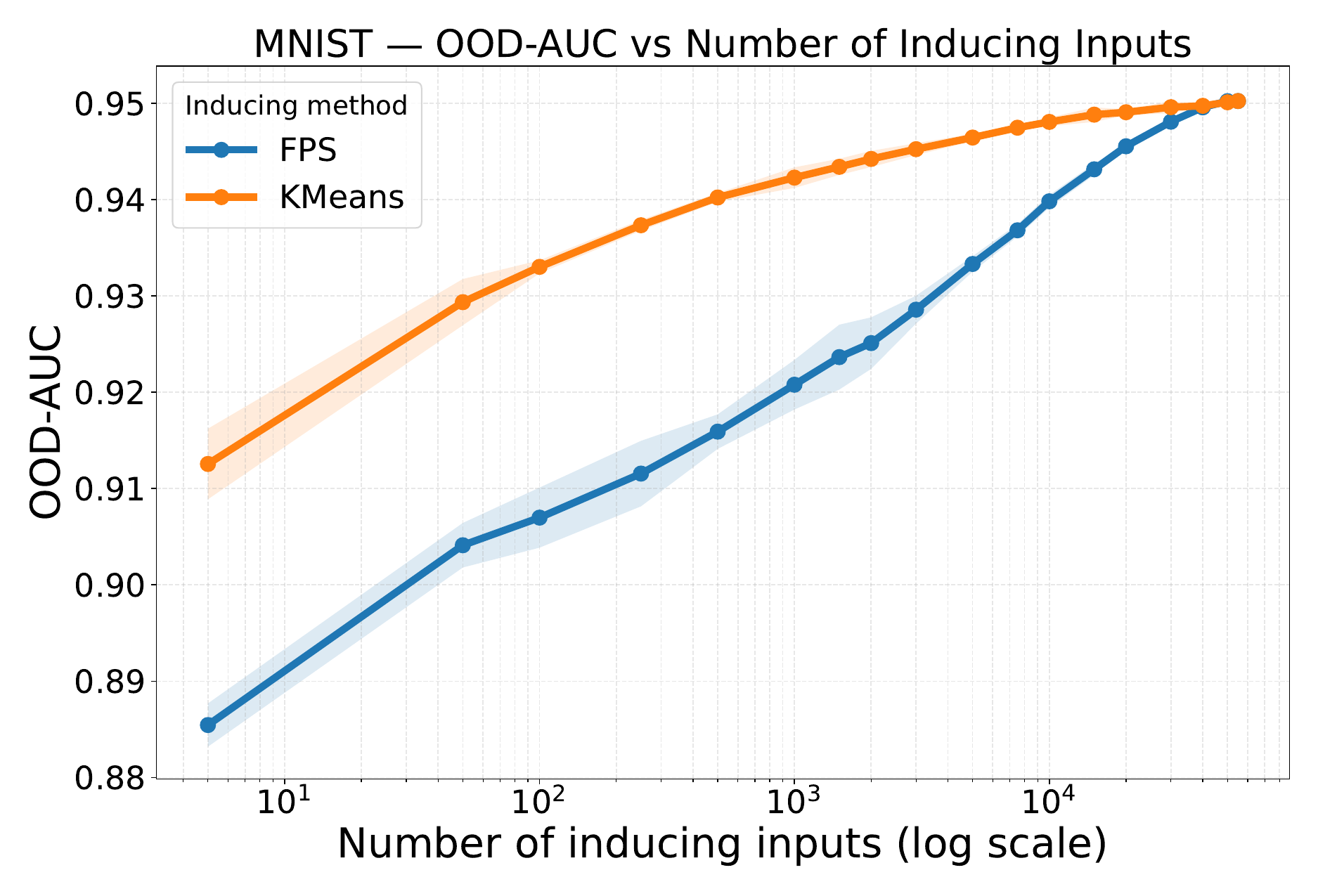}
    \caption{MNIST: NLL (left) and OOD-AUC (right) for KMeans vs.\ FPS across $M$.}
    \label{fig:kmeans_mnist_nll}
\end{figure*}

\begin{figure*}[t]
    \centering
    \includegraphics[width=0.49\textwidth]{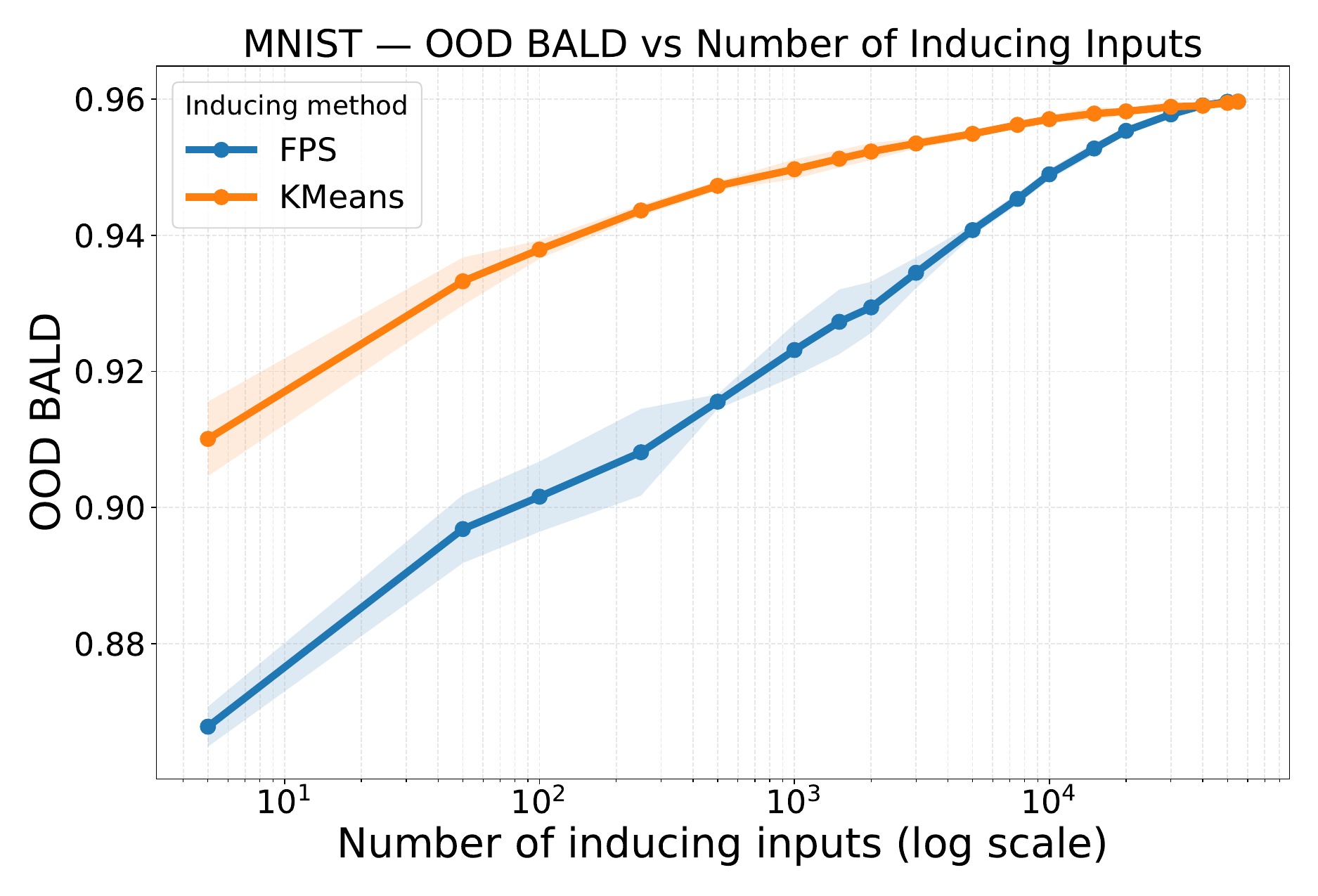}
    \includegraphics[width=0.49\textwidth]{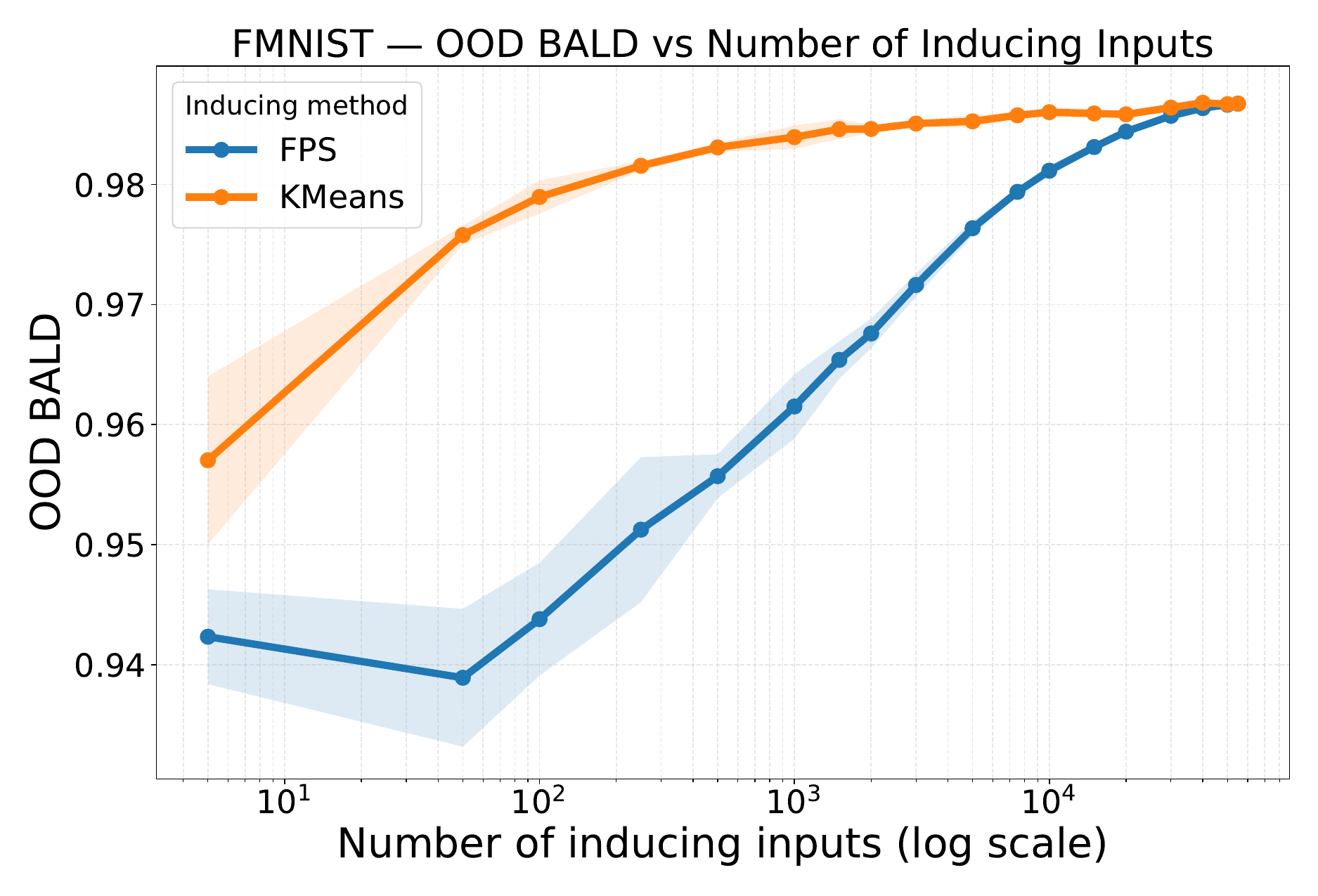}
    \caption{BALD-based OOD detection for MNIST (left) and FMNIST (right).}
    \label{fig:kmeans_mnist_bald}
\end{figure*}

\subsection{Random vs Furthers Point Sampling }

\begin{table*}[t]
\centering
\caption{Comparison of NLL and OOD BALD for FPS and three random baselines (FMNIST, gapa\_index= last layer).}
\label{tab:fps_vs_random}
\resizebox{\textwidth}{!}{%
\begin{tabular}{c|cc|cc|cc|cc}
\toprule
$M$ & FPS NLL↓ & FPS OOD↑ & Rand1 NLL↓ & Rand1 OOD↑ & Rand2 NLL↓ & Rand2 OOD↑ & Rand3 NLL↓ & Rand3 OOD↑ \\
\midrule
5000  & 0.470 & 0.943 & 0.394 & 0.957 & 0.394 & 0.957 & 0.394 & 0.957 \\
10000 & 0.442 & 0.952 & 0.380 & 0.960 & 0.380 & 0.960 & 0.380 & 0.960 \\
20000 & 0.390 & 0.964 & 0.369 & 0.964 & 0.369 & 0.964 & 0.369 & 0.964 \\
40000 & \textbf{0.355} & \textbf{0.968} & \textbf{0.359} & \textbf{0.967} & \textbf{0.359} & \textbf{0.967} & \textbf{0.359} & \textbf{0.967} \\
\bottomrule
\end{tabular}%
}
\end{table*}

Table~\ref{tab:gapa_layer_ablation} reveals that furthest point sampling (FPS) and random sampling exhibit different strengths. At smaller $M$ (5K-10K), random sampling achieves better NLL and OOD detection, likely because FPS's greedy selection may overfit to specific activation patterns. However, as $M$ increases to 40K, FPS shows marginal improvements, suggesting its structured coverage becomes beneficial with sufficient inducing points. The convergence of both methods at large $M$ indicates that with enough inducing points, the activation space is well-covered regardless of sampling strategy. 
\subsection{KNN Sweep: $K = 1$ to $500$}
\label{app:knn_sweep}

To evaluate the robustness of the KNN GAPA approximation used in GAPA, we performed a
comprehensive KNN sweep over $K = \{1,2,3,5,10,20,50,100,150,200,300,400,500\}$ on both
MNIST and FMNIST. For each $K$, we recomputed the GP posterior variance using the $K$
nearest cached activations and measured all uncertainty metrics (NLL, ECE, OOD-AUC,
OOD-BALD) as well as test-time inference cost.

Across all metrics and datasets, the results reveal a strikingly consistent pattern:
\textbf{all curves improve smoothly and monotonically with $K$}, and we observed
no instability—even at $K=1$.

\vspace{1em}
\paragraph{Negative Log-Likelihood (NLL).}
NLL decreases continuously as $K$ increases for both datasets.
MNIST improves from $\approx 0.092$ at $K{=}1$ to $\approx 0.081$ at $K{=}500$.
FMNIST improves from $\approx 0.408$ at $K{=}1$ to $\approx 0.390$ at $K{=}500$.

\begin{figure*}[t]
    \centering
    \begin{minipage}{0.45\textwidth}
        \centering
        \includegraphics[width=\textwidth]{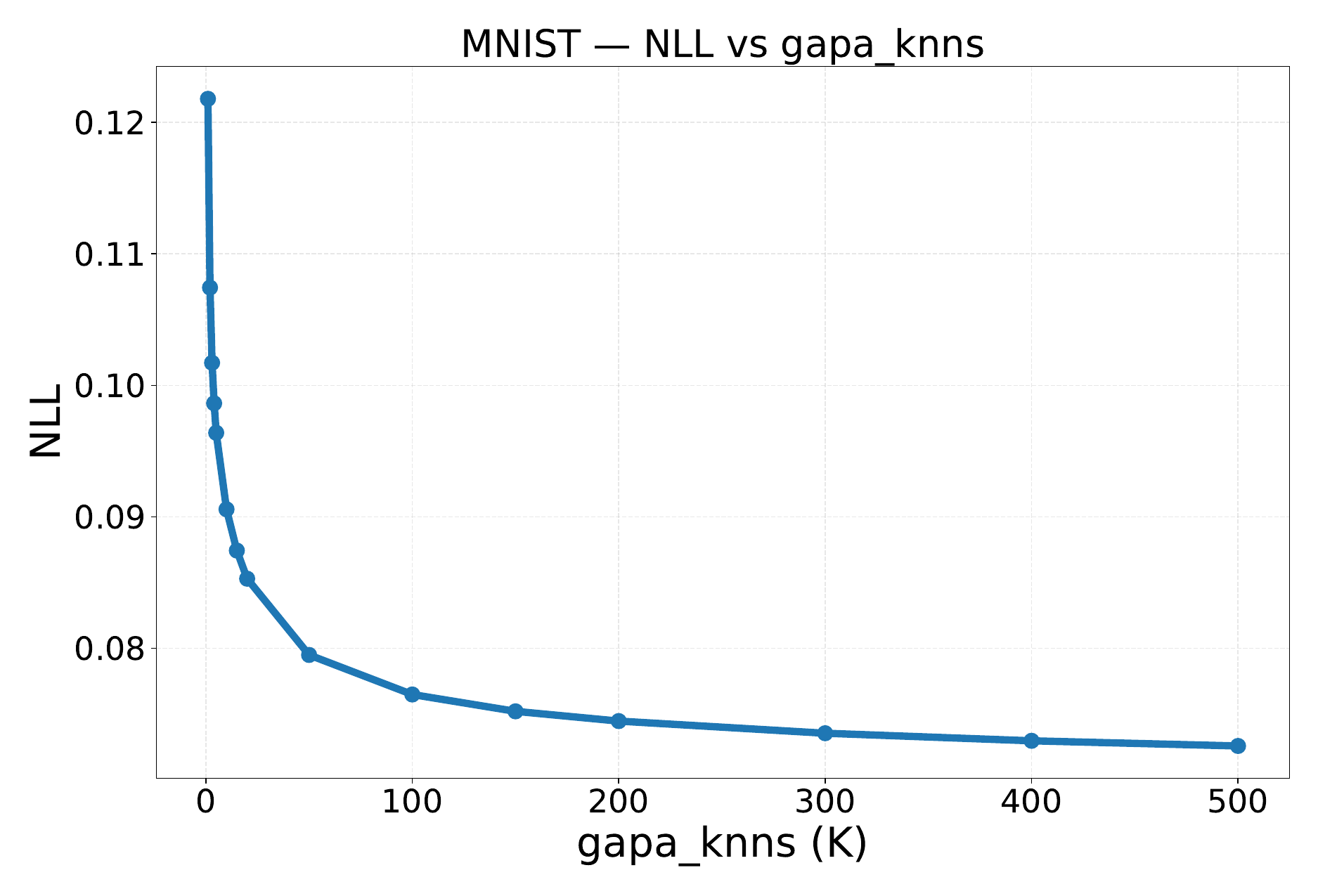}
        \caption{MNIST NLL vs.\ $K$.}
    \end{minipage}
    \hfill
    \begin{minipage}{0.45\textwidth}
        \centering
        \includegraphics[width=\textwidth]{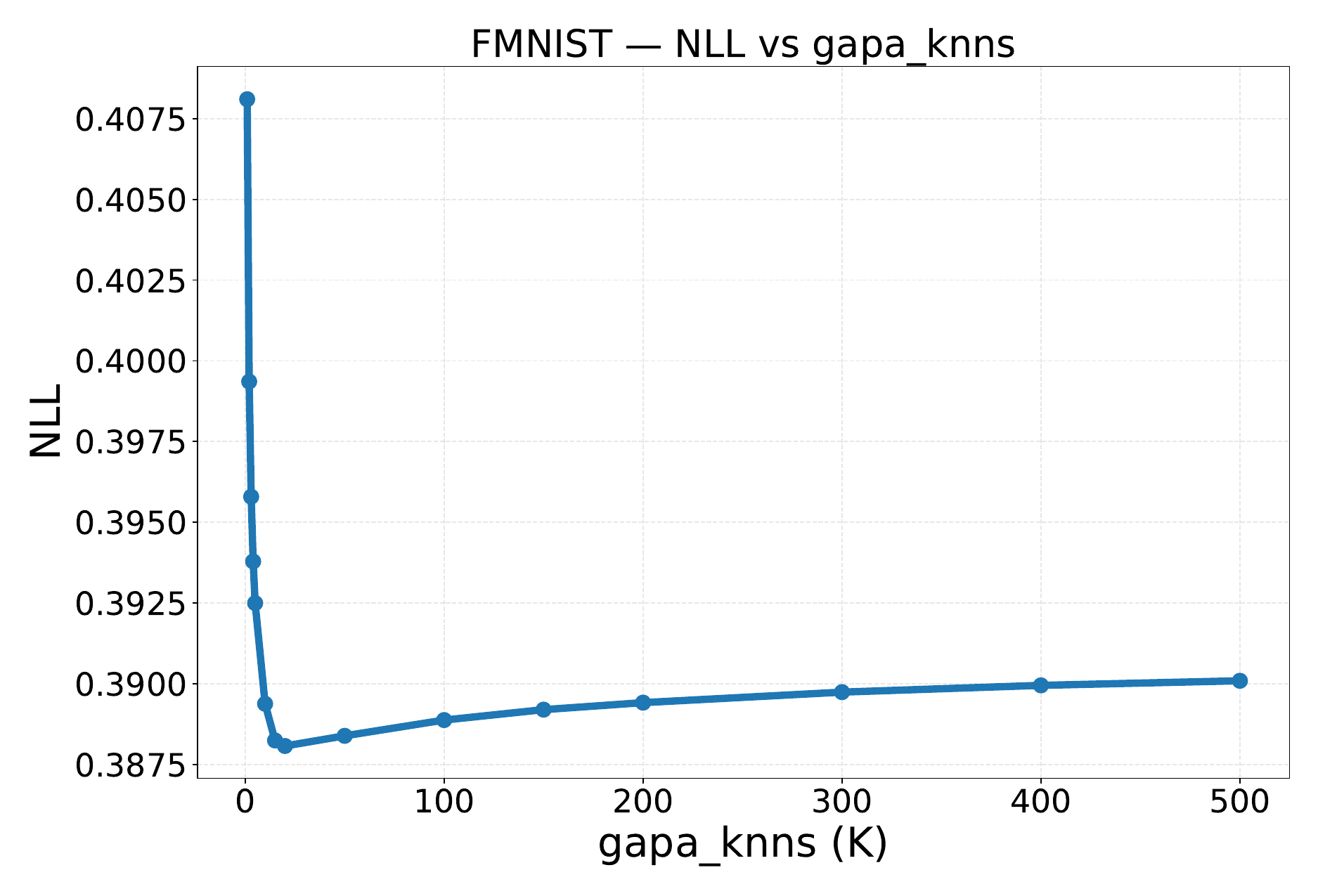}
        \caption{FMNIST NLL vs.\ $K$.}
        \label{fig:knn_fmnist_nll}
    \end{minipage}
\end{figure*}

\vspace{1em}
\paragraph{Expected Calibration Error (ECE).}
ECE improves monotonically for both datasets.
MNIST decreases from $\approx 0.062$ to $\approx 0.015$.
FMNIST shows a similar smooth trend.

\begin{figure*}[t]
    \centering
    \begin{minipage}{0.45\textwidth}
        \centering
        \includegraphics[width=\textwidth]{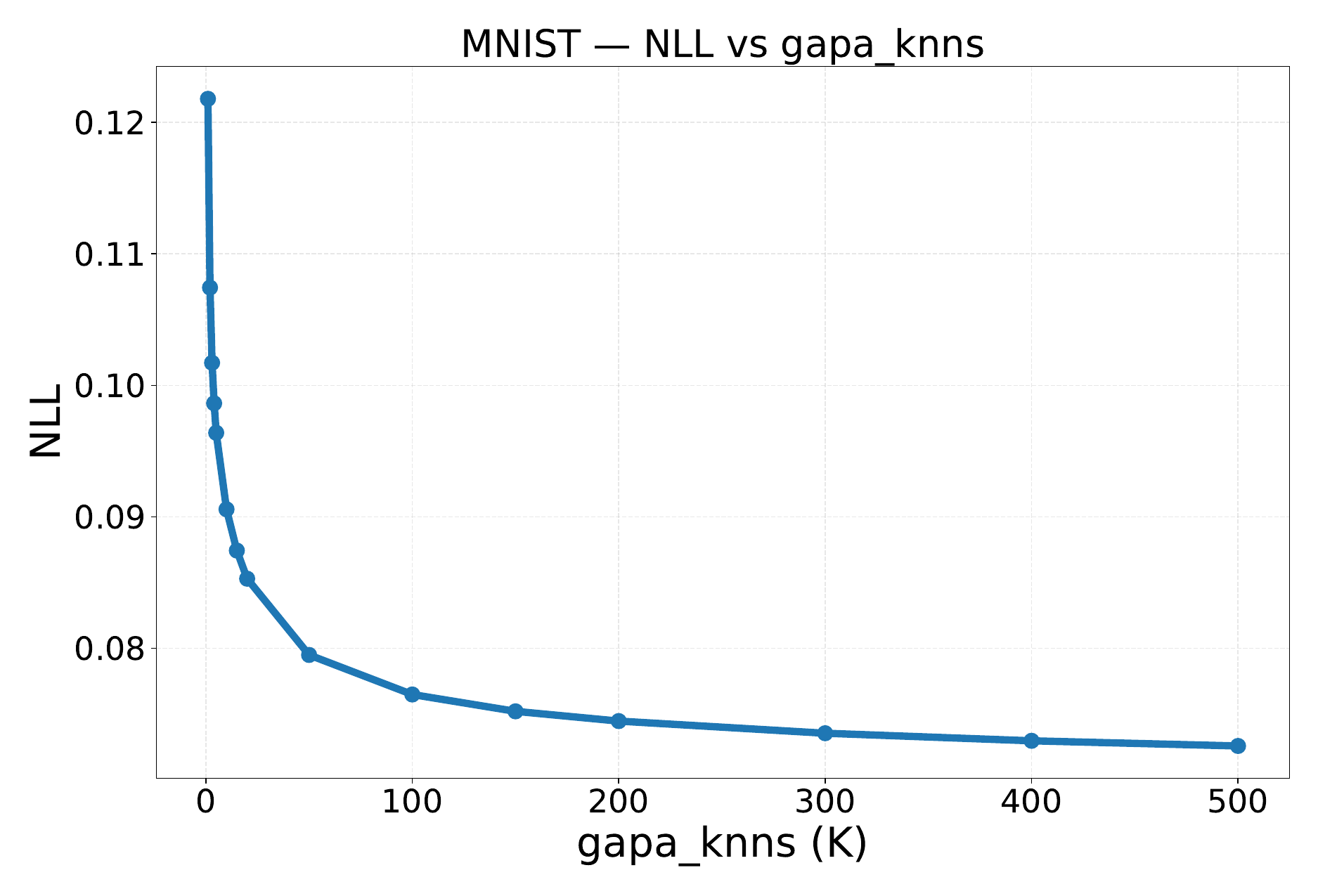}
        \caption{MNIST NLL vs.\ $K$.}
    \end{minipage}
    \hfill
    \begin{minipage}{0.45\textwidth}
        \centering
        \includegraphics[width=\textwidth]{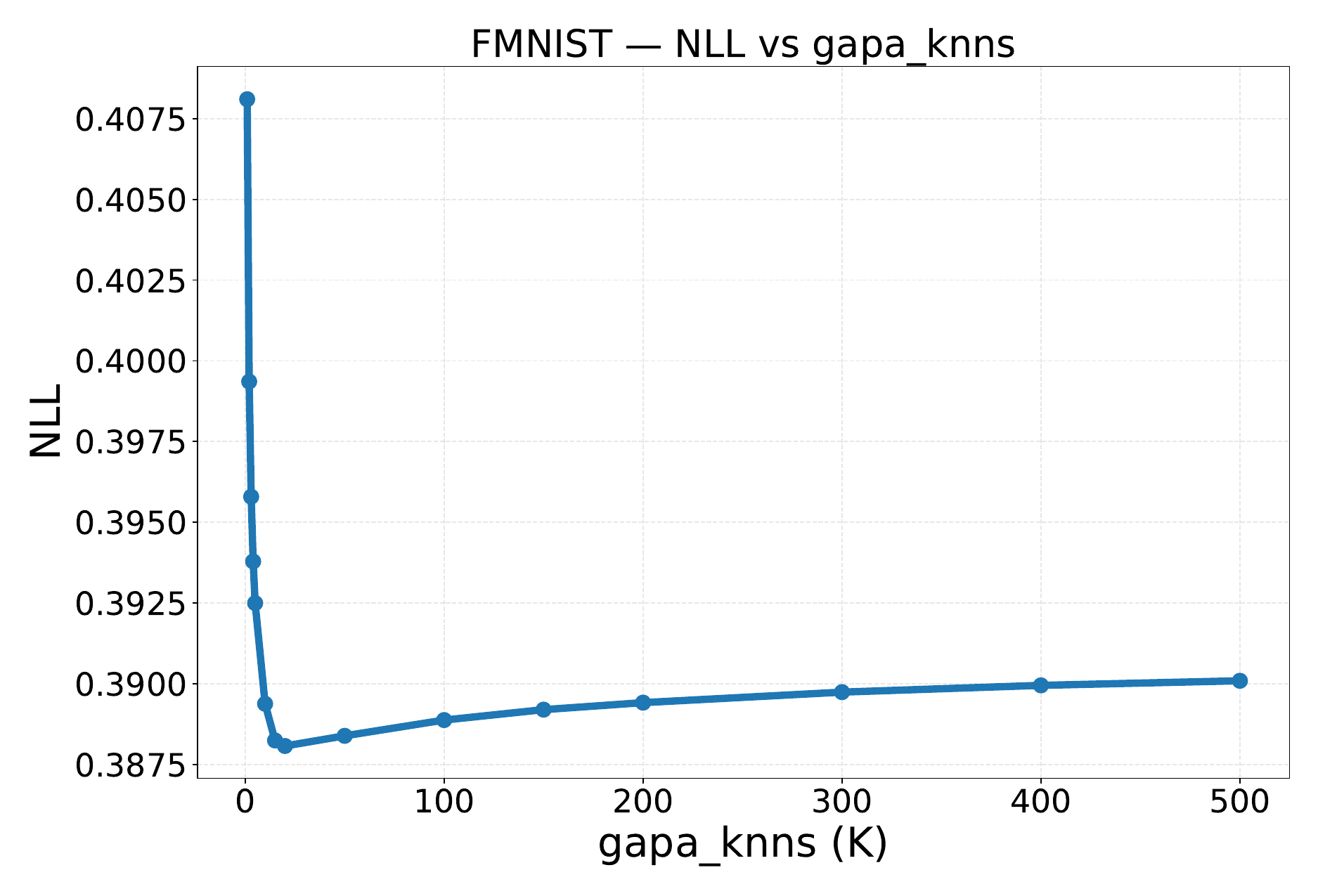}
        \caption{FMNIST NLL vs.\ $K$.}
    \end{minipage}
\end{figure*}

\vspace{1em}
\paragraph{Expected Calibration Error (ECE).}
ECE improves monotonically for both datasets.
MNIST decreases from $\approx 0.062$ to $\approx 0.015$.
FMNIST shows a similar smooth trend.

\begin{figure*}[t]
    \centering
    \begin{minipage}{0.45\textwidth}
        \centering
        \includegraphics[width=\textwidth]{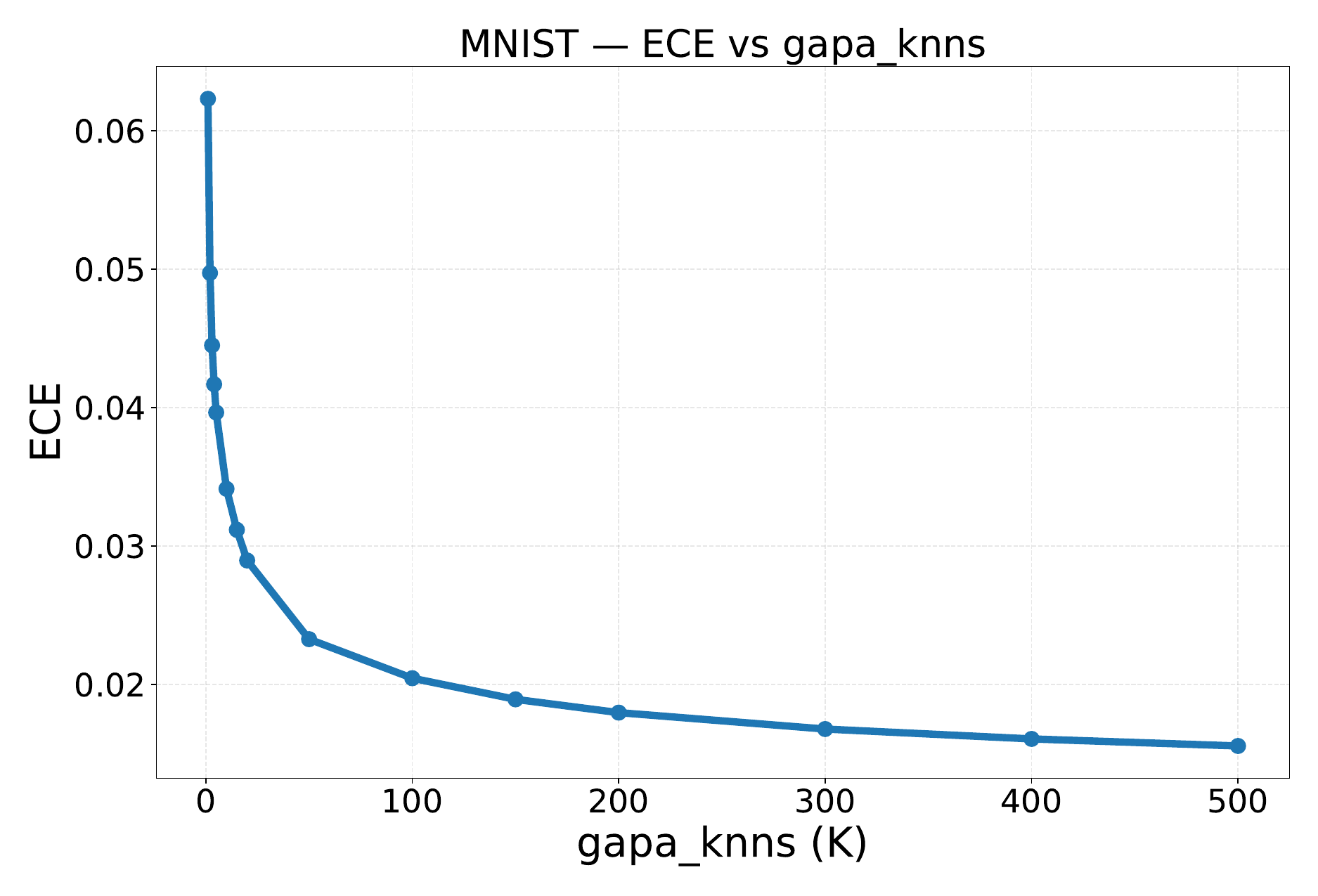}
        \caption{MNIST ECE vs.\ $K$.}
        \label{fig:knn_mnist_ece}
    \end{minipage}
    \hfill
    \begin{minipage}{0.45\textwidth}
        \centering
        \includegraphics[width=\textwidth]{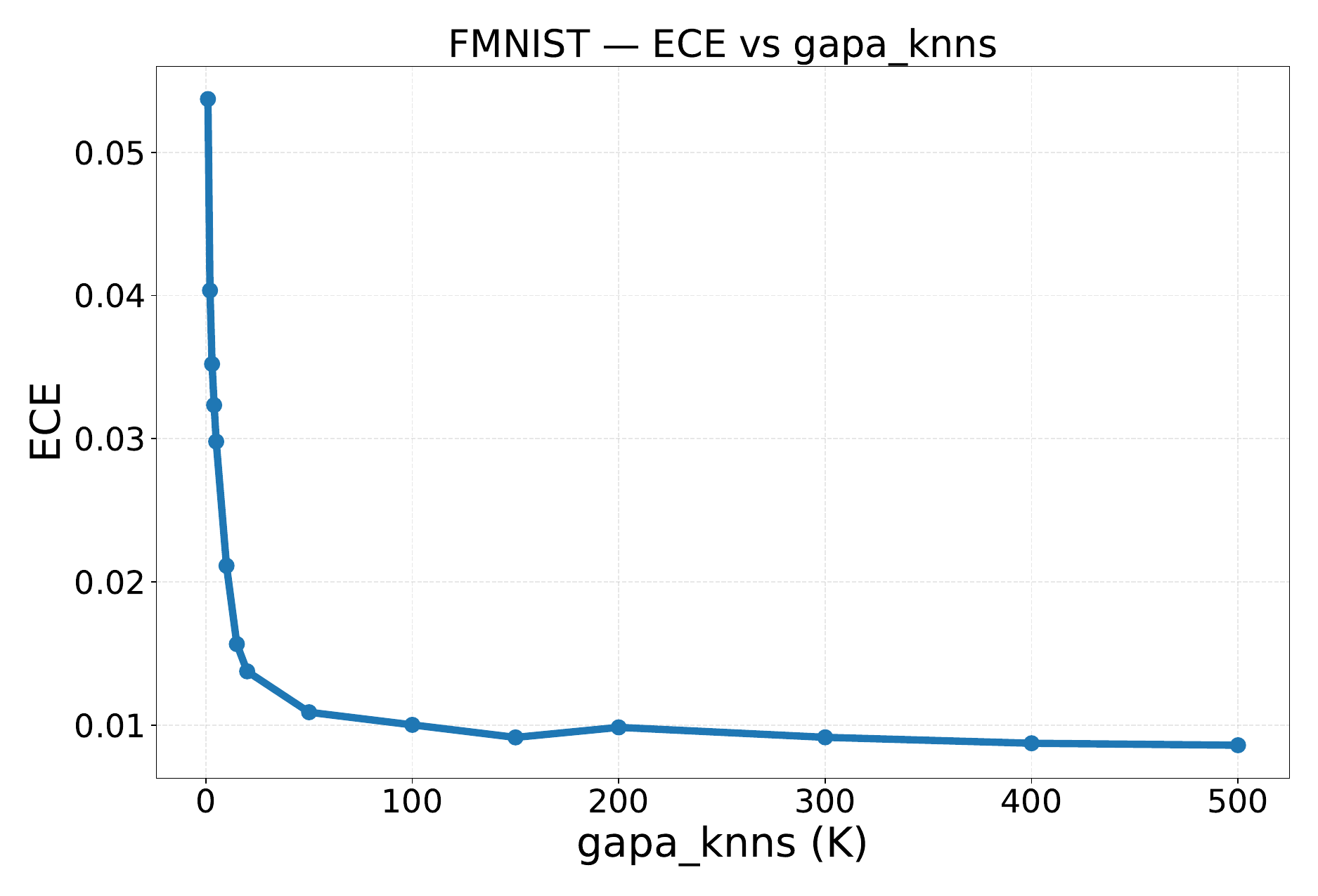}
        \caption{FMNIST ECE vs.\ $K$.}
        \label{fig:knn_fmnist_ece}
    \end{minipage}
\end{figure*}

\vspace{1em}
\paragraph{OOD-AUC.}
OOD detection improves slightly with $K$.
MNIST increases from $0.950$ (K{=}1) to $0.963$ (K{=}500).
FMNIST improves up to $K{\approx}50$, then plateaus or slightly degrades for
very large $K$ due to over-smoothing.

\begin{figure*}[t]
    \centering
    \begin{minipage}{0.45\textwidth}
        \centering
        \includegraphics[width=\textwidth]{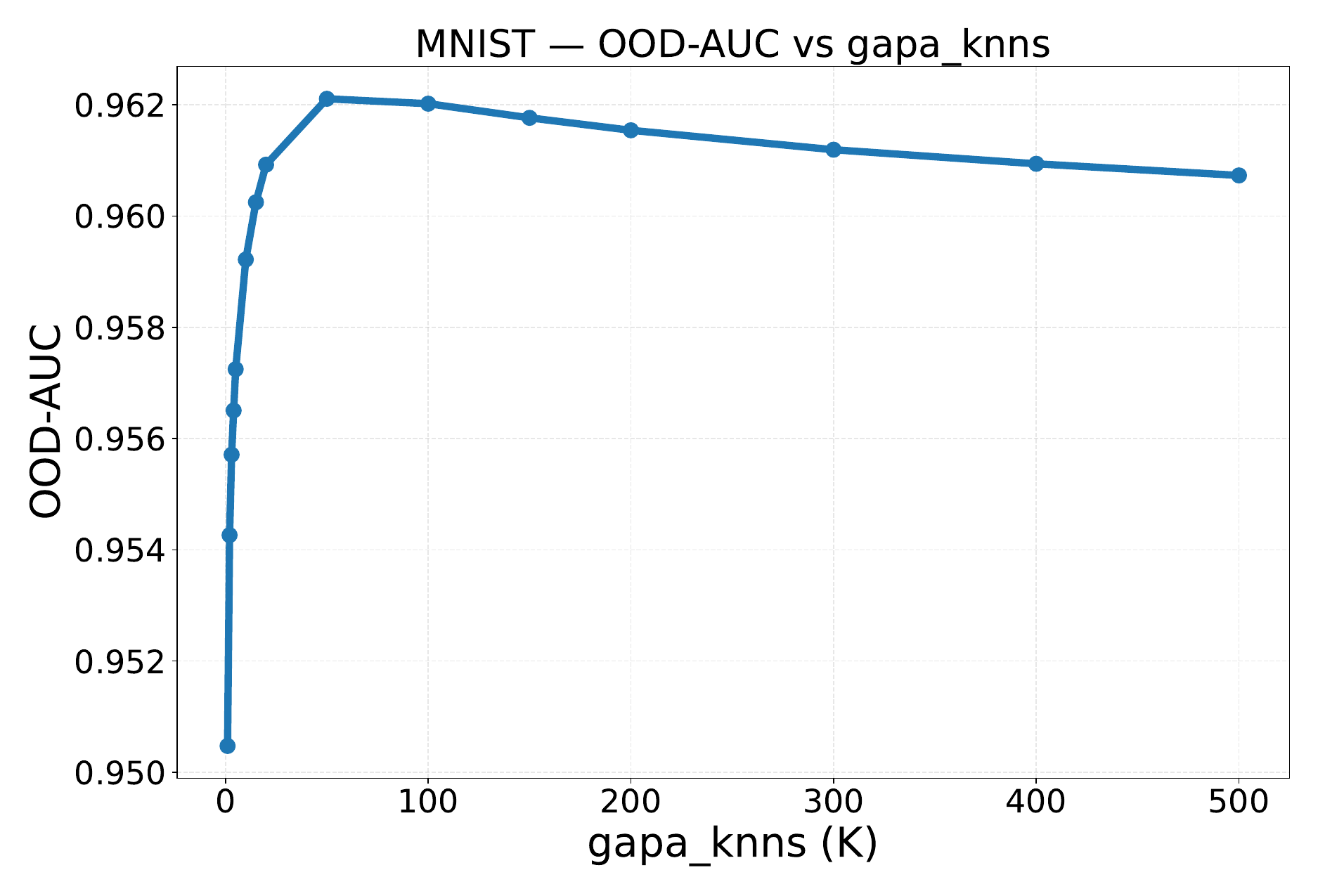}
        \caption{MNIST OOD-AUC vs.\ $K$.}
        \label{fig:knn_mnist_oodauc}
    \end{minipage}
    \hfill
    \begin{minipage}{0.45\textwidth}
        \centering
        \includegraphics[width=\textwidth]{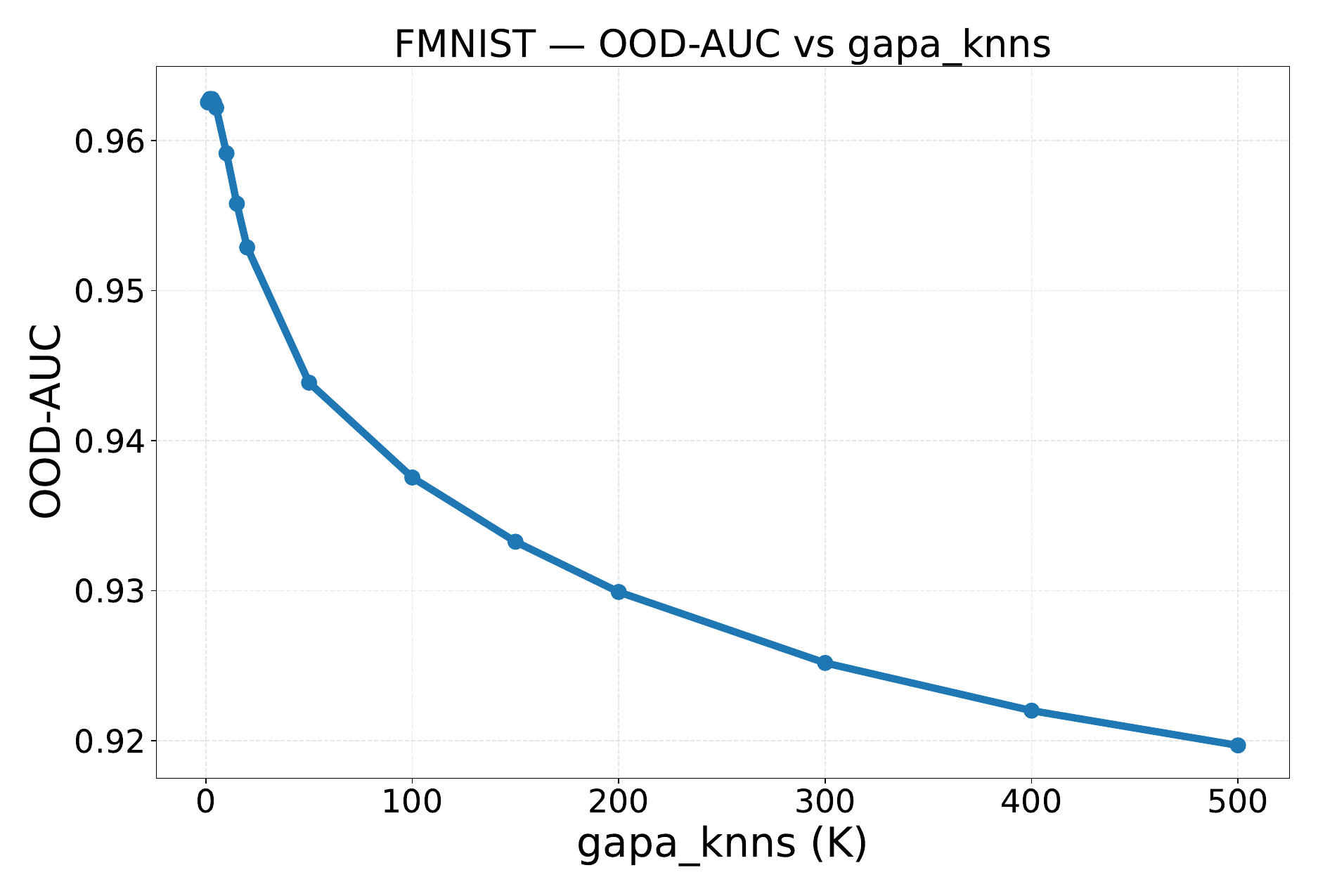}
        \caption{FMNIST OOD-AUC vs.\ $K$.}
        \label{fig:knn_fmnist_oodauc}
    \end{minipage}
\end{figure*}

\vspace{1em}
\paragraph{OOD BALD.}
Epistemic sensitivity improves steadily for both datasets, with consistent behaviour
across the entire sweep.

\begin{figure*}[t]
    \centering
    \begin{minipage}{0.45\textwidth}
        \centering
        \includegraphics[width=\textwidth]{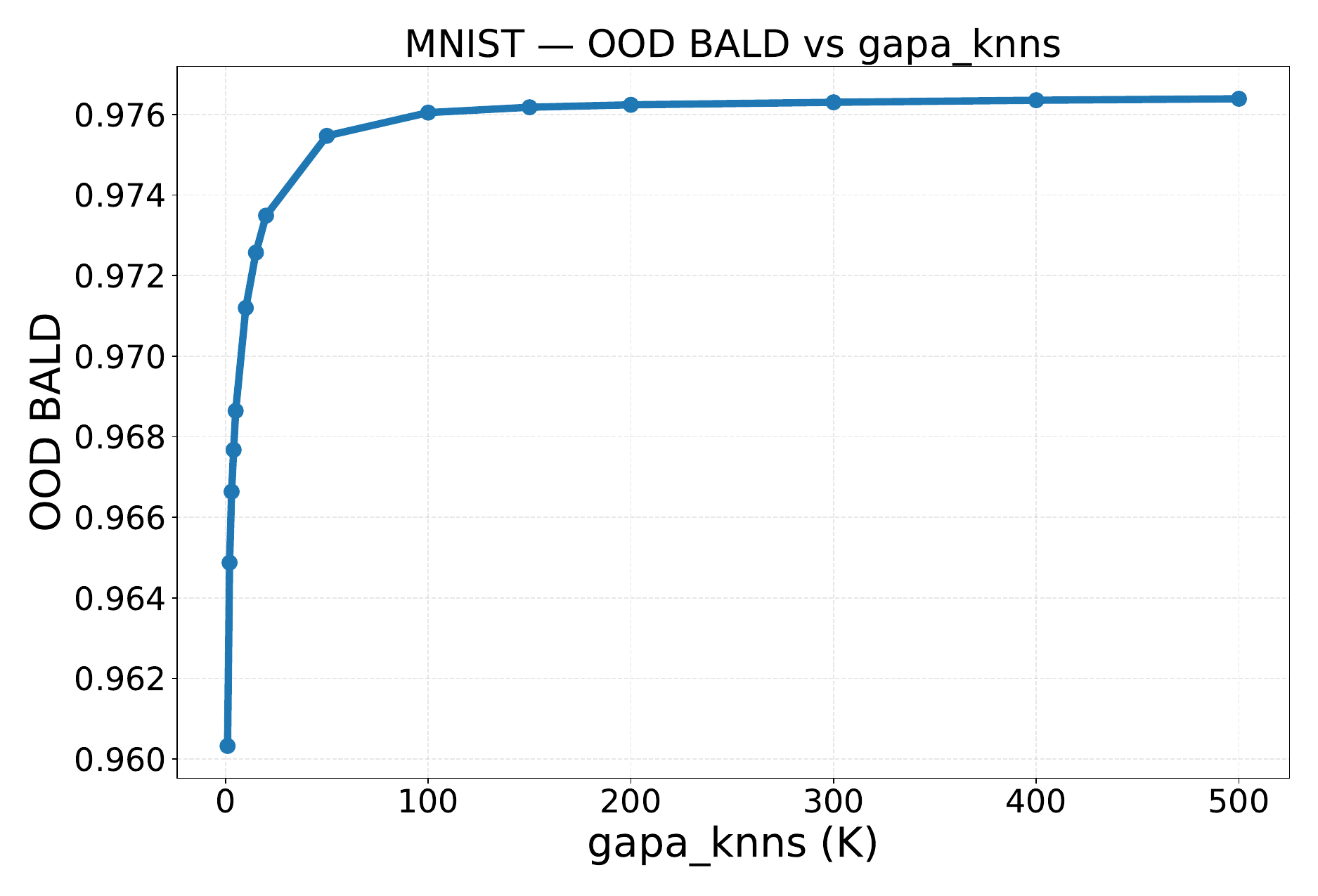}
        \caption{MNIST OOD-BALD vs.\ $K$.}
        \label{fig:knn_mnist_bald}
    \end{minipage}
    \hfill
    \begin{minipage}{0.45\textwidth}
        \centering
        \includegraphics[width=\textwidth]{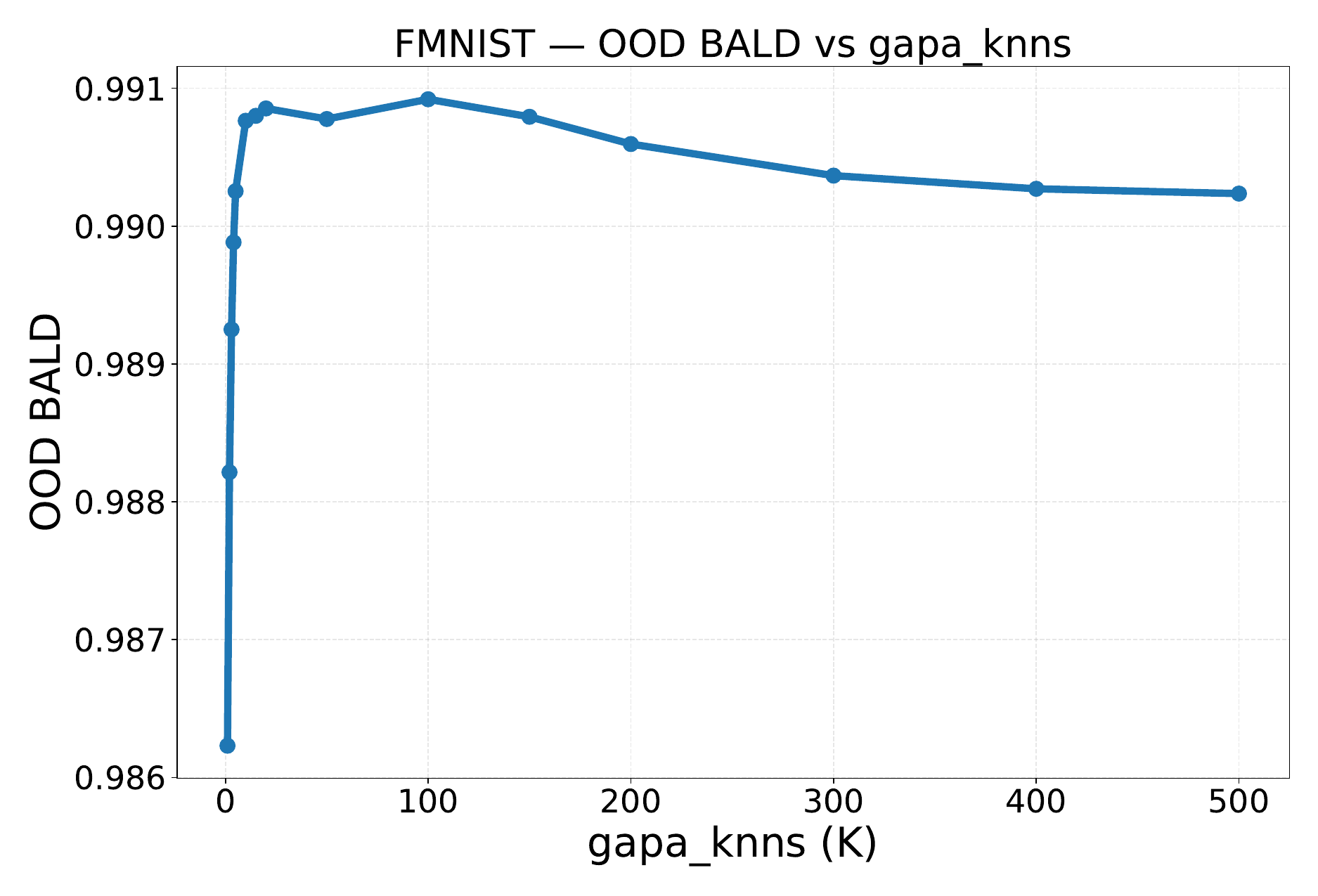}
        \caption{FMNIST OOD-BALD vs.\ $K$.}
        \label{fig:knn_fmnist_bald}
    \end{minipage}
\end{figure*}

\vspace{1em}
\paragraph{Test-time cost.}
Test-time increases roughly linearly with $K$ for both datasets.
For MNIST, inference grows from $\approx 2.1$ms to $\approx 16$ms.
FMNIST follows the same scaling pattern.

\begin{figure*}[t]
    \centering
    \begin{minipage}{0.45\textwidth}
        \centering
        \includegraphics[width=\textwidth]{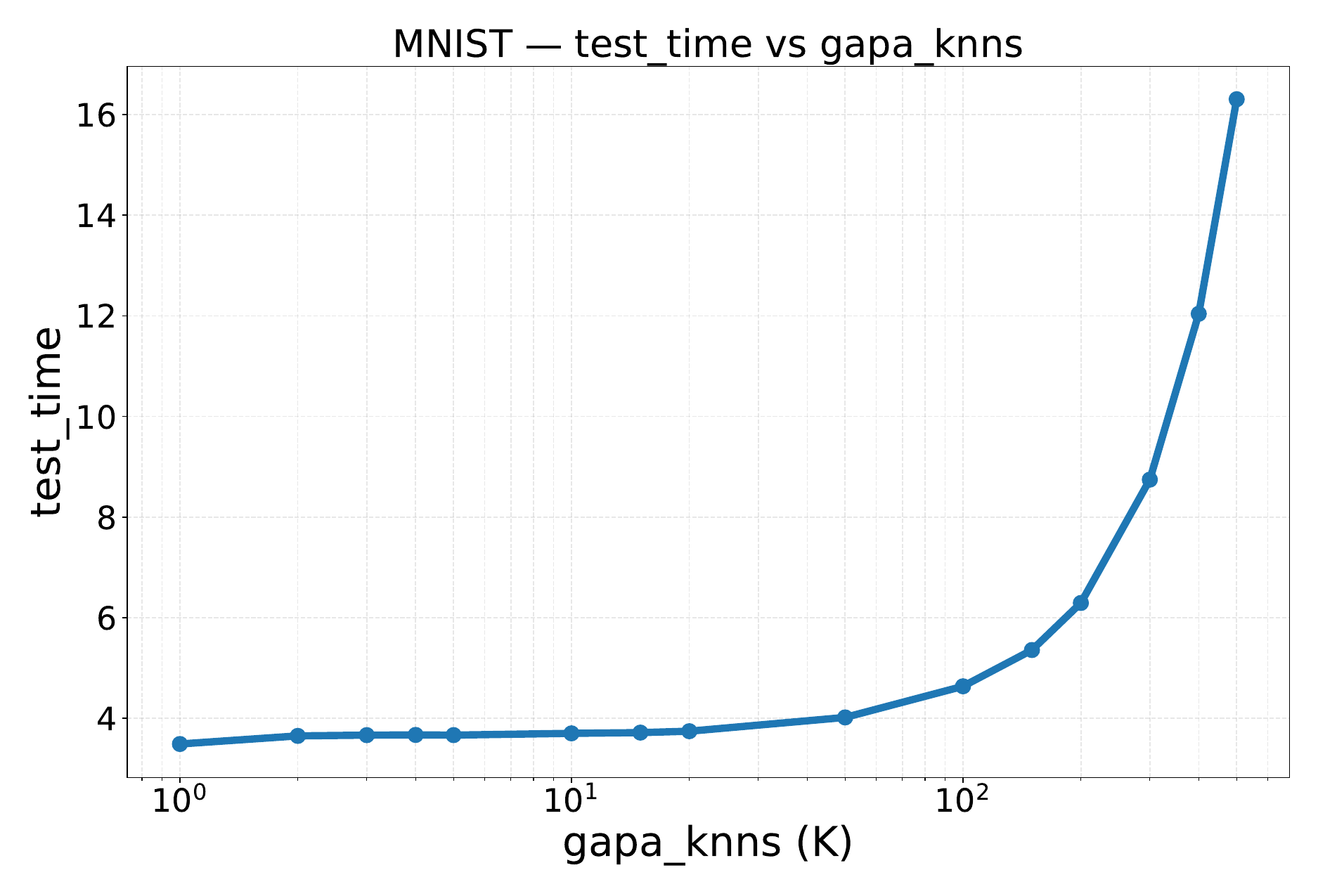}
        \caption{MNIST test time vs.\ $K$.}
        \label{fig:knn_mnist_testtime}
    \end{minipage}
    \hfill
    \begin{minipage}{0.45\textwidth}
        \centering
        \includegraphics[width=\textwidth]{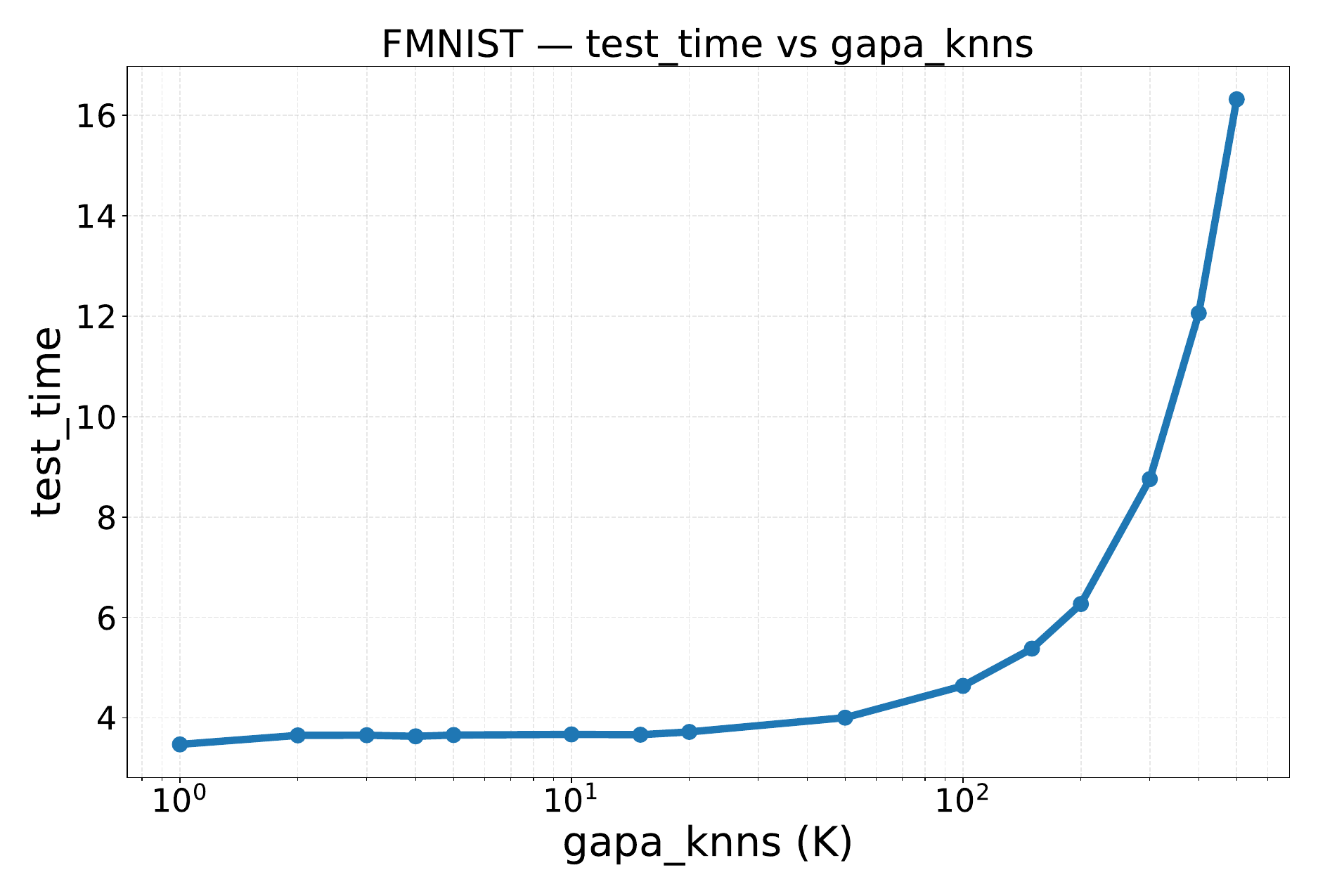}
        \caption{FMNIST test time vs.\ $K$.}
        \label{fig:knn_fmnist_testtime}
    \end{minipage}
\end{figure*}

\vspace{1em}
\paragraph{Takeaway.}
These experiments show:
\begin{itemize}
    \item \textbf{1-NN is already stable and competitive}, especially for OOD detection.
    \item Increasing $K$ to 20–50 provides clear gains in calibration and NLL.
    \item Very large $K$ has diminishing returns and incurs high compute cost.
\end{itemize}

Overall, the full sweep confirms that the KNN GAPA approximation is
\textbf{robust, stable, and effective}, and that GAPA behaves predictably across the
entire KNN range.

\section{Extended Related Work}
\label{sec:related_extended}

Uncertainty quantification (UQ) in deep learning differs primarily in \emph{where} uncertainty is placed
(e.g., weights, outputs, or representations) and in the resulting training and test-time costs.
We focus on the \emph{post-hoc} regime for \emph{frozen pretrained backbones}, where retraining,
multiple test-time samples, or full-network second-order computation can be infeasible.

\paragraph{Weight-space Bayesianization and Laplace.}
Weight-space Bayesian methods model uncertainty directly over parameters, typically requiring either
approximate inference during training or posterior sampling at test time.
Laplace approximations instead fit a local Gaussian posterior around a trained solution using curvature
information, and recent work has revisited Laplace as a competitive and practical Bayesian deep learning
baseline, including analyses of when structured approximations are necessary for scalability and fidelity
\citep{blundell2015weight, mackay1992practical, daxberger2021laplace, ortega2023variational, deng2022accelerated}. Closely related, stochastic weight-space approximations such as SWAG provide inexpensive posterior samples from SGD trajectories and often serve as strong uncertainty baselines without changing
the underlying architecture \citep{maddox2019simple}.

\paragraph{Last-layer and single-pass post-hoc methods.}
A popular compromise Bayesianizes only the final layer(s) while keeping the feature extractor frozen,
yielding post-hoc uncertainty at substantially reduced cost.
Beyond Laplace-style last-layer treatments, deterministic variational formulations such as variational
Bayesian last layers (VBLL) aim to deliver \emph{single-pass} predictive uncertainty for frozen-backbone
models \citep{harrison2024variational}. These approaches are strong deployment-friendly baselines, but they concentrate Bayesian
modeling at the head and do not in general propagate mean-preserving uncertainty through intermediate
computations of the frozen network.

\paragraph{Sampling-based baselines.}
Deep ensembles approximate epistemic uncertainty through multiple independently trained models~\citep{lakshminarayanan2017simple},
while MC Dropout relies on multiple stochastic forward passes at inference~\citep{gal2016dropout}. These are
effective but often incompatible with strict single-pass deployment constraints.

\paragraph{Representation-based and calibration baselines.}
Distance-/density-aware approaches estimate uncertainty from representations, e.g., spectral-normalized GP-style
heads (SNGP)~\citep{liu2020simple} or density-based uncertainty on deep features (DDU)~\citep{mukhoti2023deep}.
Calibration-only post-processing such as temperature scaling can improve confidence calibration but does not model
epistemic uncertainty~\citep{guo2017calibration}.

\paragraph{Gaussian processes and function-space views.}
Gaussian processes (GPs) provide a classical function-space approach to uncertainty and connect naturally to
infinite-width neural networks: fully-connected nets converge to an NNGP prior~\citep{lee2017deep}, and their
infinite-width training dynamics are characterized by the NTK~\citep{jacot2018neural}. For scalability, GP/kernel
Inference is commonly accelerated via low-rank approximations such as inducing points~\citep{titsias2009variational}
and Nyström methods~\citep{williams2000using}, while local GP approximations and noisy-input corrections (e.g., NIGP)
provide efficient variance adjustments under locality or input uncertainty~\citep{gramacy2015local, mchutchon2011Gaussian}.
GAPA builds on these function-space ideas but applies them \emph{inside} frozen networks by placing uncertainty in activation space, preserving the pretrained mean and enabling deterministic single-pass inference.

\paragraph{Summary.}
Overall, existing approaches trade off between (i) retraining or multi-sample inference, (ii) post-hoc last-layer
approximations that concentrate uncertainty at the head, or (iii) representation-based proxies. These gaps motivate
mean-preserving, post-hoc activation-space uncertainty with deterministic single-pass inference.

\end{document}